\documentclass[opre,nonblindrev]{informs3} 

\OneAndAHalfSpacedXI 

\usepackage{natbib}
 \bibpunct[, ]{(}{)}{,}{a}{}{,}%
 \def\bibfont{\small}%

\input{preamble}

\EquationsNumberedThrough    

\TheoremsNumberedThrough     
\ECRepeatTheorems  %

\MANUSCRIPTNO{MNSC-0001-2024.00}

\begin{document}


\RUNAUTHOR{Hssaine, Hu, and Pike-Burke}

\RUNTITLE{On the Limited Cost of Consistency in Points-Based Rewards Programs}

\TITLE{\revedit{On the Limited Cost of Consistency\\ in Points-Based Rewards Programs}}

\ARTICLEAUTHORS{%
\AUTHOR{Chamsi Hssaine}
\AFF{Department of Data Sciences and Operations, University of Southern California, Marshall School of Business, \EMAIL{hssaine@usc.edu}}

\AUTHOR{Yichun Hu}
\AFF{Johnson Graduate School of Management, Cornell University, \EMAIL{yh767@cornell.edu}}

\AUTHOR{Ciara Pike-Burke}
\AFF{Department of Mathematics, Imperial College London, \EMAIL{c.pike-burke@imperial.ac.uk}}

} 

\ABSTRACT{%
\revedit{Points-based rewards programs are a prevalent way to incentivize customer loyalty. In theory, the effective design of these programs is shaped by two important considerations. Customer heterogeneity suggests that tailoring thresholds to different customer types may yield revenue gains, while demand uncertainty creates incentives to update thresholds over time as the seller learns how customers respond to the program. In practice, however, there may be nontrivial operational and customer-facing benefits associated with maintaining {\it consistency} along both of these dimensions. Setting a single threshold across customer types can facilitate transparent communication and reduce the legal, reputational, and operational risks associated with personalization, while avoiding threshold increases prevents the devaluation of customers' previously earned points. Motivated by these considerations, this paper studies the cost of consistency in points-based rewards programs. We study this question within the context of the popular ``Buy $N$, Get One Free'' (BNGO) rewards programs. We first bound the {\it value of personalization} of these programs, proving that the revenue of the optimal {single-threshold} program is within a factor of $\frac{\pi^2}{6}$ of that of the optimal personalized strategy. We then tackle the problem of designing {\it devaluation-free} learning algorithms --- i.e., algorithms that never increase the redemption threshold --- in the presence of demand uncertainty. Toward this goal, we first design a greedy learning algorithm that  incurs $\widetilde{O}(\sqrt{T})$ regret in expectation over a horizon of length $T$; this guarantee is optimal, up to polylogarithmic factors. We then provide a devaluation-free modification of this algorithm, and show that this results in only a constant factor loss in regret.
Finally, we conduct extensive numerical experiments to show the limited value of personalization in average-case settings, in addition to demonstrating the strong practical performance of our proposed learning algorithms. 
}

}%




\KEYWORDS{revenue management, rewards programs, consistency, online learning, differentiation} 

\maketitle


\section{Introduction}\label{sec:intro}

Loyalty programs have long been a way for companies to increase their revenues, beginning with the introduction of grocery store trading stamps in the late 1800s \citep{greenstamps}. Since then, they have exploded in popularity, with over 90\% of companies maintaining some sort of loyalty program in 2016, and the average customer enrolled in over 15 loyalty programs today  \citep{vox}. One prominent form of loyalty program is the {\it points-based rewards} program. In a points-based rewards program, customers accumulate points every time they make a purchase; once the balance of points accumulated exceeds a certain redemption threshold, the customer is able to redeem her points for a reward (typically a free item or a discount on the next purchase). 
Prominent examples of points-based rewards programs include those offered by airlines, hotels, and casinos \citep{kopalle2012joint}, and those offered by the food service industry \citep{starbucks,mcdonalds,wendys,tacobell}, typically maintained through mobile applications.  

Due to their preponderance in practice, the impact of points-based programs on customer behavior has been a topic of extensive study in the marketing literature. In particular, a substantial amount of empirical work has found evidence of a behavioral phenomenon known as the {\it points pressure effect}, which describes the idea that points-based programs give customers a goal to work toward (i.e., accumulating enough points to obtain a reward). This goal incentivizes them to purchase more frequently than they would without the prospect of obtaining a reward; moreover, the rate at which they make purchases only increases the closer they are to achieving this goal \citep{kivetz2006goal,hartmann2008frequency,kopalle2012joint}.\footnote{The points pressure phenomenon is related to the goal gradient effect in psychology, the classic finding that animals expend more effort as they approach a reward \citep{kivetz2006goal}. This was first observed in rats searching for food in a maze \citep{hull1934rat}.} This points pressure effect engenders the following trade-off for companies (henceforth referred to as {\it decision-makers}, or {\it sellers}): while rewards programs generate additional revenue due to the increase in purchase frequency as customers approach the redemption threshold, this increase in revenue is immediately followed by a revenue loss from having to give out a reward. 
The decision-maker must therefore trade off between setting a lower redemption threshold, which would increase purchase probabilities but result in rewards being offered more frequently, and setting a higher threshold, which would reduce the regularity of rewards but result in a decrease in customers' purchase probabilities. The success of any rewards program hinges on the ability to set thresholds that optimally trade off between these incentives. 

Two important considerations arise when setting redemption thresholds: (i) there is significant variability in the points pressure effect across customers \citep{kopalle2012joint}, and (ii) the relationship between a customer's point balance and the probability with which they make a purchase is typically unknown {and must be learned while the loyalty program is running}. \revedit{Both of these considerations create incentives for the decision-maker to vary redemption thresholds, across time and customers. Namely, the presence of customer heterogeneity suggests that tailoring redemption thresholds to different customer segments (also referred to as customer {\it types}) may yield revenue gains. On the other, demand uncertainty requires the decision-maker to vary thresholds over time in order to accurately learn optimal threshold policies. 

In practice, however, there may be nontrivial operational and customer-facing benefits associated with maintaining a {\it consistent} rewards program along both of these dimensions. For instance, a significant disadvantage of tailoring redemption thresholds to individual customers relates to marketing and communication. While it may be feasible to publicly communicate customer-specific thresholds when these are determined according to easily understood segmentations such as customer tiers, in many cases, setting a single threshold for all customers allows companies to market their rewards programs publicly and transparently (e.g., a national ad campaign for a ``Buy $N$, Get One Free'' program). This issue is particularly salient if segments are inferred using customer-specific data such as purchase histories and other behavioral or demographic features. In these cases, the resulting segments may not be interpretable; \revedit{such opacity can be linked to both legal and reputational risks. For example, recent reporting points to ``loyalty penalty'' concerns in rewards programs, where companies using loyalty-app data to personalize customer benefits offer less generous terms to customers who are already likely to purchase, in effect penalizing, rather than rewarding, loyalty \citep{washingtonpost2025starbucks}. Additionally, if the features used to infer the segments are correlated with protected characteristics (e.g., purchase frequency may be correlated with income, which may in turn be correlated with race), companies may face a disparate impact risk, wherein certain customers unintentionally receive systematically more favorable terms because of correlations with protected characteristics.} Finally, personalized thresholds impose additional engineering requirements in e-commerce settings, as information systems must retrieve and apply customer-specific thresholds at the time of purchase or redemption, increasing the operational complexity of the program.

On the other hand, changing the threshold over time exhibits significant drawbacks relating to customers' long-term relationship with these rewards programs. In particular, since customers' purchase decisions depend on the redemption threshold that has been set for them, an increase in threshold is likely to be perceived as an implicit breach of contract. For instance, in ``Buy $N$, Get One Free'' programs, a redemption threshold of $N$ constitutes a guarantee of a free product after $N$ purchases from the customer's perspective. An increase in redemption threshold that may result from experimentation  effectively reneges on that promise, and risks undermining the credibility and transparency of the program, and the company more broadly. This claim is well-supported by a number of recent real-world instances wherein companies faced significant backlash after increasing the number of points required for redemption, effectively devaluing customers' previously earned points. Prominent names associated with these scandals (which were often followed by swift reversals) include Best Buy, Starbucks and Dunkin' Donuts \citep{cnn}, Chipotle \citep{chipotle}, Chick-fil-A \citep{chickfila}, Microsoft \citep{microsoft}, and Tesco \citep{tesco}. Such concerns recently reached the highest levels of government, with the United States Department of Transportation (USDOT) launching an investigation into the four largest U.S. airlines' rewards programs. In the announcement of the investigation, the USDOT noted potential unfair practices in the way these companies set point values, highlighting in particular the devaluation of previously earned points \citep{usdot}.  

Thus motivated, our work is concerned with the design of effective learning algorithms for the problem of optimal goal-setting in points-based rewards programs, with a special focus on the {\it consistency} --- both cross-customer and temporal --- of these programs. Specifically, we seek to answer the following research questions:
\smallskip

\begin{center}
{\em What is the value of personalization in points-based rewards programs? In the presence of demand uncertainty, are  devaluation-free learning algorithms able to learn optimal redemption thresholds?} 
\end{center}

\smallskip
}

Toward answering these questions, we consider a model in which a seller repeatedly offers a product at a fixed price to a finite population of heterogeneous customers. We conceptualize many of the points-based rewards programs referred to above via the classical ``Buy $N$, Get One Free'' (\bogo) program. Under this program, customers accumulate one point for each purchase that is made, and may redeem the item for free after they have made their $N$-th purchase. 
These \bogo\ programs are popular in practice due in large part to their simplicity, which has additionally made them prime candidates for tractable analysis in the operations literature \citep{liu2021analysis}.
Real-world examples of rewards given out within the context of \bogo\ programs include free hotel nights \citep{kopalle2012joint}, golf rounds \citep{hartmann2008frequency}, coffee \citep{kivetz2006goal}, and grocery items \citep{lal2003impact} (see \citet{liu2021analysis} for an excellent set of examples). Seminal work by \citet{kopalle2003economic} also noted that frequent-flyer programs can be conceptualized as ``Fly $N$ times, Get $(N+1)$-st flight free.'' 

In our model, customers are partitioned into $K$ observable types (e.g., according to characteristics such as age and gender), and make their purchase or redemption decisions in each period according to an unknown, type-specific purchase probability. In line with the points pressure phenomenon, we assume the purchase probability is non-increasing in the number of points remaining until redemption. Importantly, these purchase probabilities are unknown, so the decision-maker must experiment with various redemption thresholds over a finite horizon of $T$ time periods. \revedit{Within this framework, the main contribution of our work is to demonstrate {\it the limited cost of consistency in points-based rewards programs}.}

\subsection{Main Contributions}

\subsubsection*{\revedit{Bounding the value of personalization.}} \revedit{In our first main contribution, we study the revenue loss associated with setting a uniform threshold across all customer types. Specifically, we define the {\it value of personalization} as the ratio between the optimal {personalized} program, which may set a different redemption threshold for each customer type, and the optimal {non-personalized} program, which is constrained to set the same redemption threshold across all customer types (\Cref{def:pof}). 
} Given the limited assumptions imposed on the relationship between points to redemption and purchase probabilities, one may a priori expect that there exist instances where the \revedit{value of personalization} is arbitrarily large. This could occur for instance if implementing a ``Buy One, Get One Free'' program is optimal for one type of customer, whereas for another type of customer, it is optimal to not implement any rewards program. Moreover, previous work studying the impact of personalization on incentives for retention has found that the \revedit{value of personalization} can be unbounded \citep{freund2024fair}. However, in our first main contribution, we provide a uniform upper bound on the \revedit{value of personalization}, across all possible instances: {\it the long-run average revenue of the optimal personalized \bogo\ program is no more than \revedit{$\frac{\pi^2}{6} \approx 1.64$} times that of the optimal non-personalized program} (\Cref{thm:price-of-fairness}). {We complement this theoretical finding with extensive numerical experiments {(Section~\ref{sec:numerics})} that show that the \revedit{value of personalization} may be much lower than this worst-case upper bound in {average-case} settings. These results yield the important managerial insight that a seller cannot extract an arbitrary amount of revenue from heterogeneity in these settings.

\subsubsection*{\revedit{Devaluation-free learning}.} Having established a small \revedit{value of personalization} in complete-information settings, we turn to the question of designing {\it \revedit{devaluation-free}} algorithms in the learning setting, where the dependence of customers' purchase probabilities on the number of points to redemption is unknown. In line with much of the literature on demand learning \citep{filippi2010parametric,broder2012dynamic,ban2021personalized,bastani2021mostly}, we assume that customers' type-specific purchase probabilities follow a Generalized Linear Model (GLM) with unknown parameters. Following the previous discussion, we seek to find a single redemption threshold that maximizes the long-run average revenue across all customers.

\revedit{As a building block towards the design of a {devaluation-free} learning algorithm, we first analyze a greedy epoch-based algorithm, \revedit{Greedy-MLE} (\Cref{alg:greedy}).} This algorithm partitions the horizon into epochs of geometrically increasing length. At the beginning of each epoch, given observations of customers' purchase decisions at their respective point balances, it computes the Maximum Likelihood Estimate (MLE) of the unknown GLM parameters, and solves for the revenue-maximizing threshold, given the MLE. To allow for the possibility that not offering a rewards program is optimal, we also compare the (known) revenue without a rewards program to this estimated revenue, terminating the rewards program if this difference exceeds an epoch-specific confidence parameter. \revedit{We prove that this algorithm incurs $\widetilde{O}(\sqrt{MT})$ regret in expectation, for a fixed population of size $M$ (\Cref{thm:main-thm}), all the while only modifying the threshold $O(\log T)$ times throughout the horizon.} We show this is optimal up to polylogarithmic factors by providing a matching lower bound of $\Omega(\sqrt{MT})$ on the regret of any (potentially non-\revedit{devaluation-free}) policy (\Cref{thm: lower bound}).

Despite its strong guarantees, the possibility remains that \revedit{Greedy-MLE} may devalue customers' points by increasing the threshold, albeit infrequently. \revedit{We propose a simple, devaluation-free modification (\revedit{\DfGreedy}, \Cref{alg:non-increasing}) that similarly proceeds in epochs; however, instead of choosing the greedy threshold at the beginning of each epoch, it chooses the largest threshold within a consideration set of thresholds.} Thresholds are included in this consideration set if and only if their estimated revenue under the MLE is close enough to that of the optimal greedy solution. Importantly, the consideration sets are nested, which guarantees that the sequence of thresholds is non-increasing (i.e., devaluation-free). Leveraging our previous regret analysis, we show that this algorithm incurs only a factor of 2 loss relative to the regret bound of \revedit{Greedy-MLE} in the worst case, and is therefore also order optimal (\Cref{thm:extension}). In synthetic experiments {(Section~\ref{sec:numerics})}, we observe the strong performance of both \revedit{Greedy-MLE} and \revedit{\DfGreedy}, in addition to numerically demonstrating the trade-off between revenue and devaluation-free learning. Furthermore, we empirically show that both algorithms are robust to misspecification of the GLM.

From a technical perspective, our work uncovers the interesting fact that optimal learning algorithms do not need to explicitly explore in our setting. This lies in stark contrast to the extensively studied problem of pricing under demand uncertainty, for which the suboptimality of greedy algorithms is well-known in non-contextual settings \citep{broder2012dynamic, keskin2014dynamic, den2014simultaneously}. Work on pricing in {\it contextual} settings has however shown that greedy algorithms may be optimal, under certain regularity conditions on the exogenous distribution from which contexts are drawn \citep{qiang2016dynamic, javanmard2019dynamic}. In contrast to this latter set of results, we require no additional assumptions on customers' purchase probabilities to show the optimality of \revedit{Greedy-MLE}. Rather, our results follow from the fact that customers running through multiple redemption cycles throughout a single epoch induces a form of ``natural exploration.'' This phenomenon guarantees sufficient variability in the points to redemption that the resulting MLE is a high-quality estimate of the unknown parameters. The technical crux of our work lies in demonstrating this fact, which relies on deriving a lower bound on the minimum eigenvalue of the empirical Fisher information matrix (henceforth referred to as the design matrix) of each epoch. 
Proving this requires a careful analysis that considers a Markov chain representation of a customer's points to redemption and derives a new Chernoff-type bound for the concentration of samples from this Markov chain. This differs from the analysis in related problems, where the assumption of i.i.d. contexts significantly simplifies the concentration results.

\subsubsection*{Paper organization.} We review the related literature in the rest of this section. We present the seller's long-run average revenue maximization problem under complete information in \Cref{sec:preliminaries}, and derive a bound on the \revedit{value of personalization} in \bogo\ programs in \Cref{sec:complete-info}. The seller's learning problem under incomplete information is described in \Cref{sec:learning}. Our two main algorithmic contributions are presented in Sections \ref{sec:ub} and \ref{sec:extension}. We test their performance in computational experiments in \Cref{sec:numerics}. Conclusions are finally provided in \Cref{sec:conclusion}.

\subsection{Related Literature}\label{ssec:related-lit}

Our work contributes to the extensive literature on points-based rewards programs, studied from various perspectives in the operations, marketing, and economics literatures. We detail the most closely related works below.

\subsubsection*{Frequency rewards programs: Empirical work.} There is extensive empirical work on the impact of frequency rewards programs on customers' purchasing behavior (see \citet{dorotic2012loyalty} and \citet{chen2021three} for exhaustive overviews). For instance, early work by \citet{dreze1998exploiting} analyzed data from a rewards program offering a \$10 gift card for every \$100 spend on baby purchases, and found that the sale of baby products increased by 25\% overall as a result of this program. Significant increases in purchase frequency as a result of such spending-based programs have been identified within the context of grocery and convenience stores \citep{lal2003impact, lewis2004influence, taylor2005current, liu2007long}, with the strongest effect found on infrequent shoppers. The impact of points pressure has also been studied within the context of points-based rewards programs more specifically. \citet{kivetz2006goal} studied a caf\'e rewards program and observed that customers purchase coffee more frequently the closer they are to earning a free coffee.  They also found the points pressure effect within an employment context, where internet users who rate songs in exchange for reward certificates rate more songs as they approach their goal. Using data from a ``Buy 10, Get One Free'' program offered by a golf course, \citet{hartmann2008frequency} found that customers' switching costs --- costs incurred by purchasing from a firm other than the one with whom they are accumulating points --- monotonically increase as customers earn additional credits toward a reward; these switching costs return to their initial level immediately after the reward is cashed in. This phenomenon is precisely the type of points pressure captured by our model. \citet{kopalle2012joint}  similarly observe the points pressure effect in a major hotel chain’s rewards program, highlighting substantial variation in how customers value a free hotel stay. The findings of \cite{kivetz2006goal,hartmann2008frequency,kopalle2012joint} are the basis for the behavioral model we consider here.

\subsubsection*{Frequency rewards programs: Analytical work.} {To the best of our knowledge, our work is the first to study the task of learning points-based rewards programs.} The analytical study of frequency rewards programs in complete-information settings, however, has a long history in economics and marketing. 
Early work studied the mechanisms underlying the profitability of these rewards programs, in the hopes of providing theoretical explanations for the empirical findings described above (see, e.g., \citet{klemperer1987markets, kim2001reward, kopalle2003economic,kim2004managing,singh2008research} for seminal works). Using stylized game-theoretic models of duopolistic competition, these papers analytically show the effectiveness of points-based programs due to the switching costs that arise when customers accumulate rewards.

More recently, a growing line of work has focused on various operational aspects of the design of rewards programs in a monopoly, similarly under the assumption that the underlying behavioral model is known. For instance, \citet{sun2019model} investigate the economic rationale behind finite expiration terms, for a rewards program in which a reward is always available, but can only be redeemed through future purchases. This work was later extended to the study of rewards programs in two-sided markets \citep{lyu2024customer}. Similar to our setting, this work explicitly models customer heterogeneity (i.e., low versus high valuation customers, frequent and infrequent customers). However, their model does not incorporate reward accumulation and redemption thresholds, an important feature of many points-based programs. \citet{chun2020loyalty} study the problem of a firm optimally setting points' value within the context of liability management. In the model they consider, the customer population is modeled in aggregate, with the total existing point balance evolving as a random, exogenous quantity in each period. Our work, on the other hand, is specifically interested in learning the impact of points pressure on customer decision-making, which requires us to model customers at the micro-level and keep track of the way in which they accumulate rewards. \citet{chung2022dynamic} study another important operational aspect of rewards programs, namely their impact on dynamic pricing decisions when a firm has a limited inventory of products. In our model, the seller has an infinite inventory of products, a reasonable assumption for, e.g., dining and grocery settings; additionally, we explicitly model reward accumulation. 

The ``Buy $N$, Get One Free'' program that we study follows that of \citet{liu2021analysis}, who analytically show how points pressure arises in a complete-information setting. As in our model, they assume that a product's price is fixed over time, and aim to find a fixed redemption threshold to maximize the long-run average revenue. In their model, the seller interacts with a single customer with a known valuation and stochastic outside option. The customer is assumed to be forward-looking, and strategically decides to purchase, redeem, or opt out in each period in order to maximize her total discounted utility. The authors derive conditions under which a \bogo\ program can improve firm profitability, in addition to showing when a customer's willingness to make a purchase increases with her inventory of points. The major modeling difference that we have relative to \citet{liu2021analysis} is their assumption that customers are forward-looking. While such a model is useful from the perspective of {\it explaining} how points pressure may arise in complete-information settings, we instead are interested in deriving {\it prescriptive} solutions for the task of learning optimal redemption thresholds in the presence of demand uncertainty. From this perspective, we assume that customers are non-strategic, with their purchase decisions governed by an exogenously given purchase probability that depends on the number of points to redemption. This assumption is also made in the vast majority of work on pricing under demand uncertainty (see discussion below). This parsimonious model allows us to capture the most salient feature of customer behavior induced by points-based rewards programs --- that of points pressure --- all the while being flexible enough to allow for customer heterogeneity and to model reward accumulation at the micro-level. We also note that, contrary to this and many of the works discussed above, our work does not consider the joint optimization of prices and redemption thresholds. This design decision is due to the fact that enrollment is not automatic in most points-based rewards programs. As a result, jointly optimizing over prices and rewards would require also modeling unenrolled customers. While this is an interesting future direction, we view this as less fundamental to the learning challenge on which our work is focused.

Finally, we briefly mention recent work on the design of {\it tier}-based loyalty programs (e.g., \citet{chun2019strategic}), wherein customer behavior is motivated by access to tiers providing additional benefits. This line of work is tangential to the main focus of this paper, which specifically studies {\it rewards}-based loyalty programs. However, the effectiveness of policies that enforce consistency constraints in tier-based programs is an interesting topic for future work.

\revedit{
\subsubsection*{Flexibility in operations.}  {Our focus on consistency in points-based rewards relates to the literature on flexibility in operations.} {Within this literature, several papers study the value of personalization in pricing and incentive design.} {Most closely related is \citet{freund2024fair}, who study cross-customer uniformity constraints in the design of ``surprise and delight'' incentives for retention. In their setting, the value of personalization may be unbounded; by contrast, we find that the value of personalization is uniformly bounded in points-based rewards programs.} {Within the context of pricing, \citet{elmachtoub2021value} study the value of idealized personalized pricing --- which charges each customer exactly their willingness-to-pay --- over single-price strategies. More closely related to our model, \citet{cohen2022price} study the revenue impact of ``limited differentiation'' constraints that bound the difference between prices charged to different customer segments. Subsequent works study the impact of such constraints on dynamic pricing \citep{chen2021fairness, cohen2021dynamic, xu2023doubly,chen2023utility} and competitive pricing \citep{yang2023regulating}.\footnote{\revedit{Less closely related are works that provide bounds on the revenue of the optimal anonymous (i.e., non-personalized) price relative to the revenue of the optimal auction in single-item settings in the mechanism design literature (see, e.g., \citet{alaei2019optimal}).}} {Finally, \citet{kallus2021fairness} explore the relationship between fairness, welfare, and equity considerations in personalized pricing.}

Separately, our focus on temporal consistency falls under the umbrella of online decision-making under limited adaptivity. Within this literature, adaptivity constraints typically appear in the form of switching costs; see, e.g. \citet{cesa2013online,dekel2014bandits}. They have also appeared as costs incurred from assigning arriving demand to non-preferred resources \citep{asadpour2020online, freund2023power}, or deviating from target allocations \citep{hssaine2024online}. Limited adaptivity constraints also arise within the context of fairness in resource allocation, where {envy-freeness} constraints bound the difference in allocations individuals receive across time \citep{sinclair2022sequential, banerjee2023online, onyeze2026sequential}. For pricing specifically, existing works have modeled limited adaptivity by assuming that the seller can make only finitely many price changes (e.g., \citet{broder2011online}, \citet{cheung2017dynamic}, \citet{chen2020data}, \citet{perakis2024dynamic}), or by assuming that the seller has implemented price protection guarantees \citep{feng2024temporal}. The techniques developed in these latter works do not apply to our setting, since the type of limited adaptivity we are interested in is (i) one-directional, and (ii) enforced via a strict constraint, as opposed to softly discouraged by imposing switching costs on the seller, as many of the aforementioned works do.

Finally, one of the main motivations behind consistency in algorithmic decision-making is the negative long-term impacts associated with instability in customer-facing levers. As a result, the general topic of long-term customer engagement has attracted increasing attention in recent years. For instance, \citet{bastani2022learning} study the problem of personalizing product recommendations in the presence of disengagement. \citet{sumida2023optimizing} propose  a learning algorithm for repeated assortment optimization when customers' purchase probabilities depend on their past purchase history. \citet{li2026inter} study dynamic pricing in complete-information settings under price monotonicity and promotion fatigue constraints.   Taking the reverse perspective as ours, \citet{lugosi2023bandit} consider a multi-armed bandit problem from the customer's point of view, where the customer learns her own preferences for different arms (i.e., sellers), all the while obtaining an additional payoff, a ``fidelity reward,'' depending on how loyal the customer has been to that arm in the past. 
}
}

\subsubsection*{Parametric models of pricing under demand uncertainty.} We conclude the section with a discussion of the methodological connections between our work and the abundant line of work on learning optimal pricing strategies under demand uncertainty, more specifically when demand follows an unknown parametric model (commonly linear or generalized linear demand). We highlight the most closely related works below, and refer the reader to \citet{den2015dynamic} for a survey of existing work.

Early work by \citet{broder2012dynamic} studied a model in which a seller prices a product over a sequence of $T$ customers who make Bernoulli purchasing decisions given the offered price. They design an epoch-based policy that consecutively explores and exploits, using the MLE from past observations. Under the assumption that there exists a known set of exploration prices for which the minimum eigenvalue of the design matrix is lower bounded by a constant, they show that such a policy achieves the optimal $O(\sqrt{T})$ regret relative to a clairvoyant policy that has access to the unknown parameters. Guaranteeing a lower bound on the minimum eigenvalue of the design matrix turns out to be generally necessary for policies to learn the revenue-maximizing price. When the seller does not have access to such a set of ``good'' exploration prices, \citet{keskin2014dynamic} and \citet{den2014simultaneously} both highlight that greedy MLE-based policies may suffer from the phenomenon of {\it incomplete learning}. They however show that injecting exploration carefully into these policies in a way that guarantees a constant lower bound on this minimum eigenvalue achieves the optimal regret guarantee. 

In contrast to the vanilla pricing settings considered in these prior works, \citet{qiang2016dynamic} later showed that when the seller has additional feature information in each period (i.e., the seller is in a {\it contextual} setting), greedy personalized pricing policies may no longer be suboptimal. Specifically, under the assumption that the covariance matrix of demand covariates (typically assumed to be drawn i.i.d. in each period) is positive definite, a greedy iterated least squares policy achieves the optimal $O(\log T)$ regret guarantee under linear demand. \citet{javanmard2019dynamic} similarly assume positive definiteness of this covariance matrix, and show that an epoch-based greedy MLE policy achieves $O(\log T)$ regret in settings where demand depends only on a sparse set of features.  \citet{bastani2021mostly} also find that greedy is optimal in a contextual bandits problem, under an assumption termed {\it covariate diversity} that imposes a type of positive definiteness constraint on the contexts. In all of these papers, the authors show that positive definiteness implies that a constant lower bound on the minimum eigenvalue of the design matrix holds under greedy policies. At a high level, this assumption ensures enough ``natural exploration'' to guarantee a fast learning rate of model parameters.  One of our main contributions, which demonstrates that a greedy policy is optimal for the problem of learning the optimal redemption threshold, is in line with these latter findings. While we are not in the contextual setting, the points to redemption in each period can be considered as a covariate in an online regression problem, meaning that we can leverage recent results on generalized linear contextual bandits \citep{li2017provably} in the proof of our main technical result. 
We highlight however that our work differs from the above series of papers in two crucial ways. First, in the pricing setting, demand is i.i.d. across time; this models, for instance, a decision-maker selling the product to a different customer in each period. In our setting, we assume that a fixed population of customers {\it repeatedly} interacts with the system throughout the horizon; our modeling of individual reward accumulation through time induces {\it Markov-modulated}, as opposed to i.i.d., demand as a result. Additionally, our results do not require assumptions on the positive definiteness of the covariance matrix of the ``contexts'' in our setting. Rather, we obtain a lower bound on the minimum eigenvalue of the design matrix via a careful analysis of the variance of the underlying Markov chain, a feature that is not present in the pricing literature. 
Finally, we note that despite the connections to the contextual pricing setting, we show via an $\Omega(\sqrt{T})$ lower bound on the regret that our setting is fundamentally more difficult than contextual pricing, where $O(\log T)$ regret is achievable. 

\section{Preliminaries}\label{sec:preliminaries}

{In this section we present our model for the ``Buy $\threshold$, Get One Free'' program under repeated interactions, and define the long-term optimization problem faced by the seller under the assumption that she has complete information on all model primitives.
We defer the specification of the learning setting to \Cref{sec:learning}. 
A discussion of our modeling assumptions is provided at the end of this section.}

\subsubsection*{Technical notation.} We use the notation $\mathbb{N}^+$ to denote the set of strictly positive integers. For any $T \in \mathbb{N}^+$, we let $[T] = \{1,2,\ldots,T\}$. We moreover denote the positive part function by $(\cdot)^+ = \max\{\cdot,0\}$, and let \mbox{$a \wedge b = \min\{a, b\}$}. For any $\mu \in [0,1]$, we let $\bern(\mu)$ denote a random variable drawn from a Bernoulli distribution with mean $\mu$. 
{Finally, $\|\cdot\|$ is used to denote the $\ell_2$-norm of a given vector.}

\subsubsection*{The ``Buy \threshold, Get One Free'' program.} We consider a multiperiod problem where a seller (also referred to as a decision-maker) offering a single product or service {repeatedly} interacts with {a fixed population of customers} {over an infinite horizon}. Let $\population$ denote the fixed population of customers, which has size ${M} = \abs{\population}$. In each period, the seller offers the product to {each} customer at a fixed price;\footnote{The fixed price assumption is motivated by the practical reality that sellers typically implement reward programs well after an original pricing decision is made.} we assume that the marginal cost of producing the product is zero. {The seller may choose to implement} a ``Buy $\threshold$, Get One Free'' (\bogo) program, wherein {each} customer is eligible to receive a  free product after making $\threshold$ purchases.  We henceforth refer to $\threshold$ as the {\it redemption threshold} or {\it goal}, which will be chosen from a set of feasible thresholds $\{1,2,\ldots,\maxthreshold\}$, where \revedit{$\maxthreshold \in \mathbb{N}$ is a finite constant}.\footnote{\revedit{We assume that $\maxthreshold$ is an input to the model and does not scale with any other model primitive. For instance, within the context of caf\'e rewards programs, a retailer may impose an upper bound of $\maxthreshold=30$ on the set of feasible redemption thresholds. Such an assumption is reasonable in practice, where exorbitantly high redemption thresholds may undermine the credibility of the rewards program, or incur administrative / operational costs that outweigh the marginal increase in purchase probability associated with high redemption thresholds.}} We use the notational convention that $N = +\infty$ if the seller does not implement a \bogo\ program, and refer to this as the \emph{no-loyalty} option. 

{We model the implementation of the \bogo\ program as in \citet{liu2021analysis}.} Consider a fixed threshold $\threshold$ and a customer $j \in {\mathcal{M}}$. At the beginning of each period $t \in \mathbb{N}^+$, customer $j$ has current point balance (or point {\it stock}) denoted by $\randomstock_{jt} \in [N]\cup \{0\}$. If the customer has not yet reached the redemption threshold (i.e., $\randomstock_{jt} < \threshold$), she makes a random decision as to whether or not to purchase the product. {The randomness in this decision may, for instance, be due to variability in the customer's valuation for the product, or in competitive outside options that are outside of the seller's control.} The customer earns one point if she makes a purchase, and zero points otherwise. {Once her point balance reaches the threshold (i.e., $\randomstock_{jt} = \threshold$), she may either redeem all $\threshold$ points for a free product, or choose an outside option. We assume the customer cannot make a cash purchase once the redemption threshold is met.} We use the variable $\randompurchase_{jt} \in \{0,1\}$ to represent the random purchase decision, or redemption decision, where applicable, and assume it is made independently across customers. Once the redemption threshold is met, if the customer chooses to redeem her points for the product, her point balance resets to zero at the start of the next period; otherwise, it remains the same. Once $S_{jt}$ resets to zero, the sequence of interactions repeats. We refer to the process of purchase decisions until eventual redemption as a {\it redemption cycle}.

\subsubsection*{Behavioral model.} Customers are partitioned into $K \in \mathbb{N}^+$ observable types which determine the probability with which customers purchase and redeem the product.{\footnote{\revedit{Types can be thought of as fixed segments according to which the seller has partitioned the customer population. For instance, the seller may partition customers based on baseline purchase probability, prior to the implementation of the rewards program (e.g., frequent versus infrequent customers). We moreover note that such segmentation is standard in the literature (see, e.g., \citet{cohen2022price}, \citet{chen2021fairness}, \citet{cohen2021dynamic}, \citet{xu2023doubly}, \citet{freund2024fair}).} 
} For $k \in [K]$, we let $\rho_k$ be the fraction of type-$k$ customers in the population, with $\rho_{\min} = \min_{k\in[K]}\rho_k$. For a customer $j \in {\mathcal{M}}$, we use $k(j)$ to denote their type. 

To model the points pressure effect,
we define the {\it points to redemption} as the number of purchases remaining until the customer attains the redemption threshold, and denote this by $\tau_{jt} = N-\randomstock_{jt}$ for customer $j \in {\mathcal{M}}$ and period $t \in \mathbb{N}^+$.
We assume that the customer's random purchase (resp., redemption) probability is a function of $\tau_{jt}$.
Formally, let $\purchaseprob_{k}: \{0,1,\ldots,\threshold\}\mapsto[0,1]$ be the purchase probability function, such that for any number of points to redemption $\tau$, $\purchaseprob_k(\tau)$ is \revedit{the probability with which a type-$k$ customer obtains the product (i.e., purchases if $\tau > 0$, or redeems if $\tau = 0$)}, i.e. \mbox{$\purchaseprob_k(\tau) = \mathbb{P}(X_{jt}=1|\tau_{jt}=\tau, k(j)=k)$}. {In other words, given $\tau_{jt} = \tau$ and $k(j) = k$, $X_{jt}$ is drawn independently from a Bernoulli distribution with parameter $\purchaseprob_k(\tau)$.} In line with the {empirical literature \citep{kivetz2006goal,hartmann2008frequency,kopalle2012joint}}, we assume $\purchaseprob_k(\cdot)$ is non-increasing in $\tau$, for all $k \in [K]$. That is, customers are more likely to purchase a product as they approach the redemption threshold. Finally, in the absence of the \bogo\ program, customers make a random purchase decision in each period, which we denote by $\noloyaltyprob_k$, for $k \in [K]$. To model the idea that customers are likely to ignore extremely large redemption thresholds, again in line with the literature on points pressure, we assume that  \mbox{$\lim_{\tau\to\infty} \purchaseprob_k(\tau) = \noloyaltyprob_k$}. \Cref{ex:purchase-prob-example} shows common instantiations of the purchase probability $\purchaseprob_k(\cdot)$ satisfying our mild structural assumptions.
\begin{example}\label{ex:purchase-prob-example} Consider the following special cases of $\purchaseprob_k(\cdot)$.
\begin{itemize}[leftmargin=0.75cm, labelsep=0.25cm]
\item \textbf{No points pressure:}  $\purchaseprob_k(\tau) = \noloyaltyprob_k$. {This models a setting where the \bogo\ program has no impact on the customer's purchase decision.}
\item \textbf{Linear points pressure:} $\purchaseprob_k(\tau) = \noloyaltyprob_k + (\alpha_k-\beta_k\tau)^+$ for some $\alpha_k, \beta_k > 0$. 
\item \textbf{Exponential points pressure:} $\purchaseprob_k(\tau) = \noloyaltyprob_k +e^{\alpha_k-\beta_k\tau}$ for some $\alpha_k, \beta_k > 0$.
\item \textbf{Logit points pressure:} $\purchaseprob_k(\tau) = \noloyaltyprob_k+\frac{e^{\alpha_k-\beta_k\tau}}{1+e^{\alpha_k-\beta_k\tau}}$ for some $\alpha_k,\beta_k > 0$.
\item \textbf{Generalized linear points pressure:} $\purchaseprob_k(\tau) = \mu_k(\alpha_k-\beta_k\tau)$, for some $\alpha_k,\beta_k > 0$ and increasing link function $\mu_k(\cdot)$, with $\lim_{\tau\to\infty}\mu_k(\alpha_k-\beta_k\tau) = \noloyaltyprob_k$. We will assume such a generalized linear model (GLM) for the learning setting (see \Cref{sec:learning}).
\end{itemize}
\end{example}

\subsubsection*{Objective.} In the complete-information setting, the goal of the seller is to design a \bogo\ program by selecting thresholds that maximize her long-run average expected revenue per customer (including potentially not offering a \bogo\ program at all). 

Toward understanding the value of personalization in these programs, we will consider both type-specific and type-agnostic thresholds. Let $\mathbf{N} = (N_1,\ldots,N_K)$ be the vector of redemption thresholds set for each type.  Since the price of the product is fixed, maximizing long-run average revenue is equivalent to maximizing the long-run average purchase probability,\footnote{In the remainder of the paper, we abuse terminology and frequently refer to the long-run average purchase probability as the long-run average revenue.} given by:
\begin{align}\label{eq:rev-def}
R(\mathbf{N}) = \lim_{T\to\infty}\frac{1}{{M}T}\sum_{j=1}^{{M}}\sum_{t=1}^{T} \phi_{k(j)}(\tau_{jt})\mathds{1}\left\{\tau_{jt} > 0\right\},
\end{align}
for an arbitrary initial number of points to redemption $\tau_{j1}$, with $\tau_{j,t+1} = (\tau_{jt}-X_{jt})\mod(N_{k(j)}+1)$ for all $j \in \population$, $t \in \mathbb{N}^+$.
In \Cref{prop:closed-form-lr-avg-per-cust} we establish that this limit indeed exists and is unique. 

\Cref{eq:rev-def} highlights the key trade-off in designing a \bogo\ program: when the redemption threshold is small, the points pressure effect kicks in early, resulting in customers purchasing the product with higher likelihood throughout the redemption cycle. This, however, comes at the cost of customers being able to redeem more frequently, resulting in a loss in revenue. Conversely, if the threshold is large, more purchases are required for redemption, but the likelihood that a customer purchases remains low for longer, as the customer needs to acquire more points for the points pressure effect to kick in significantly. The absence of the \bogo\ program pushes this effect to the limit, with customers always purchasing the product with the same (potentially low) probability, but the seller never giving up any revenue by giving out the item for free. We formalize this key trade-off by providing closed-form expressions for (i) the long-run average purchase probability, and (ii) the long-run average fraction of time a customer spends with a specific point balance, for each customer $j \in \population$.
\begin{proposition}\label{prop:closed-form-lr-avg-per-cust}
Given redemption threshold $\threshold$, for any initial number of points to redemption \mbox{$(\tau_{j1}, j \in \population)$}, the long-run average purchase probability for each customer $j \in \population$ is given by:
\begin{align}\label{eq:closed-form-lr-avg-per-cust}
\lim_{T\to\infty}\frac1T\sum_{t=1}^T\phi_{k(j)}(\tau_{jt})\mathds{1}\{\tau_{jt} > 0\} = \frac{\threshold}{\sum_{\tau=0}^N\frac{1}{\phi_{k(j)}(\tau)}}.
\end{align}
Moreover, the long-run average fraction of time customer $j$ has $\tau \in \{0,1,\ldots,N\}$ points remaining until redemption is given by:
\begin{align}\label{eq:stat-dist-per-type}
\statprob{k(j)}{\tau}{N} :=\lim_{T\to\infty}\frac1T\sum_{t=1}^T\mathds{1}\{\tau_{jt} = \tau\} = \frac{1}{\sum_{\tau'=0}^N\frac{\phi_{k(j)}(\tau)}{\phi_{k(j)}(\tau')}}.
\end{align}
\end{proposition}

We defer the proof of \Cref{prop:closed-form-lr-avg-per-cust} to Appendix \ref{apx:closed-form-lr-avg-per-cust}. For ease of notation, we let $R_k(N) = \frac{N}{\sum_{\tau=0}^N\frac{1}{\phi_k(\tau)}}$ be the long-run average purchase probability for a type-$k$ customer. Applying the dominated convergence theorem to \Cref{eq:closed-form-lr-avg-per-cust} in \Cref{prop:closed-form-lr-avg-per-cust}, the long-run average revenue across the entire population of customers given a vector of thresholds $\mathbf{N}$ is then given by 
\begin{align}
R(\mathbf{N}) = \sum_{k\in[K]}{\rho_k}{R_k(N_k)}.
\end{align}

\Cref{prop:closed-form-lr-avg-per-cust} formalizes the key trade-off between maximizing the purchase probabilities and minimizing the number of free products described above. In particular, notice that the denominator in the right-hand side of \Cref{eq:closed-form-lr-avg-per-cust} represents the expected time to complete a redemption cycle, since the time to move from $\tau$ to $\tau-1$ for any customer $j$ is a Geometric random variable parametrized by success probability $\phi_{k(j)}(\tau)$. Since the number of purchases per redemption cycle is necessarily $N$, \Cref{eq:closed-form-lr-avg-per-cust} can therefore be interpreted as the average number of purchases per period in a redemption cycle. So, while a larger threshold is beneficial from a revenue perspective, as this would reduce the number of times the seller needs to give out the product for free, this effect is dampened by how long it takes for the customer to complete a redemption cycle, since a higher threshold also reduces the purchase probability early on in the cycle. \

Finally, as a corollary of \Cref{prop:closed-form-lr-avg-per-cust}, we recover that, as the redemption threshold grows large, the seller's long-run average revenue converges to her revenue without a loyalty program. We defer the proof of \Cref{prop:special-case-noloyalty} to Appendix \ref{apx:no-loyalty-cvg}.
\begin{corollary}\label{prop:special-case-noloyalty}
$\lim_{\threshold\to\infty}R_k(\threshold) = \bar{\phi}_k$.
\end{corollary}

\revedit{
\begin{remark}
Given the generality of our model and the discrete nature of the seller's optimization problem, obtaining a closed-form expression for the optimal threshold is challenging, even for a single type of customer without additional structural assumptions on the purchase probability function $\phi(\cdot)$. In Appendix \ref{apx:special-case-lin} we derive insights into the structure of the optimal solution for the special case of linear points pressure.
\end{remark}
}

{
\subsubsection*{Discussion of modeling assumptions.} We conclude this section by discussing our main modeling assumptions. Chief among these is the fact that we consider an {\it exogenous} model of customer behavior, as opposed to assuming that customers strategically make their purchase and redemption decisions in each period to maximize their long-run average utility. As noted in \Cref{ssec:related-lit}, this modeling decision is in line with the literature on pricing under demand uncertainty \citep{den2015dynamic} and long-term impacts more generally \citep{bastani2022learning,sumida2023optimizing,hamilton2023churning,kanoria2024managing,freund2024fair}; it is moreover motivated by models that are used for learning customer preferences in practice. Our behavioral model parsimoniously captures one of the most salient features of customer behavior induced by points-based rewards: that of points pressure, which increases as the customer approaches the reward, and returns back to its initial level after redemption \citep{hartmann2008frequency}. This points pressure may arise due to some underlying switching costs that are a non-decreasing function of the number of points to redemption, as posited in early analytical works \citep{klemperer1987markets}; our model is general enough to include this possibility.

In line with the exogeneity assumption, we do not model strategic consumer stockpiling, wherein customers purchase more than one product in anticipation of potential redemption threshold increases. This assumption is also made in \citet{liu2021analysis}, who cite settings in which the product or service cannot be inventoried, as is the case for many of the \bogo\ programs mentioned in \Cref{sec:intro}.

Also similar to \citet{liu2021analysis},  we assume that once a customer has attained the redemption threshold, she cannot continue to accumulate points by purchasing the product with cash. From a practical perspective, such a design decision is easily implementable by the seller. From a theoretical perspective, under a fixed redemption threshold and price, customers have no incentive to delay redemption. 
In the learning setting, while the redemption threshold may vary, this variation is unpredictable from the customer's perspective, making it unlikely that the customer would time redemption in anticipation of such events. {With that said, our analysis is easily amenable to models in which customers probabilistically decide between making a cash purchase or redeeming after attaining the redemption threshold, since such a change would simply require analyzing a different, possibly infinite-state, Markov chain. While our exact bounds may change (in particular, the uniform upper bound on the \revedit{value of personalization} we derive in \Cref{thm:price-of-fairness}), we conjecture that our main insights do not. Namely, even under such generalizations we expect that (i) there exists a uniform upper bound on the  \revedit{value of personalization}, and (ii) \revedit{devaluation-freeness} comes essentially for free in learning settings.}

Finally, we emphasize that the goal of this work is not to capture all existing points-based rewards programs, or all aspects of customer behavior in response to such programs.  The goal of this work is to provide a first study of the \revedit{cost of consistency} in points-based rewards programs, given the well-documented phenomenon of points pressure. For this, we focus on  the simple and popular \bogo\ program. As discussed above, we conjecture that our main insights and the effectiveness of the types of algorithms we propose are invariant to added model complexity, and defer a discussion of interesting modeling extensions to \Cref{sec:conclusion}.
}

\section{On The Limited Value of Personalization}\label{sec:complete-info}

\revedit{We begin by quantifying the value of personalization in \bogo\ programs.} We first introduce some additional notation in order to define this concept. Let $\mathcal{R}^{\text{pers}}$ and $\mathcal{R}^{\text{non-pers}}$ respectively denote the optimal revenues under personalized and non-personalized thresholds, i.e.,
\[\mathcal{R}^{\text{pers}} = \sum_{k\in[K]}\rho_k\cdot\left[ \max_{N_k\in [\maxthreshold]\cup\{+\infty\}} R_k(N_k) \right] \qquad, \qquad \mathcal{R}^{\text{non-pers}} = \max_{N\in[\maxthreshold]\cup\{+\infty\}} \sum_{k\in[K]}\rho_k R_k(N).\]

We formally define the \revedit{value of personalization} below.

\smallskip

\begin{definition}[\revedit{Value of Personalization (VoP)}]\label{def:pof}
Given any \bogo\ instance, the \revedit{value of personalization} is the ratio of the optimal personalized revenue to the optimal non-personalized revenue. Formally:
\begin{align}
\text{\revedit{VoP}} = \frac{\mathcal{R}^{\text{pers}}}{\mathcal{R}^{\text{non-pers}}}.
\end{align}
\end{definition}

\smallskip
\revedit{One might a priori expect the \revedit{value of personalization} to in general be quite large, given that we impose very few structural assumptions on the relationship between the number of points to redemption $\tau$ and the purchase probability $\phi_k(\cdot)$, for each type $k$. One would moreover expect this value to grow with the number of types, as the seller needs to reconcile increasingly conflicting preferences.\footnote{\revedit{Indeed, existing works have established that personalization can lead to arbitrarily large revenue gains in other settings. Within the context of pricing, \citet{elmachtoub2021value} show that the \revedit{value of personalization} can grow as $\Omega(S)$, where $S$ denotes the dispersion of customers' valuations for a product; \citet{freund2024fair} show an unbounded \revedit{value of personalization} within the context of randomized incentives for retention.}} In our main result for this section, however, we show that there is a limit to the gains that a seller can extract from personalization. In particular, {\it the optimal personalized scheme guarantees no more than \revedit{$\frac{\pi^2}{6}$} times the revenue generated by the optimal uniform threshold, independent of all model primitives.} We formally state this in \Cref{thm:price-of-fairness} below, deferring the proof of the result to Appendix \ref{apx:price-of-fairness}. In what follows, we let \mbox{$N_k^* \in \argmax_{N\in[\maxthreshold]\cup\{+\infty\}} R_k(N)$} denote the optimal personalized threshold for type-$k$ customers, and use the convention that $N_1^* \leq N_2^* \leq \ldots \leq N_K^*$.}
\begin{theorem}\label{thm:price-of-fairness}
 For any instance of the \bogo\ problem,
\revedit
{
\begin{align}\label{eq:pof-bound}
\text{\revedit{VoP}} \leq \min\left\{1 + \frac{1}{N_1^*} - \frac{1}{N_K^*},K-(K-1)2^{-1/(K-1)},\frac{\pi^2}{6}\right\}.
\end{align}
Moreover, the bound of $K-(K-1)2^{-1/(K-1)}$ is tight for $K = 2$, and the bound of $\frac{\pi^2}{6}$ is asymptotically tight in the limit as $K\to\infty$.\footnote{\revedit{We highlight that the above bounds are {\it worst-case} bounds, over the set of all possible problem instances. 
In \Cref{sec:numerics} we numerically show that the \revedit{value of personalization} is on average much lower, for a wide set of randomly generated problem instances. {We moreover investigate the dependence of the \revedit{value of personalization} on the heterogeneity in the population, as measured by the number of types $K$ and the imbalance across types.}}
}
}
\end{theorem}

\revedit{\paragraph{Interpretation of the \revedit{VoP} bounds.} The first instance-dependent bound of $1 + \frac{1}{N_1^*} - \frac{1}{N_K^*}$ establishes that the main driver of the \revedit{value of personalization} is the gap between the customer type that is the {\it most} sensitive to the rewards program (and therefore has the lowest personalized threshold of $N_1^*$) and the type that is the {\it least} sensitive to the rewards program (and therefore has the highest personalized threshold of $N_K^*$). Intuitively, if $N_1^*$ and $N_K^*$ are close, then so must be $N_2^*,\ldots,N_{K-1}^*$; hence, the benefit of setting a uniform threshold across all $K$ types of customers is small. Three additional important insights arise from this instance-dependent bound. First, for fixed values of $N_1^*$ and $N_K^*$, the worst-case \revedit{value of personalization} does not depend on heterogeneity across ``intermediate'' types $2,\ldots,K-1$. Second, for a fixed gap $N_K^*-N_1^*$, the worst-case \revedit{value of personalization} {\it decreases} as the thresholds increase, suggesting that personalization is {\it more} valuable when customers are sensitive to the presence of a rewards program.} \revedit{One of the main technical contributions of our work is to further bound the VoP by $\min\left\{K-(K-1)2^{-1/(K-1)},\frac{\pi^2}{6}\right\}$, thereby establishing that, while personalization is (weakly) more valuable as the number of types increases and the population becomes more heterogeneous, these marginal gains steeply decrease as the number of types grows large. 

\paragraph{Intuition for bounded gains from personalization.} We now provide some intuition as to why the \revedit{VoP} is uniformly bounded in the worst case. Consider the uniform scheme that sets $\hat{N} = N_1^*$ for all customers. Under this scheme, the retailer extracts the maximum possible revenue from type-1 customers. For all other customer types, $\hat{N}$ is {\it lower} than the optimal personalized threshold. To understand the impact of this reduction in redemption threshold, it is useful to decompose the revenue extracted from each type into two components: (i) the average purchase probability throughout a redemption cycle, and (ii) the ``effective price'' of a product, given by $\frac{N}{N+1}$ for any fixed threshold $N$.\footnote{Note that $\frac{N}{N+1}$ is the effective price since it represents the fraction of {\it paid} purchases ($N$), relative to the total number of items acquired ($N+1$), in a redemption cycle.} By monotonicity of $\phi_k(\cdot)$, lowering the redemption threshold from $N_k^*$ to $\hat{N}$ results in an {\it increase} in average purchase probability (a revenue {\it gain}). The effective price $\frac{N}{N+1}$, however, {\it decreases} (a revenue {\it loss}, intuitively reflecting the cost of giving out the item for free more frequently). Hence, the \revedit{VoP} can be upper bounded by the relative loss due to the decrease in effective price. This relative loss is upper bounded by 2 in the worst case, since $\frac{\hat{N}}{\hat{N}+1} \geq \frac12$ (achieved at $\hat{N}=1$), and $\frac{N_k^*}{N_k^*+1} \leq 1$ (achieved as $N_k^* \to \infty$). Since this bound holds for all $k \in [K]$, this intuitive argument yields the loose upper bound of $2$ on the \revedit{VoP}. 

Tightening this upper bound requires a sharper worst-case analysis that cannot be achieved by simply bounding the effective price of each type in isolation, as described above. We tackle this challenge by instead defining the uniform threshold $\hat{N}$ to be the revenue-maximizing personalized threshold amongst $\{N_1^*,\ldots,N_K^*\}$, and explicitly characterizing the worst-case relative loss incurred by $\hat{N}$. This worst-case relative loss exhibits nontrivial dependence on the effective prices for each type in the optimal personalized solution. Specifically, we show that it depends on the ratio of effective prices associated with {\it consecutive types} under their respective personalized thresholds. The key observation is that, while the ratio of effective prices seen by two consecutive types may be as small as 1/2, this ratio cannot be small for {\it all} consecutive types. Careful subsequent bounding gives rise to our final result.}

With this result in hand, in the remainder of the paper we restrict our attention to the problem of learning the optimal {\it non-personalized} threshold that achieves $\mathcal{R}^{\text{non-pers}}$. (We note however that the learning algorithms we develop can easily be applied to each of the $K$ individual types, and immediately inherit the regret guarantees we derive, up to constant factors.) Hence, throughout the remainder of the paper we abuse notation and denote the long-run average revenue across the population of customers given a single threshold $N$ by $R(N) = \sum_{k\in[K]}\rho_k R_k(N)$.

\revedit{
\begin{remark}[Complexity of finding an optimal threshold]
While we assume that $N_{\max}$ is a constant, we note that $\Omega(N_{\max})$ operations are unavoidable to determine an optimal non-personalized threshold. This follows from the fact that, even when $K = 1$, simply computing $R(N) = \frac{N}{\sum_{\tau=0}^N\frac{1}{\phi(\tau)}}$ requires $\Theta(N_{\max})$ operations in the worst case. However, one can show that, when $K = 1$, assuming all values of $\phi(\tau)$ and partial sums $\sum_{i=0}^\tau \frac{1}{\phi(i)}$, for $\tau \in \{0,\ldots, N_{\max}\}$, are precomputed, an optimal non-personalized threshold can be determined in $O(\log N_{\max})$ operations. A corollary of this fact is that an optimal set of personalized thresholds can be found in $O(K\log N_{\max})$ operations. We refer the reader to Appendix \ref{apx:computational-complexity} for details.
\end{remark}

\begin{remark}[Impact of personalization on customers]
While our main focus in this section is on bounding the value of personalization for the seller from the perspective of {\it revenue}, an important real-world concern is the impact of personalization on customer utility. We quantify this impact analytically and numerically in Appendix \revedit{\ref{apx:price-of-personalization}}.
\end{remark}
}
\section{The Learning Setting}\label{sec:learning}

Having analyzed the \revedit{value of personalization} in the complete-information setting, we now turn to the incomplete-information setting, where the seller seeks to learn an optimal redemption threshold without prior knowledge of the relationship between customers' purchase probabilities and the points remaining to redemption. We devote this section to a complete description of the learning setting and a derivation of a lower bound on the regret of any learning algorithm. Our algorithmic contributions are deferred to \Cref{sec:learning-algs}.

\subsection{Setup}

We consider a finite horizon of $T$ periods over which the seller seeks to learn an optimal redemption threshold.
For each type $k \in [K]$, let $\mathcal{M}_k$ be the collection of all type-$k$ customers (recall, a customer's type is observable in our setting), with $\abs{\mathcal{M}_k} = \rho_k M$. {For simplicity we assume that $\abs{\mathcal{M}_k}$ is integral, for all $k \in [K]$.} 

At the beginning of each period $t\in[T]$, the seller sets a common redemption threshold for all {$M$} customers, or decides to pause the rewards program (i.e., she sets the redemption threshold to $+\infty$ and does not allow for redemption or point accumulation). If a redemption threshold is set, given the number of points remaining to redemption $\tau_j$, each customer $j \in \mathcal{M}$ independently makes a purchase or redemption decision according to $\phi_{k(j)}(\tau_j)$, which is unknown to the seller. If the seller pauses the rewards program in that period, the customer makes a purchase with probability $\bar{\phi}_{k(j)}$, which we assume is known.\footnote{\revedit{The assumption that $\bar{\phi}_k$ is known follows from the fact that the price is fixed in our model. For instance, the seller may have experimented with prices extensively in the absence of a rewards program, and thus already have a high-quality estimate of the relationship between price and purchase probability.}}

 \subsubsection*{Behavioral model.} We assume that the purchase probability of each type-$k$ customer follows a generalized linear model, i.e.,
$
    \phi_k(\tau) = \mu_k(\beta_{k,1} + \beta_{k,2} \tau) \ \forall \ \tau \in \{0,\ldots,\maxthreshold\},
$
where $N_{\max}$ is assumed to be known, and $\beta_{k,1} \in \mathbb{R}, \beta_{k,2} \in \mathbb{R}^-$ are parameters that are unknown to the seller. The function $\mu_k: \mathbb{R} \to [0,1]$ is a {known, strictly increasing link function such that $\lim_{x\to-\infty}\mu_k(x) = \bar{\phi}_k$.}\footnote{That $\mu_k$ is strictly increasing and $\beta_{k,2} \leq 0$ together imply that $\mu_k(\beta_{k,1}+\beta_{k,2}\tau)$ is decreasing in $\tau$, as required by our assumptions on $\phi_k$.} We assume that $\beta_{k,1}, \beta_{k,2}$ respectively take on values over known, compact subsets of $\mathbb{R}$ and $\mathbb{R}^-$, and let $\Theta_k$ denote the set of admissible parameters $\beta_k = (\beta_{k,1},\beta_{k,2})$. Let $\beta = (\beta_{k}; k \in [K])$. Finally, we impose the following standard regularity conditions on $\mu_k$.
\begin{assumption}[Basic assumptions]\label{asp:glm} For all $k \in [K]$, $\mu_k(\cdot)$ satisfies the following conditions:
\begin{enumerate}[label=\emph{(\alph*)}, labelsep=1em, left=0.25in]
    \item\label{asp:bound} \textbf{Boundedness:} There exist known constants $\mu_{\min}, \mu_{\max} \in (0,1]^2$ such that \[\mu_{\min} \le \mu_k(\beta_{k,1} + \beta_{k,2}\tau) \le \mu_{\max} \quad \forall \ \tau \in \{0,\ldots,N_{\max}\}, \ \beta_k \in \Theta_k \]
    \item\label{asp:lip} \textbf{Lipschitz continuity:} $\mu_k(\cdot)$ is $L_{\mu}$-Lipschitz, for some known constant $L_{\mu} > 0$.
    \item\label{asp:diff} \textbf{Twice differentiability:} $\mu_k(\cdot)$ is twice differentiable with respect to $\tau$. Moreover, there exist known constants $\kappa > 0$ and $G_{\mu} > 0$ such that
    \begin{align*}
    \kappa \leq \inf_{\substack{\norm{\beta'_k-\beta_k}\le {1}/{\sqrt{1+N_{\max}^2}}}} \dot{\mu}_k \prns{\beta'_{k,1}  +\beta'_{k,2} \tau} \quad \forall \ \tau \in \{0,\ldots,N_{\max}\} 
    \end{align*}
    and 
    \begin{align*} \abs{\ddot\mu_k(\beta_{k,1}+\beta_{k,2}\tau)} \leq G_{\mu} \quad \forall \ \tau \in \{0,\ldots,N_{\max}\}, \ \beta_k \in \Theta_k,
    \end{align*}
    where $\dot\mu_k(x)$ and $\ddot\mu_k(x)$ respectively denote the first and second derivatives with respect to $x$.
\end{enumerate}
\end{assumption}

As noted above, the conditions stated in \Cref{asp:glm} are commonly made in the literature on learning parametric choice models (see, e.g., \citet{broder2012dynamic}, \citet{li2017provably}). \revedit{\Cref{asp:glm} (a) is satisfied by taking $\mu_{\min} = \min_{k \in [K]} \bar{\phi}_k > 0$; moreover, the decision-maker can trivially set $\mu_{\max} = 1$, since $\phi_k(\tau) \leq 1$ for all \mbox{$k \in [K]$}, by definition.} \Cref{asp:glm} (b) states that the purchase probability does not vary too much if the number of points to redemption varies by a small amount. \Cref{asp:glm} (c) imposes a smoothness condition on $\mu$. These assumptions can easily be shown to hold for the linear, convex, and logit points pressure functions presented in \Cref{ex:purchase-prob-example}, under the assumption that $\phi_k(\tau) \in (0,1)$ for all $\tau \leq N_{\max}$, $k\in[K]$.

\subsubsection*{Policies and regret metric.} A policy $\policy$ is a mapping from the history of redemption thresholds and customers' purchase and redemption decisions, to a redemption threshold for the current period. Let $\Pi$ denote the set of all such policies. Given policy $\pi$, for $t \in [T]$ we let $N_t^\pi$ be the redemption threshold chosen at the beginning of period $t$, with $N_t^\pi = +\infty$ denoting the decision to pause the rewards program in period $t$. Leveraging the same notation as in \Cref{sec:preliminaries}, for each customer $j \in \mathcal{M}$, we let $S_{jt}^\pi$ be their point balance at the beginning of period $t$ under policy $\pi$, with {$\tau_{jt}^{\pi} = (N_t^\pi - S_{jt}^\pi)^+$} the corresponding points to redemption{\footnote{{Note that the positive part is only needed if the algorithm decreases the threshold in such a way that $N_t^\pi < S_{jt}^{\pi}$ at the beginning of period $t$.}}} and $X_{jt}^\pi$ the purchase or redemption decision made by the customer. Without loss of generality, we assume that all customers begin with a point balance of zero (i.e., $S_{j1}^\pi = 0$ for all $j \in \mathcal{M}$). Finally, we use the notation $\mathbb{E}_{\pi}[\cdot]$ to denote the expectation of a random variable with respect to the randomness induced by $\pi$. 

Our main performance metric will be a policy's cumulative regret relative to a clairvoyant decision-maker who has knowledge of the true parameters $\beta$ governing customer behavior. To formally define this metric, recall that $\mathcal{R}^{\text{non-pers}} = \max_{N\in[\maxthreshold]\cup\{+\infty\}} \sum_{k\in[K]}\rho_k R_k(N)$ denotes the optimal long-run average revenue under complete information. In the remainder of the paper we define \mbox{$N^* \in \arg\max_{N \in [N_{\max}]} \sum_{k\in[K]}\rho_kR_k(N)$} to be an optimal redemption threshold, breaking ties arbitrarily.

\smallskip

\begin{definition}[Regret]\label{def:regret}
Given a sample of customers $\mathcal{M}$ with purchase and redemption decisions governed by $\beta$, the \emph{regret} of policy $\pi \in \Pi$ is defined as:
\begin{align*}
& \regret(\policy, M, T)
= MT\mathcal{R}^{\text{non-pers}}- \sum_{t\in[T]} M R(N_t^\pi).
\end{align*}
\end{definition}

\smallskip

The notion of regret defined above can be thought of as a type of {\it counterfactual} regret. In particular, recall that, for any given period $t \in [T]$, $ R(N_t^\pi)$ is the long-run average revenue collected by the decision-maker per customer, if she had set $N_t^\pi$ for all $M$ customers in perpetuity (i.e., her {\it counterfactual} revenue). Therefore, the per-customer regret in period $t$, $\mathcal{R}^{\text{non-pers}}-R(N_t^\pi)$, quantifies the long-run average cost incurred by setting a sub-optimal threshold in a given period. Note that this regret metric differs from the one commonly used for the problem of pricing under demand uncertainty, where policies are evaluated according to the expected revenue collected throughout the horizon \citep{den2015dynamic}. We formally define the analogous notion of regret in our setting, which we refer to as the {\it observable regret}, below. 

\smallskip 

\begin{definition}[Observable Regret]\label{def:actual-regret}
Given a sample of customers $\mathcal{M}$ with purchase and redemption decisions governed by $\beta$, the \emph{observable regret} of policy $\pi \in \Pi$ is defined as:
\begin{align}\label{eq:actual-regret}
\text{Obs-Regret}(\pi,M,T) = MT\mathcal{R}^{\text{non-pers}}- \sum_{t\in[T]}\sum_{k\in[K]}\sum_{j\in\mathcal{M}_k} \phi_k(\tau_{jt}^\pi)\mathds{1}\{\tau_{jt}^\pi > 0\}.
\end{align}
\end{definition}

\smallskip

We argue that the counterfactual regret {defined in Definition~\ref{def:regret}} is a more reasonable metric than observable regret in our setting. One reason for this is that the optimal long-run average revenue per customer $\mathcal{R}^{\text{non-pers}}$, which is what the decision-maker truly cares about, need not be an upper bound on the expected revenue collected throughout a finite horizon. To see this, consider an instance where the seller interacts with a single customer. In this case, it is easy to construct instances for which a policy that sets the redemption thresholds such that the customer is always exactly one point away from redemption can far outperform the benchmark $T\mathcal{R}^{\text{non-pers}}$. While such a policy may generate high revenues throughout the horizon, it fails to achieve our goal of learning the optimal redemption threshold.

Moreover, even when $MT\mathcal{R}^{\text{non-pers}}$ is a valid upper bound on the seller's expected revenue, we claim that the observable regret remains an unfair metric against which to evaluate policies. This is due to the fact that, for any policy $\pi$, there exists some unavoidable finite-time convergence error relative to the long-run average revenue. To make this more concrete, suppose the decision-maker knew $\beta$. In this case, she would be able to compute the optimal redemption threshold exactly, and set this threshold in each period (or not offer a loyalty program at all). However, the decision-maker would still incur strictly non-zero observable regret, simply because $T$ is finite. This example highlights the key issue with the observable regret metric: it confounds the loss due to incomplete information about customers' redemption preferences with the loss due to the finite-time convergence error of the underlying Markov chain. This latter source of loss, which we refer to as the {\it mixing loss},  is uncorrelated with the quality of a learning algorithm. We formally define the mixing loss below.

\smallskip 

\begin{definition}[Mixing Loss]\label{def:mixing-loss}
Given a sample of customers $\mathcal{M}$ with purchase and redemption decisions governed by $\beta$, the \emph{mixing loss} of policy $\pi$ is given by:
\begin{align}
\mixingloss = \sum_{t\in[T]}\sum_{k\in[K]}\bigg[\rho_k M R_k(N_t^\pi) - \sum_{j\in\mathcal{M}_k} \phi_k(\tau_{jt}^\pi)\mathds{1}\{\tau_{jt}^\pi > 0\}\bigg].
\end{align}
\end{definition}

\smallskip 

While the mixing loss will not be our performance metric, it remains of independent interest, as it quantifies at a high level the ``closeness'' of the system to stationarity. A small mixing loss (in the absolute sense) reflects a system that is reflective of the steady-state system over which the decision-maker optimizes. In fact, for the policies analyzed in \Cref{sec:ub,sec:extension} we will additionally show that the corresponding mixing loss is vanishing with respect to $T$. Noting that the observable regret is the sum of these two terms, our results immediately imply bounds on our policies' observable regret.

\subsubsection*{Additional notation.} In the remainder of the paper we use Big O notation to denote the scaling with respect to $T$. Moreover, $\widetilde{O}(\cdot)$ is used to indicate the presence of polylogarithmic factors with respect to $T$.

\subsection{Regret Lower Bound}\label{ssec:lb}

With our main performance metric in hand, one of our goals will be to design learning algorithms that achieve low regret in expectation, where regret is defined as in \Cref{def:regret}. Prior to designing such policies, it is natural to characterize the complexity of the problem by providing a  lower bound on the regret the decision-maker can hope to achieve, as we scale $M$ and $T$. \Cref{thm: lower bound} provides such a lower bound.

\begin{theorem} \label{thm: lower bound}
    For any policy $\policy$, there exists an instance such that
    \begin{align*}
        \expect_\policy \left[\text{Regret}(\pi, M, T)\right] \ge \frac{\exp(-1/2)}{160(1+\sqrt{2})}\sqrt{MT} .
    \end{align*}
\end{theorem}

We defer the proof of \Cref{thm: lower bound} to Appendix \ref{apx:lb}. To prove the lower bound, we construct two instances, each with $K = 1$ and $N_{\max} = 2$. In the first instance, the optimal action is to set a redemption threshold of $N^* = 1$, whereas in the second the optimal action sets a redemption threshold of $N^* = 2$. The instances are constructed such that the true GLM parameters are within $\Theta(1/\sqrt{MT})$ of each other, making them difficult enough to identify while inducing large enough regret if they are not identified correctly.  We show that such a construction ensures that, in the worst case, any policy chooses the incorrect threshold with constant probability, thereby incurring an $\Omega(1/\sqrt{MT})$ revenue loss per period, per customer. This results in a lower bound of $\Omega(\sqrt{MT})$ regret.

\revedit{\section{A Devaluation-Free Algorithm}\label{sec:learning-algs}

\Cref{thm:price-of-fairness} established the important insight that personalization cannot generate arbitrary revenue gains in points-based rewards programs. In this section we study the impact of temporal consistency constraints in incomplete-information settings. Specifically, {we aim to design policies that achieve $\widetilde{O}(\sqrt{MT})$ regret while also being {\it devaluation-free}, that is, only ever {\it decreasing} the redemption threshold. As a first step towards this goal, we} analyze a simple algorithm which --- while not devaluation-free --- matches the lower bound in \Cref{thm: lower bound} (up to polylogarithmic factors). We will then build upon this algorithm and its analysis {to design a devaluation-free algorithm that never devalues customers' points via threshold increases. Methodologically, this is the most challenging aspect of our work and represents an important technical contribution.}

\subsection{Warmup: A Greedy Learning Algorithm}\label{sec:ub}

\subsubsection{Algorithm description.} In building toward a devaluation-free learning algorithm, we first propose a greedy epoch-based algorithm, similar to the one proposed by \citet{javanmard2019dynamic}. Specifically, our algorithm, which we call ``\revedit{Greedy-MLE},'' takes as input a set of epochs of geometrically increasing length. At the beginning of each epoch $h$, our algorithm computes the Maximum Likelihood Estimate (MLE) of the true parameters $\beta$ using the history of purchase and redemption decisions in the previous epoch.\footnote{Under \Cref{asp:glm}, the log-likelihood is strictly concave \citep{filippi2010parametric}. Therefore, the MLE is unique and can be efficiently computed.} 
 Given the MLE, denoted by $\hat{\beta}^{(h)}$, it then computes the redemption threshold that maximizes the long-run average revenue, assuming that $\hat{\beta}^{(h)}$ is the true parameter. Abusing notation, we use $N_h$ to denote this greedy threshold. In order to account for the possibility that the optimal action is to not offer a rewards program altogether, our algorithm compares the revenue without a rewards program to the estimated optimal revenue under $\hat{\beta}^{(h)}$. (Recall, we assume that the revenue without a rewards program is known.) If the former revenue exceeds the latter by some epoch-specific confidence threshold $\Delta_h$, we terminate the rewards program until the end of the horizon; otherwise, we set the redemption threshold to be $N_h$ throughout the entire epoch. {The data gathered is then used to update the MLE in the next epoch.} We provide a formal description of \revedit{Greedy-MLE} in \Cref{alg:greedy}. Given a predetermined epoch schedule, we let $H(t)$ be the epoch that time $t$ is in. For $h\in [H(T)]$, $\tcal_h$ denotes the set of periods contained in epoch $h$, with $T_h = |\tcal_h|$. As in the proof of \Cref{thm: lower bound}, for clarity of exposition we abuse notation and let $R(N;\beta')$ be the long-run average revenue under redemption threshold $N$, given that the true parameter is $\beta'$. Note that, by definition, $R(N) = R(N;\beta)$.

We highlight that, while this algorithm is not devaluation-free, it has the desirable feature of infrequently changing the redemption threshold. By fixing this threshold over epochs of geometrically increasing length, the algorithm allows customers to complete multiple redemption cycles before a change in goal. Such stability is important in practical settings where customers repeatedly interact with a decision-maker, as frequent, unpredictable changes can undermine customer trust. Moreover, in our numerical experiments we will see that the greediness of our algorithm allows for faster convergence to a fixed threshold, with significantly fewer than $H(T)$ changes in practice.

\begin{algorithm}[t]
\caption{\revedit{Greedy-MLE}}\label{alg:greedy}
\begin{algorithmic}[1]
\STATE \textbf{Input:} Initial redemption goal $N_1$, epoch schedule $\tcal_h, h \in [H(T)]$, epoch-specific termination thresholds $\Delta_h, h \in [H(T)]$
\FOR{$t \in \tcal_1$}
\STATE Set redemption goal $N_t^\pi = N_1$.
\FOR{$j \in [M]$}
\STATE Observe purchase decision $X_{jt}$ and points until redemption $\tau_{jt}$.
\ENDFOR
\ENDFOR
\FOR{$h\in \{2, 3, \ldots, H(T)\}$}
\FOR{$k\in[K]$}
\STATE Compute the maximum likelihood estimate of $\beta_k$ using samples collected from all type-$k$ customers in epoch $h-1$. That is, solve:
\begin{align}\label{eq:mle}
\hat{\beta}_k^{(h)} = \arg\max_{\beta_k \in \Theta_k} \mathcal{L}_{k}^{(h-1)}(\beta_k),
\end{align}
where 
\[\mathcal{L}_{k}^{(h-1)}(\beta_k) = \sum_{t \in \mathcal{T}_{h-1}} \sum_{j \in \mathcal{M}_k}\mathds{1}\{X_{jt}=1\}\log\left(\mu_k(\beta_{k,1}+\beta_{k,2}\tau_{jt})\right)+\mathds{1}\{X_{jt}=0\}\log\left(1-\mu_k(\beta_{k,1}+\beta_{k,2}\tau_{jt})\right).\]
\ENDFOR
\STATE Given $\hat{\beta}^{(h)} = (\hat{\beta}_1^{(h)},\ldots,\hat{\beta}_K^{(h)})$, compute an optimal redemption goal for epoch $h$: \[\threshold_h \in \arg\max_{N\in [N_{\max}]} R(N; \hat{\beta}^{(h)}).\] 
\STATE If $R(+\infty) > R(N_h; \hat{\beta}^{(h)}) + \Delta_h$, terminate, setting $N_{h'} = +\infty$ for all $h' \geq h$.
\FOR{$t\in \mathcal{T}_h$}
\STATE Set redemption goal $N_t^\pi = N_h$.
\FOR{$j \in [M]$}
\STATE Observe purchase decision $X_{jt}$ and points until redemption $\tau_{jt}$.
\ENDFOR
\ENDFOR
\ENDFOR
\end{algorithmic}
\end{algorithm}

\subsubsection{Regret guarantee.}

Before stating the \revedit{Greedy-MLE} algorithm's regret guarantee, we introduce some additional notation. For any $N \in [N_{\max}]$, \mbox{$k \in [K]$}, consider the Markov chain representing the points to redemption of a type-$k$ customer, given redemption threshold $N$. Recall from \Cref{prop:closed-form-lr-avg-per-cust} that $\statprob{k}{\tau}{N}$ is used to denote the steady-state probability that this Markov chain is in state $\tau$. For $t \in \mathbb{N}^+$, we use $P_k^t(\tau_0,\cdot;N)$ to denote the $t$-step transition probability of this Markov chain, given initial number of points to redemption $\tau_0$. We moreover let \mbox{$d_k(t; N) = \max_{\tau_0 \in \{0,\ldots,N\}}\|P_k^t(\tau_0,\cdot;N)-\statprob{k}{\tau}{N}\|_{\text{TV}}$} be the Markov chain's $t$-step total variation (TV) distance from stationarity. Finally, we define $t_{mix,k}(N) = \inf\left\{t \in \mathbb{N}^+ \mid d_k(t; N) \leq 1/4 \right\}$ to be the mixing time of this Markov chain, with $t_{mix} = \max_{k \in [K], N \in [N_{\max}]} t_{mix,k}(N)$.

\Cref{thm:main-thm} bounds the regret of \revedit{Greedy-MLE} for an epoch schedule of geometrically increasing length. In order to highlight the dependence of our algorithm's guarantees on the most salient quantities we defer explicit definitions of constants to the proof of the theorem. 
\begin{theorem}[\revedit{Greedy-MLE} Regret]\label{thm:main-thm}
    Fix $\delta \in (0,1)$, and {let $\tmixhat$ be any known upper bound on $t_{mix}$.} There exist known positive constants $C_1,\ldots,C_5$ such that, under
    the following epoch schedule\footnote{We assume $T_1$ is integral for simplicity. All results go through by rounding up to the nearest integer.} and termination thresholds:
    \begin{align}
    &T_1 = \max\braces{\frac{C_1}{1-2^{-1/\tmixhat}},\frac{C_2+C_3\log(1/\delta)}{M},\frac{C_4{\tmixhat} \log(1/\delta)}{M}} \qquad T_h = 2^{h-1} T_1 \  \forall \ h\in[H(T)], \label{eq: epoch-main} \\
    &\Delta_h = C_{5}\sqrt{\frac{\log(1/\delta)}{MT_{h-1}}} \quad \forall \ h \in \{2,\ldots,H(T)\},\label{eq: threshold-main}
    \end{align}
    with probability at least $1-{7K}H(T)\delta$, \Cref{alg:greedy} guarantees{, for all $N_1 \in [N_{\max}]$}: 
    \begin{align*}
    \regret(\policy, M, T) \leq MT_1\mu_{\max} + \frac{{12\mu_{\max}^3L_\mu}\sqrt{{3(1+N_{\max}^2)}}}{\mu_{\min}^3\kappa} \left(\sum_{k\in[K]}\sqrt{\rho_k}\right)\left(\sum_{h=2}^{H(T)}\sqrt{T_h}\right)\sqrt{M \log(1/\delta)}
    \end{align*}
\end{theorem}
\smallskip 

Observe that the above construction requires an upper bound on the worst-case mixing time $t_{mix}$. This quantity, however, is a priori unknown to the decision-maker, given its dependence on the purchase and redemption probabilities that she seeks to learn. \Cref{prop:tmix-ub} provides a constant upper bound on $t_{mix}$.

{
\begin{proposition}\label{prop:tmix-ub}
$t_{mix} \leq  \frac{(N_{\max}+1)^2}{2(1-\mu_{\max})\mu_{\min}}.$
\end{proposition}

We defer the formal proof of \Cref{prop:tmix-ub} to Appendix \ref{apx:tmix-ub}. Note that, for any given threshold $N$, the Markov chain representing the number of points to redemption is a lazy, state-dependent directed random walk on an $(N+1)$-cycle. Despite the added complexity of state-dependency, \Cref{prop:tmix-ub} recovers the fact that, for lazy undirected random walks on a cycle, the mixing time has quadratic dependence on the number of nodes in the cycle \citep{levin2017markov}. We prove this upper bound via coupling, reducing the problem to that of analyzing the absorption time of a more tractable Gambler's Ruin problem.
}

\smallskip 

Leveraging the fact that $t_{mix}$ is upper bounded by a constant, \Cref{thm:main-thm} implies that our algorithm achieves the lower bound derived in \Cref{thm: lower bound}. 

\begin{corollary}\label{cor:exp-regret}
   Let $\delta = O(1/T)$. Under the epoch schedule and termination thresholds specified in \Cref{thm:main-thm}, \Cref{alg:greedy} guarantees:
   {
   \[\mathbb{E}_{\pi}\left[\regret(\pi,M,T)\right] = \widetilde{O}(\sqrt{MT}+M).\] 
   }
\end{corollary}

\revedit{We defer the straightforward proof of the corollary to Appendix \ref{apx:exp-regret}.} Additionally, in Appendix \ref{apx:mixing-loss} we establish a high-probability bound of $\widetilde{O}(M+\sqrt{MT})$ on our algorithm's mixing loss (see \Cref{def:mixing-loss}). Putting this bound together with the high-probability bound on our algorithm's regret, we obtain a high-probability bound of $\widetilde{O}({M}+\sqrt{MT})$ on our algorithm's observable regret (see \Cref{def:actual-regret}).  The bound on our algorithm's mixing loss follows from the geometrically increasing construction of the epoch schedule, which allows the system to approach stationarity as the epoch length grows. Such a result can be thought of as a Chernoff-type bound for the Markov chains induced by our algorithm, which naturally has a dependence on the chains' respective mixing times. {We moreover note that the linear dependence on $M$ here is to be expected, given that we are union bounding the distance to stationarity for $M$ independent Markov chains.}

\subsubsection{Proof sketch.} We conclude this section with a proof sketch of \Cref{thm:main-thm}, deferring a formal proof to Appendix \ref{ssec:proof-of-main-thm}. The proof is partitioned into two cases, depending on whether or not the no-loyalty option is optimal (i.e., whether $R(+\infty) \geq R(N^*)$).

\textbf{Case 1: $R(+\infty) < R(N^*)$.} In this case, we bound our algorithm's regret by the following two sources of loss, for all $h \in [H(T)]$: (i) the loss incurred from optimizing according to $\hat{\beta}^{(h)}$ instead of $\beta$ prior to termination, given by $\max_{N\in[N_{\max}]} |R(N;\beta)-R(N;\hat{\beta}^{(h)})|$, and (ii) the loss incurred from incorrect early termination, given by $R(N^*)-R(+\infty)$. The bulk of the proof lies in bounding the first source of loss. In particular, we establish that, with sufficiently high probability, for all epochs $h$ prior to termination:
\begin{align}\label{eq:rev-estimation-error}
\max_{N \in [N_{\max}]} \abs{R(N;\beta)-R(N;\hat{\beta}^{(h)})} \leq \Delta_h = O\left(\sqrt{\frac{\log(1/\delta)}{MT_h}}\right).
\end{align}

The above inequality is the main driver of our algorithm's regret guarantee. To prove this result, we first establish the straightforward fact that the revenue loss from mis-estimation is proportional to the MLE estimation error. We then leverage recent results from the generalized linear contextual bandits literature \citep{li2017provably} by interpreting the points to redemption $\tau_{jt}$ as ``contexts,'' and establishing that, with probability at least $1-3\delta$, the MLE estimation error for each type $k$ is upper bounded by $O\left(\sqrt{\frac{\log(1/\delta)}{\lambda_{\min}(V_k^{(h)})}}\right)$, where $\lambda_{\min}(V_k^{(h)})$ denotes the minimum eigenvalue of the empirical Fisher information matrix $V_k^{(h)}$ induced by type-$k$'s purchase decisions throughout the epoch. 

As noted in \Cref{sec:intro}, the requirement that $\lambda_{\min}(V_k^{(h)})$ is lower bounded for the MLE to have low estimation error is well-known in the literature.
In settings such as pricing or contextual bandits, achieving a linear growth in the minimum eigenvalue of the design matrix (with high probability) typically requires either exogenous assumptions on the distribution from which contexts are independently drawn, or an algorithm that actively explores to generate sufficient diversity in the observed contexts.
In contrast to these settings, however, in our case, the ``contexts'' $\tau_{jt}$ are {\it endogenous} to the threshold chosen by our policy within a given epoch, since they are induced by the Markov chain governing the remaining points to redemption. Such a departure from the standard literature necessitates a different approach in establishing a lower bound on $\lambda_{\min}(V_k^{(h)})$, making it the primary technical contribution of this section. At a high level, we prove that $\lambda_{\min}(V_k^{(h)})$ grows {\it linearly} in the sample variance of the observations collected during epoch $h$, which, by our choice of parameters, grows linearly in $\rho_k M T_h$, the type-$k$ sample size of epoch $h$, with high probability. Deriving a lower bound for this sample variance is the key departure from existing work. In particular, we show that for each type $k \in [K]$, the natural variability of the $\rho_kM$ Markov chains that run throughout the epoch guarantees the required lower bound on the sample variance. One can then think of the Markov chain as providing ``natural exploration'' for our algorithm. From a technical perspective, bounding this sample variance, which otherwise would follow from a simple application of Hoeffding's inequality in the i.i.d. setting \citep{blm}, requires us to derive an explicit Chernoff-type bound for the Markov chain of each type-$k$ customer. 

We conclude the proof by arguing that under the event defined by \Cref{eq:rev-estimation-error}, the algorithm never incorrectly terminates. Then, summing the loss specified in \Cref{eq:rev-estimation-error} over all epochs $h$, we obtain our regret guarantee for this case.

\textbf{Case 2: $R(+\infty) \geq R(N^*)$.} In this case, the algorithm only incurs regret prior to termination. In those epochs, $R(+\infty) \leq R(N_h;\hat{\beta}^{(h)}) + \Delta_h$ by construction. Moreover, by  \Cref{eq:rev-estimation-error}, $R(N_h;\hat{\beta}^{(h)}) \leq R(N_h;\beta) + O\left(\sqrt{\frac{\log(1/\delta)}{MT_h}}\right)$. Using the instantiation of $\Delta_h$ specified in \Cref{eq: threshold-main} and summing over all epochs $h$, we obtain the regret guarantee for this case as well. \hfill\Halmos

\subsection{Main Algorithm: \revedit{\DfGreedy}}\label{sec:extension}

While the Greedy-MLE policy achieves the optimal rate of regret for the learning problem, it has no guardrails surrounding how the thresholds change from epoch to epoch. In this section, we show that a simple modification to \revedit{Greedy-MLE} satisfies the desideratum of never devaluing customers' points by increasing the threshold, all the while only losing a constant factor of two in its regret guarantee.

Our proposed algorithm, which we call \revedit{\DfGreedy}, is a semi-greedy elimination-style algorithm. Similar to \Cref{alg:greedy}, it proceeds in epochs of geometrically increasing length, computing the MLE $\hat{\beta}^{(h)}_k$ for each type $k \in [K]$, in each epoch $h \in \{2,\ldots,H(T)\}$. However, rather than greedily choosing the threshold for that epoch with respect to the estimated revenue under $\hat{\beta}^{(h)}$, it cautiously chooses the largest threshold within an epoch-specific {\it consideration set} of thresholds. These epoch-specific consideration sets are iteratively defined: each corresponds to the set of all thresholds in the previous consideration set that guarantee an estimated revenue that is within $2\gaph$ of the greedy revenue in the last consideration set, for some appropriately defined $\gaph$. The nestedness of the consideration sets throughout the horizon guarantees that our algorithm's choice of thresholds is non-increasing. We formally present the algorithm in \Cref{alg:non-increasing}.

\begin{algorithm}[t]
\caption{\revedit{\DfGreedy}}\label{alg:non-increasing}
\begin{algorithmic}[1]
\STATE \textbf{Input:} Initial redemption goal $N_1 = N_{\max}$, initial consideration set $\mathcal{N}_1 = [N_{\max}]$, epoch schedule $\tcal_h, h \in [H(T)]$, epoch-specific termination thresholds $\Delta_h, h \in [H(T)]$
\FOR{$t \in \tcal_1$}
\STATE Set redemption goal $N_t^\pi = N_1$.
\FOR{$j \in [M]$}
\STATE Observe purchase decision $X_{jt}$ and points until redemption $\tau_{jt}$.
\ENDFOR
\ENDFOR
\FOR{$h\in \{2, 3, \ldots, H(T)\}$}
\FOR{$k\in[K]$}
\STATE Compute the type-$k$ MLE $\hat{\beta}^{(h)}_k$ using samples collected from all type-$k$ customers in epoch $h-1$ (see \Cref{eq:mle}).
\ENDFOR
\STATE Given $\hat{\beta}^{(h)} = (\hat{\beta}_1^{(h)},\ldots,\hat{\beta}_K^{(h)})$, compute the epoch-$h$ consideration set $\mathcal{N}_h$:
\begin{align}\label{eq:consideration-set}
 \mathcal{N}_h = \braces{N\in \mathcal{N}_{h-1}:  R(N; \hat{\beta}^{(h)}) \ge \max_{N\in \mathcal{N}_{h-1}} R(N; \hat{\beta}^{(h)}) - 2\gaph}.   
\end{align}
\STATE Let $N_h = \max_{N\in\mathcal{N}_h} N$.
\STATE If $R(+\infty) > R(N_h; \hat{\beta}^{(h)}) + 3\Delta_h$, terminate, setting $N_{h'} = +\infty$ for all $h' \geq h$.
\FOR{$t\in \mathcal{T}_h$}
\STATE Set redemption goal $N_t^\pi = N_h$.
\FOR{$j \in [M]$}
\STATE Observe purchase decision $X_{jt}$ and points until redemption $\tau_{jt}$.
\ENDFOR
\ENDFOR
\ENDFOR
\end{algorithmic}
\end{algorithm}

Notice that \Cref{alg:non-increasing} is cautious on two fronts. First, it sets the largest threshold within the consideration set $\mathcal{N}_h$ in each epoch, as opposed to the optimal threshold amongst {\it all} possible thresholds. In addition to this, it is cautious with respect to the termination condition. Indeed, it requires $R(+\infty)$ to exceed the largest threshold in the consideration set by $3\Delta_h$, as opposed to $\Delta_h$, as in \Cref{alg:greedy}. This additional source of cautiousness protects against any additional sub-optimality caused by not choosing the greedy threshold, and ensures that when we terminate we can be confident the no-loyalty scheme is optimal. \Cref{thm:extension} shows that, despite these two potentially sub-optimal changes, this practical modification only results in a factor of two loss relative to the greedy algorithm. As a result, we retain the expected regret bound of $\widetilde{O}(\sqrt{MT}+M)$ for $\delta = O(1/{T})$, implying that devaluation-freeness comes essentially for free in our setting.\footnote{We omit an analysis of the mixing loss of \Cref{alg:non-increasing}, as it is identical to that of \Cref{alg:greedy}.}

\begin{theorem}[\revedit{\DfGreedy} Regret]\label{thm:extension}
    Fix $\delta \in (0,1)$. For the same epoch schedule and termination thresholds specified in \Cref{thm:main-thm},
    with probability at least $1-{7K}H(T)\delta$, \Cref{alg:non-increasing} guarantees: 
    \begin{align*}
    \regret(\policy, M, T) \leq MT_1\mu_{\max} + \frac{{24\mu_{\max}^3L_\mu}\sqrt{{3(1+N_{\max}^2)}}}{\mu_{\min}^3\kappa} \left(\sum_{k\in[K]}\sqrt{\rho_k}\right)\left(\sum_{h=2}^{H(T)}\sqrt{T_h}\right)\sqrt{M \log(1/\delta)}.
    \end{align*}
\end{theorem}

At a high level, one would not expect \Cref{alg:non-increasing} to be order-wise worse than \Cref{alg:greedy}. Intuitively, this follows from the dependence of the consideration set on $\Delta_h$, which enforces that the thresholds included in $\mathcal{N}_h$ generate closer revenue to the greedy threshold for epochs later on in the horizon. We provide a proof sketch of \Cref{thm:extension} below, deferring its formal proof to Appendix \ref{apx:extension}.

\subsubsection{Proof sketch.} We show that the regret incurred by our algorithm in each epoch $h \in \{2,\ldots,H(T)\}$ is upper bounded by $4T_h\gaph$. The final bound then follows from plugging in the definitions of $T_h$ and $\gaph$.

Suppose first that implementing a \bogo\ program is optimal, i.e., $R(N^*) > R(+\infty)$. As in the proof of \Cref{thm:main-thm}, we establish that, with high probability, \Cref{alg:non-increasing} does not mistakenly terminate. Therefore, it suffices to bound the loss incurred by choosing the largest threshold contained in the consideration set $\mathcal{N}_h$, instead of choosing the optimal $N^*$. In \Cref{thm:main-thm}, we were able to bound this loss since, in each epoch, our algorithm greedily chose the best threshold over {\it all} possible thresholds $N \in [N_{\max}]$, thus guaranteeing high revenue relative to $R(N^*;\hat{\beta}^{(h)})$. The result then followed from the fact that, with high probability, the revenue under the MLE $\hat{\beta}^{(h)}$ was close enough to the revenue under the true parameter $\beta$; therefore, optimizing with respect to the incorrect parameters did not introduce too much regret. Under \Cref{alg:non-increasing}, however, relating the algorithm's choice of threshold $N_h$ to $N^*$ is not as straightforward. This is because the algorithm chooses from a non-increasing consideration set of thresholds, which a priori need not include $N^*$. We however show that, with high probability, $N^*$ is {never} eliminated from the algorithm's consideration set. As a result, the estimated revenues under $N^*$ and $N_h$, respectively, are within 2$\gaph$ of each other in each period, by \Cref{eq:consideration-set}. Our previous result bounding the revenue loss due to the MLE's estimation error (see \Cref{eq:rev-estimation-error}) then gives us our final per-epoch regret bound of 4$T_h\gaph$. 

In the case where the no-loyalty option is optimal (i.e., $R(+\infty) \geq R(N^*)$), \Cref{alg:non-increasing} incurs regret in all epochs for which it hasn't satisfied the termination condition. By construction, however, it must have been that the estimated revenue under $N_h$ was at least within $3\gaph$ of the no-loyalty revenue. The additional additive gap of $\gaph$ follows from the estimated loss under the MLE $\hat{\beta}^{(h)}$, again by \Cref{eq:rev-estimation-error}.
$\square$
}

\medskip 

\revedit{
We conclude this section by discussing our regret guarantees' dependence on three salient quantities: the size of the sampled population $M$, the worst-case mixing time $t_{mix}$, and the maximum threshold $N_{\max}$. 

\paragraph{Dependence on $M$.} {The number of customers $M$ impacts various quantities used in our algorithm; consequently, our final regret bound also has dependence on $M$.} For fixed $T$ and $\delta$, the epoch lengths $T_h$ are non-increasing in $M$. This dependence is intuitive, as a larger value of $M$ implies that the decision-maker has more data in each period, reflecting the value of information sharing across customers. Additionally, the termination threshold $\Delta_h$ is decreasing in $MT_{h-1}$ (see \Cref{eq: threshold-main}), the effective sample size in the previous epoch. This reflects the fact that, for a larger number of observations, the algorithm has more confidence in the MLE $\hat{\beta}^{(h)}$, and therefore does not need to be as conservative with respect to terminating the rewards program. Finally, {the regret guarantee demonstrates} the positive effect of pooling on learning algorithms, as the expected regret per customer is $\widetilde{O}\left(\sqrt{T/M}+1\right)$ over the entire horizon, which is decreasing in the population size.

\paragraph{Dependence on $t_{mix}$.} At a high level, the linear dependence of the epoch lengths on $t_{mix}$ arises from the fact that, in order for the greedy solution to converge to the optimal redemption threshold, there must be sufficient variability in the observed points to redemption for the MLE $\hat{\beta}^{(h)}$ to be close to the true parameter $\beta$. We bound this variability by analyzing the variance of the steady-state distribution of each type-$k$ customer's Markov chain. The tightness of this approximation, however, relies on the system being close to stationarity, hence the dependence on $t_{mix}$.

\paragraph{Dependence on $N_{\max}$.} Finally, the proof of \Cref{cor:exp-regret} establishes that our algorithms' expected regret grows as $\widetilde{O}(N_{\max}\sqrt{MT})$, in the regime where $N_{\max}$ and $M$ are fixed and $T$ is scaled large. This dependence on $N_{\max}$ enters into our analysis through  \Cref{lem:connecting-to-min-eval}, which leverages existing results on generalized linear contextual bandits to prove that the MLE error associated with the purchase probability for each value of $\tau$ is linear in $N_{\max}$. Hence, improving this dependence on $N_{\max}$ seems a priori challenging. While we leave the task of establishing tightness as an interesting open question, in Appendix \ref{sec: nmax exp} we demonstrate that, while the practical performance of \revedit{Greedy-MLE} is relatively insensitive to $N_{\max}$, the linear dependence on $N_{\max}$ does indeed arise under \revedit{\DfGreedy} due to this parameter's first-order impact on the algorithm's initial consideration set.
}
\section{Computational Experiments}\label{sec:numerics}

In this section we conduct extensive numerical experiments to gain additional insights into the limited cost of consistency in \bogo\ programs. In particular, we study the \revedit{value of personalization} over a large set of randomly generated (as opposed to worst-case) instances. We moreover demonstrate the practical efficacy of our learning algorithms.

Except when specified, we let $N_{\max} = 20$, with purchase probabilities given by:
\[\phi_k(\tau) = \min\left\{\bar{\phi}_k\left(1 + \exp\left(\alpha_k - \beta_k\tau\right)\right), 1\right\} \quad \forall \ k \in [K], \]
with $\alpha_k > 0$, $\beta_k > 0$. We moreover let $K = 2$, with $\rho_1 = \rho_2 = 1/2$.

\subsection{On the Limited Value of Personalization}\label{ssec:pof-numerics}

\subsubsection{Distributional analysis of \revedit{value of personalization}.}\label{ssec:pof} 

For this set of experiments, we build upon the empirical findings described in \Cref{ssec:related-lit} and define the two types of customers as follows. We assume one type of customer is a frequent customer with high baseline purchase probability under the no-loyalty option; the other type is an infrequent customer who has a low baseline purchase probability but is very sensitive to the presence of a rewards program. We model such settings by randomly generating the parameters $(\bar{\phi}_k, \alpha_k, \beta_k), k \in [K]$, over 10,000 replications, as follows:
\begin{align}\label{eq:two-type-setup}
\begin{cases}
\bar{\phi}_1 \sim \text{Unif}[0.05,0.25]\\
\alpha_1 \sim \text{Unif}[3,3.5]\\
\beta_1 \sim \text{Unif}[1,1.5]
\end{cases} \hspace{2cm}\text{and} \hspace{2cm}  
\begin{cases}
\bar{\phi}_2 \sim \text{Unif}[0.5,0.75]\\
\alpha_2 \sim \text{Unif}[0.5,1]\\
\beta_2 \sim \text{Unif}[0,0.5]
\end{cases}
\end{align}

Note that any such randomly generated instance still represents a pessimistic (albeit no longer worst) case. The reason for this is that, when $K = 2$, we know that a worst-case instance is a ``frequent versus infrequent'' setting, pushed to the extreme of $\bar{\phi}_1 = 0$ and $\bar{\phi}_2 = 1$ (see proof of \Cref{thm:price-of-fairness}).

\begin{figure}
\begin{subfigure}[b]{0.45\textwidth}
\centering
\includegraphics[width=\textwidth]{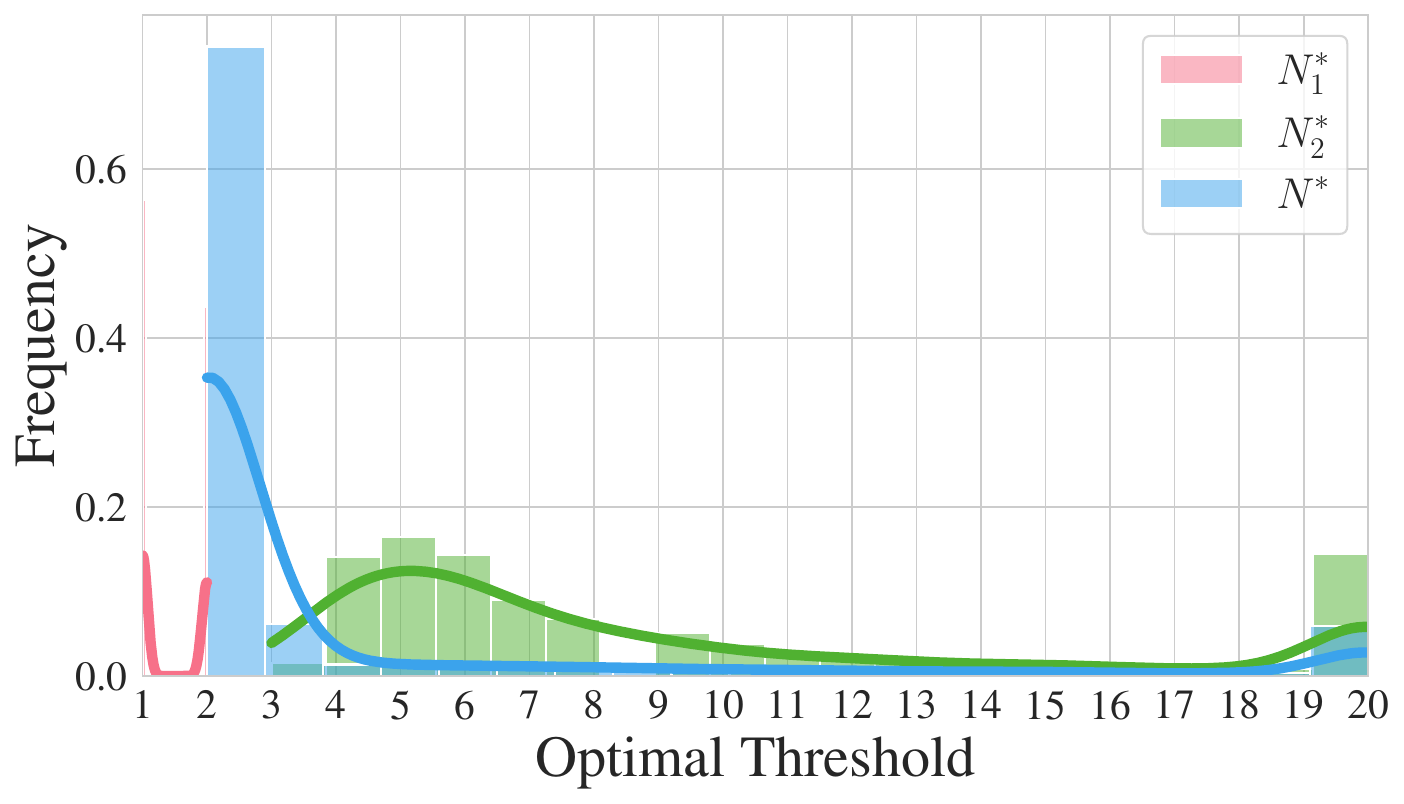}
\caption{\small Histogram of optimal thresholds \label{fig:opt_hist}}
\end{subfigure}
\qquad
\begin{subfigure}[b]{0.45\textwidth}
\centering
\includegraphics[width=\textwidth]{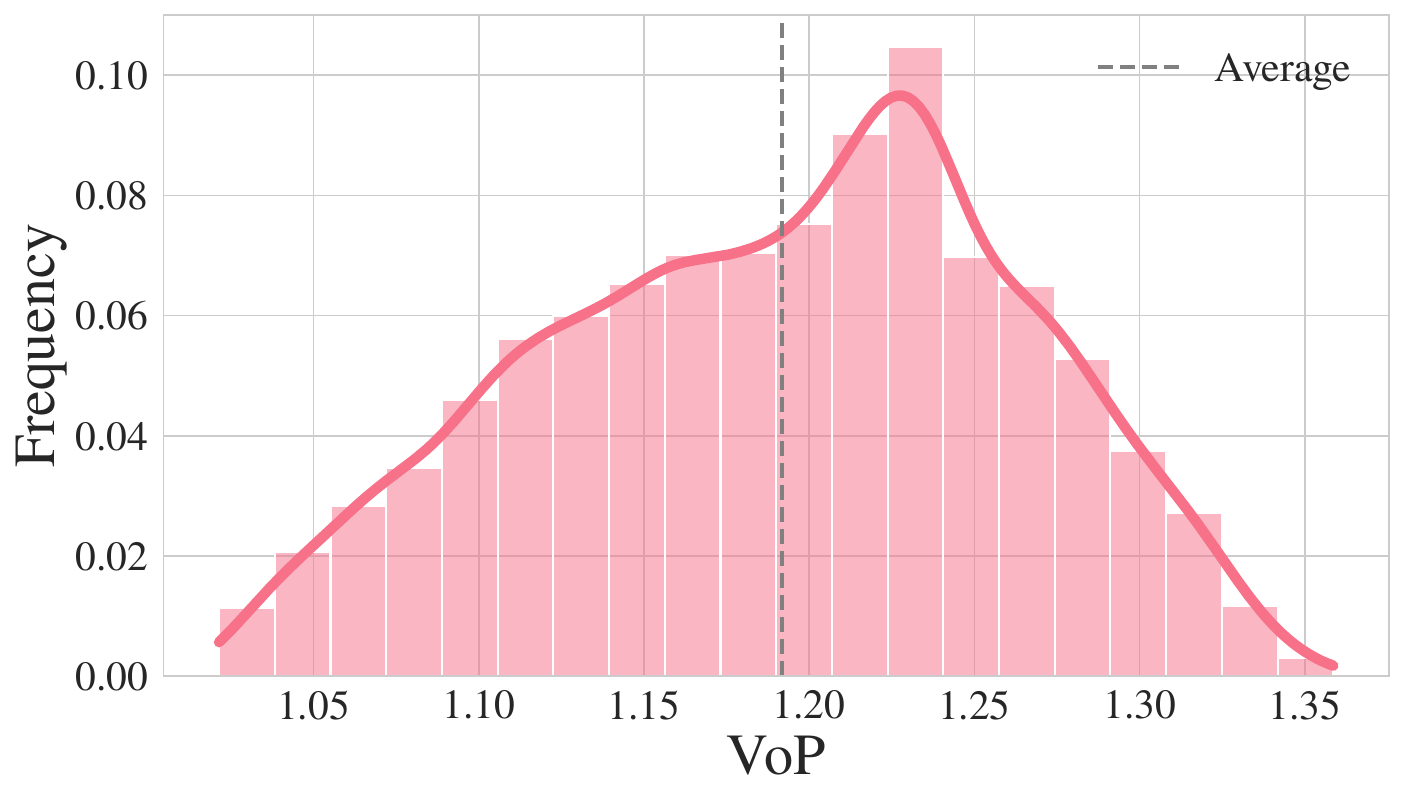}
\caption{\small {Histogram of \revedit{VoPs}
 \label{fig:pof_hist}}}
\end{subfigure}
\caption{{Distribution of the \revedit{value of personalization} and optimal thresholds across all 10,000 randomly generated instances described in \Cref{eq:two-type-setup}. 
In \Cref{fig:opt_hist}, $N_1^*, N_2^*$ respectively correspond to the optimal personalized threshold for type-1 and type-2 customers; $N^*$ corresponds to the optimal non-personalized threshold.
The dashed grey line in \Cref{fig:pof_hist} corresponds to the average \revedit{VoP} of 1.1990 across all instances.}
\label{fig:pof_opt_hist}}
\end{figure}

\Cref{fig:opt_hist}, which shows a histogram of the optimal personalized and non-personalized thresholds across all randomly generated instances, numerically illustrates the challenge presented by the ``frequent versus infrequent'' setup. In 100\% of instances, the optimal personalized threshold is at most two for type-1 customers; on the other hand, the optimal personalized threshold is always at least three for type-2 customers. The distribution of the type-2 optimal threshold $N_2^*$ moreover has a long tail: over 32\% of replications are such that $N_2^* \geq 10$. The optimal non-personalized threshold $N^*$ hedges between these conflicting incentives: $N^* \in \{2,3\}$ in 81\% of replications, and $N^* \geq 10$ in 11\% of replications. 

Despite the fact that $N^*$ is potentially far from both $N_1^*$ and $N_2^*$ in many instances, \Cref{fig:pof_hist} shows that the \revedit{value of personalization} is frequently much lower than the worst-case upper bound of 1.5 derived in \Cref{thm:price-of-fairness} for $K=2$. In particular, the average \revedit{value of personalization} is strictly less than 1.2, with 95\% of instances yielding a \revedit{value of personalization} of at most 1.31, and the maximum \revedit{value of personalization} across all replications being 1.36. These results suggest that, in practice, when customers' preferences are more closely aligned (i.e., less extreme baseline purchase probabilities and rewards program sensitivities), the \revedit{value of personalization} is expected to be significantly lower than 1.2.

\subsubsection{Impact of heterogeneity.}\label{ssec:heterog} \revedit{We next investigate the impact of {customer} heterogeneity on the \revedit{value of personalization} {by varying} the proportion of type-1 customers in the population, for the same setup as the one used in \Cref{ssec:pof}.}

When $K = 2$, the setting where $\rho_1 = 0.5$ can be interpreted as one in which the population is very heterogeneous. As we approach the extremes of $\rho_1 = 0$ and $\rho_1 = 1$, however, the population becomes more homogeneous, with one type dominating the other. \Cref{tab:rho1}, which reports the average and maximum \revedit{VoP} across all randomly generated instances for  \mbox{$\rho_1 \in \{0.1,0.2,\ldots,0.8,0.9\}$}, demonstrates the impact of this type of heterogeneity.  We observe that both the average and maximum \revedit{VoP} increase for \mbox{$\rho_1 \in [0.1, 0.5]$}, and decrease thereafter. This numerically validates the intuition that \revedit{personalization is the most valuable} when types are equally likely, since the seller needs to simultaneously satisfy conflicting preferences (as shown in \Cref{fig:opt_hist}). 
\Cref{fig:rho-thresholds}, which shows the distribution of non-personalized thresholds for \mbox{$\rho_1 \in \{0.1, 0.5, 0.9\}$},  further illustrates this. We find that at the extreme of \mbox{$\rho_1 = 0.1$} in which most individuals are of type 2, the distribution of the optimal non-personalized threshold closely resembles the distribution of $N_2^*$ observed in \Cref{fig:opt_hist}. Similarly, when $\rho_1 = 0.9$, the distribution of the optimal non-personalized thresholds resembles that of $N_1^*$. Intuitively, since the seller only loses out on revenue from 10\% of customers in these cases, it should be optimal for the seller to optimize over the dominant type. This explains why we observe a \revedit{value of personalization} of less than 1.11 in the worst case, across both extremes.

\revedit{Finally, in Appendix \ref{apx:exper-k-types} we numerically study the impact of the number of types $K$ on the {value of personalization}.}

\begin{table}[t]
\centering
\begin{tabular}{@{}cccccccccc@{}}
\toprule
$\rho_1$         & {0.1} & {0.2} & {0.3} & {0.4} & {0.5} & {0.6} & {0.7} & {0.8} & {0.9} \\ \midrule
{Average \revedit{VoP}} & 1.0366       & 1.0788       & 1.1264       & 1.1737       & 1.1951       & 1.1813       & 1.1521       & 1.1091       & 1.0573       \\
{Max \revedit{VoP}}     & 1.0568       & 1.1278       & 1.2103       & 1.3074       & 1.4244       & 1.4035       & 1.3156       & 1.2130       & 1.1087       \\ \bottomrule
\end{tabular}
\caption{Dependence of \revedit{value of personalization}  on fraction of type-1 customers across randomly generated instances.}
\label{tab:rho1}
\end{table}

{
\begin{figure}[t]
\centering
\includegraphics[width=0.5\textwidth]{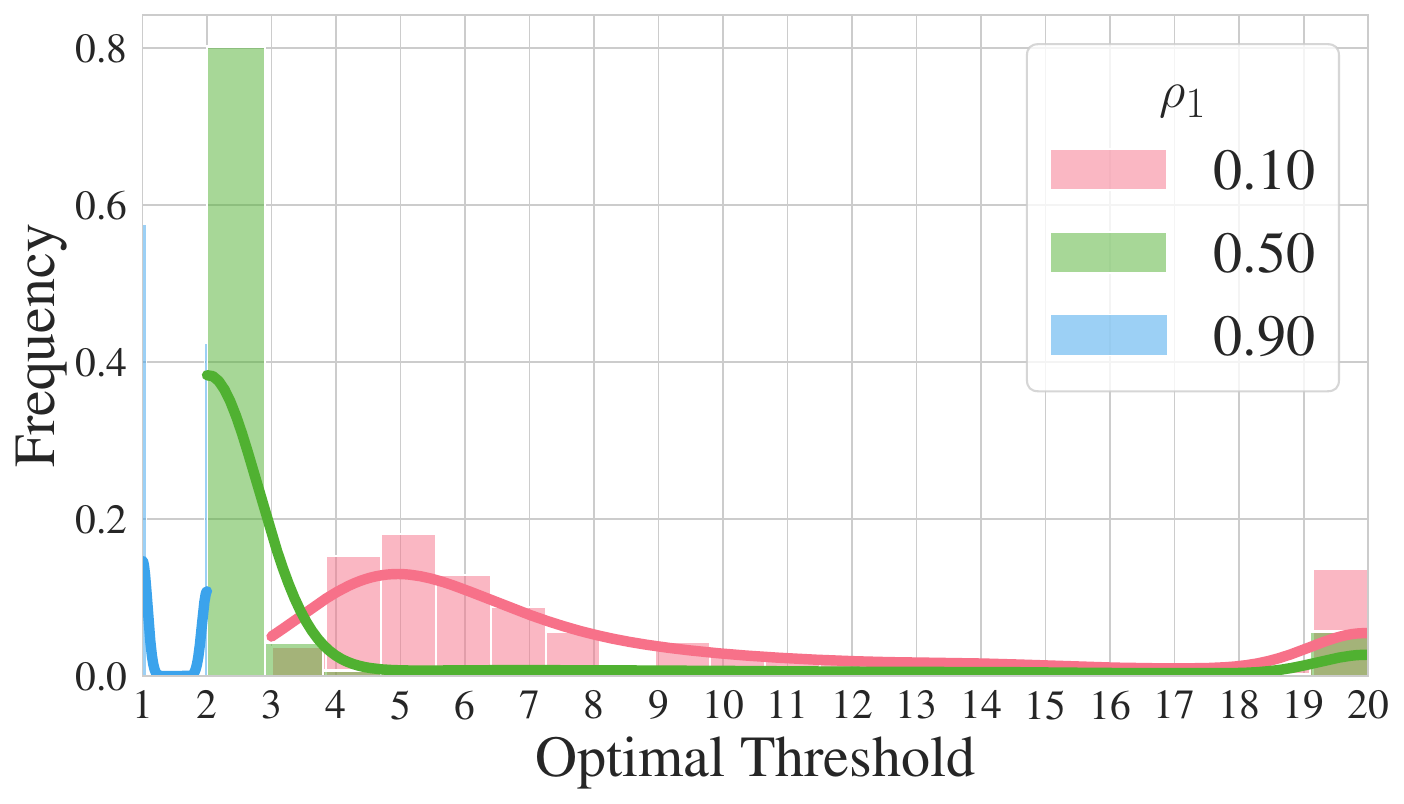}
\caption{\centering{\revedit{Histogram of optimal non-personalized thresholds as a function of $\rho_1$.}}
\label{fig:rho-thresholds}}
\end{figure}
}

\subsection{Learning Experiments}\label{ssec:learning-numerics}

We conclude the section by evaluating the numerical performance of our two algorithms on synthetic data. For both \revedit{Greedy-MLE} and \revedit{\DfGreedy}, we use a doubling epoch schedule with $T_1=1$, $T_h = 2^{h-1}T_1$, and set the termination thresholds to be $\Delta_h = \frac{0.15}{\sqrt{M(\sum_{i=1}^{h-1}T_i)}}$ for all $h \geq 1$. Additionally, we implement a practical modification of the algorithm that estimates the MLE using all data points collected up to the start of the current epoch, rather than only the data from the previous epoch.

\subsubsection{Regret comparison and learning behavior.}\label{ssec:num-regret} 

{We first compare the regret of our two algorithms in a setting where the decision-maker experiments with $M = 2$ customers, each of whom is of different type. In line with the ``frequent versus infrequent'' setting studied in \Cref{ssec:pof-numerics}, we consider an instance for which $\bar{\phi}_1 = 0.25$, $\bar{\phi}_2 = 0.5$, $\alpha_1 = 1.5$, $\alpha_2 = 0.05$, $\beta_1 = 1.5$, $\beta_2 = 0.05$. We run 100 replications for each experiment. 
}

\begin{figure}[t]
\begin{subfigure}[b]{0.45\textwidth}
\centering
\includegraphics[width=\textwidth]{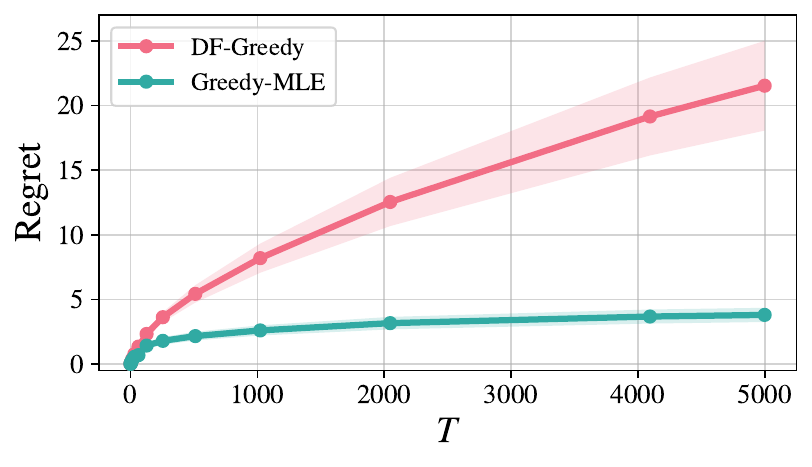}
\caption{\small Cumulative regret vs. $T$ \label{fig:regret-cumulative}}
\end{subfigure}
\hfill
\begin{subfigure}[b]{0.45\textwidth}
\centering
\includegraphics[width=\textwidth]{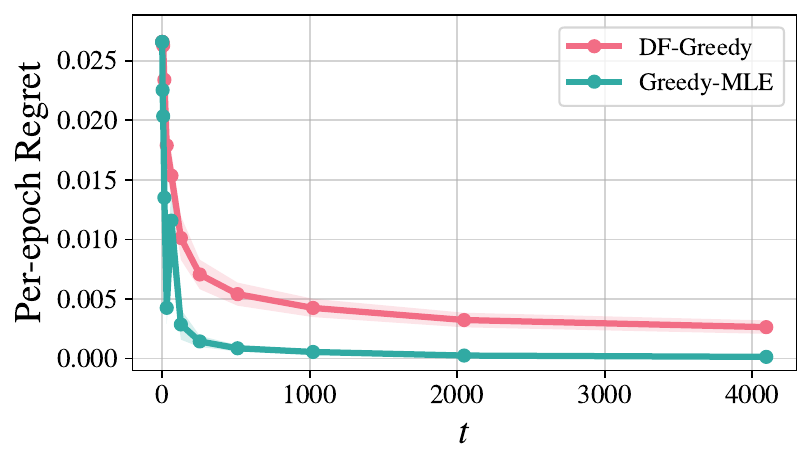}
\caption{\small {Per-epoch regret vs. $t$
 \label{fig:per-epoch-regret}}}
\end{subfigure}
\centering
\vspace{1mm}
\begin{subfigure}[b]{0.45\textwidth}
\includegraphics[width=\textwidth]{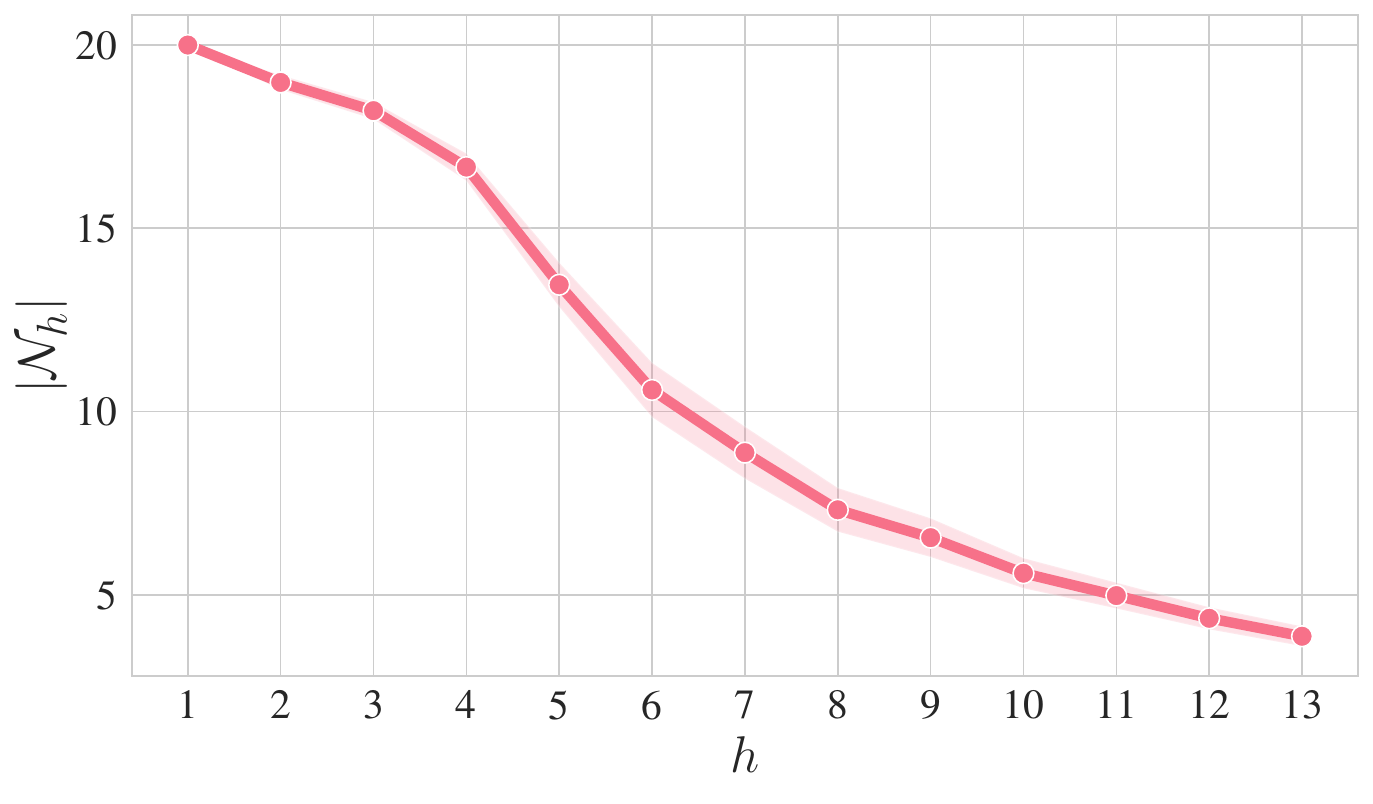}
\caption{\small {$|\mathcal{N}_h|$ vs. $h$ under \revedit{\DfGreedy}. 
 \label{fig:avg-consideration-set-size}}}
\end{subfigure}
\caption{{Performance of \revedit{Greedy-MLE} (\Cref{alg:greedy}) and \revedit{\DfGreedy} (\Cref{alg:non-increasing}). \cref{fig:regret-cumulative} plots the cumulative regret versus the horizon $T \in \{1, 2, 4, \ldots, 5000\}$. \cref{fig:per-epoch-regret} plots the average regret per period in each epoch, versus the start of each epoch on the $x$-axis. \cref{fig:avg-consideration-set-size} shows the average size of the consideration set $\mathcal{N}_h$ in each epoch $h$ under \revedit{\DfGreedy}, for $T = 5,000$. All results are averaged across 100 replications.}
\label{fig:regret}}
\end{figure}

\Cref{fig:regret-cumulative} shows the cumulative regret of \revedit{Greedy-MLE} and \revedit{\DfGreedy} as the horizon $T$ grows large. \revedit{These results numerically validate our theoretical findings: namely, that the regret of both algorithms is sublinear in $T$, with \revedit{\DfGreedy} exhibiting worse performance due to its monotonicity constraint. Additionally, \Cref{fig:per-epoch-regret} plots the average regret of each algorithm within each epoch for a fixed $T = 5,000$. We observe the same trend for both algorithms: the per-period regret in each epoch steeply decreases in the first few epochs, then gradually converges to $0$. These results illustrate the fast convergence of \revedit{Greedy-MLE} to the optimal threshold; \revedit{\DfGreedy} naturally exhibits a slower rate of convergence due to the fact that it constrains itself to choose a sub-optimal threshold within a larger consideration set, with on average three thresholds remaining within the consideration set, for $T = 5,000$ (see \Cref{fig:avg-consideration-set-size}). As a result, this algorithm requires more samples in order to eliminate high thresholds (to which it is constrained to never return) with confidence.

Our experiments illustrate the intuitive fact that \revedit{devaluation-free} constraints have an impact on the seller's revenue throughout the learning horizon.} On the flip side of this, \Cref{tab:fairness-learning-num} shows the major gains in stability that arise from these constraints. We first observe that \revedit{Greedy-MLE} exhibits the desideratum of limited adaptivity by only changing the threshold less than 6.5 times on average, over a horizon of \mbox{$T = 5,000$} (and in less than half of the number of epochs $H(T) = 13$). However, \revedit{\DfGreedy} is able to obtain its strong guarantees while only changing the threshold 3.4 times, on average. Moreover, while \revedit{Greedy-MLE} only increases the threshold 2.3 times, on average, the average relative increase of these changes is 46\%, representing a significant devaluation of earned points. \revedit{\DfGreedy}, on the other hand, never increases the threshold by construction, and benefits from an average relative decrease of thresholds of 20\%.
\begin{table}
    \centering
    \begin{tabular}{ccccc}
        \toprule
         & Number of changes & Relative change (\%) &Number of increases & Relative increase (\%) \\
        \midrule
        \revedit{Greedy-MLE} & 6.43 & 37 & 2.25 & 46 \\
        \revedit{\DfGreedy} & 3.39 & 20 & 0 & 0 \\
        \bottomrule
    \end{tabular}
    \caption{Adaptivity statistics of \revedit{Greedy-MLE} and \revedit{\DfGreedy} for $T = 5000$, with $H(T) = 13$. We report the number of threshold changes, absolute relative change in threshold, number of threshold increases, and relative increase in threshold, averaged across 100 replications.}
    \label{tab:fairness-learning-num}
\end{table}

\subsubsection{Robustness to misspecification.}
{Recall, our theoretical guarantees rely on the knowledge of the specific form of the link function $\mu_k(\cdot)$.} In this section we numerically investigate the impact of {a misspecified purchase probability model on our algorithms' performance}. To better isolate the impact of misspecification, we consider {the setting where $M = 1$}, omitting the subscript $k$ throughout.

We consider three true underlying models for the customer's purchase probability $\phi(\tau)$ {(also referred to as the {\it ground truth})}:
\begin{enumerate}
\item Linear: $\phi(\tau) = \min\left\{\bar{\phi} + (\alpha-\beta\tau)^+, 1\right\}$
\item Exponential: $\phi(\tau) = \min\left\{\bar{\phi} + \exp\left(\alpha - \beta\tau\right), 1\right\}$
\item Logit: $\phi(\tau) = \min\left\{\bar{\phi} + \frac{\exp\left(\alpha - \beta\tau\right)}{1+\exp\left(\alpha - \beta\tau\right)} , 1\right\}$.
\end{enumerate}
{We assume that our algorithms do not have access to this ground truth. Rather, they use a linear model in the maximum likelihood estimation step, i.e.,
\[\phi(\tau) = \min\left\{\bar{\phi} + \left(\alpha - \beta\tau\right)^+, 1\right\}.\]}
{In this case, the linear ground-truth model corresponds to a well-specified setting, whereas the exponential and logit ground-truth models correspond to misspecified settings. Including the linear ground-truth model provides us with a benchmark to disentangle the statistical error due to noisy purchase observations and the misspecification error due to estimating an incorrect model.}

Following \cite{besbes2015surprising}, we measure the performance of the algorithms by computing the fraction of the optimal long-run revenue achieved on each sample path{, referred to as the {\it performance ratio} $\gamma$. Formally:}
\begin{align*}
    \gamma = \frac{\sum_{t=1}^T R(N_t)}{TR(N^*)}.
\end{align*}
A higher value of $\gamma$ indicates better algorithm performance.

For each ground-truth model, we test three values of the base probability $\bar{\phi} \in \{0.05, 0.15, 0.25\}$, and run our algorithms with $T \in \{10^3,2\cdot 10^3, 5\cdot 10^3\}$ for each value of $\bar{\phi}$. For each such instance we conduct 500 replications, {independently generating the ground truth parameters} \mbox{$\alpha \sim \text{Unif}[1,1.5]$} and $\beta \sim \text{Unif}[1,1.5]$ {in each replication}. We report the average performance ratio of each algorithm on each tested instance in \cref{tab:misspecification}.

\begin{table}
\centering
\renewcommand{\arraystretch}{1.2}
\setlength{\tabcolsep}{6.5pt} 
\begin{tabular}{lc*{9}{c}} 
\toprule
\multicolumn{2}{l}{} & \multicolumn{3}{c}{Linear} & \multicolumn{3}{c}{Exponential} & \multicolumn{3}{c}{Logit} \\ \midrule
&   & \multicolumn{3}{c}{$T$} & \multicolumn{3}{c}{$T$} & \multicolumn{3}{c}{$T$} \\
&   & $10^3$ & $2\cdot 10^3$ & $5\cdot 10^3$ & $10^3$ & $2\cdot 10^3$ & $5\cdot 10^3$ & $10^3$ & $2\cdot 10^3$ & $5\cdot 10^3$ \\
\midrule
\multirow{3}{*}{\revedit{Greedy-MLE}} & $\bar{\phi}=0.05$ & 0.83 & 0.91 & 0.95 & 0.52 & 0.74 & 0.89 & 0.58 & 0.79 & 0.91 \\
& $\bar{\phi}=0.15$  & 0.94 & 0.97 & 0.98 & 0.86 & 0.93 & 0.95 & 0.89 & 0.94 & 0.98 \\
& $\bar{\phi}=0.25$ & 0.97 & 0.98 & 0.99 & 0.93 & 0.96 & 0.98 & 0.94 & 0.96 & 0.98 \\
\midrule
\multirow{3}{*}{\revedit{\DfGreedy}} & $\bar{\phi}=0.05$ & 0.81 & 0.89 & 0.94 & 0.50 & 0.69 & 0.80 & 0.57 & 0.76 & 0.88 \\
& $\bar{\phi}=0.15$  & 0.93 & 0.95 & 0.97 & 0.82 & 0.87 & 0.90 & 0.87 & 0.92 & 0.95 \\
& $\bar{\phi}=0.25$ & 0.93 & 0.94 & 0.95 & 0.90 & 0.93 & 0.94 & 0.92 & 0.95 & 0.97 \\
\bottomrule
\end{tabular}
\vspace{0.5em}
\caption{{Performance ratio of \revedit{Greedy-MLE} and \revedit{\DfGreedy} over a variety of randomly generated instances}.} \label{tab:misspecification}
\end{table}

{A number of observations emerge. First of all, these results echo the numerical findings of \Cref{ssec:num-regret}. In particular, we observe that the performance ratio is increasing in $T$, for both the well-specified and misspecified cases. This follows from the fact that learning becomes easier over longer decision-making horizons (as observed in \Cref{fig:regret}).  \revedit{We moreover observe that the performance ratio of \revedit{\DfGreedy} slightly underperforms that of \revedit{Greedy-MLE} across all instances. This again reflects our previous numerical finding that the decision-maker's revenue suffers from the devaluation-free constraint.}

Most importantly, our results illustrate the robustness of our algorithms' performance to a misspecified purchase probability model as $T$ grows large. In particular, for $T = 5\cdot 10^3$, $\gamma = 0.8$ in the worst case (achieved by \revedit{\DfGreedy} when $\bar{\phi} = 0.05$ and the ground-truth model is exponential). In all other instances, $\gamma \geq 0.88$ for this value of $T$. (These performance ratios are of the same magnitude as those computed for the problem of pricing; see Table 1 in \citet{besbes2015surprising}.)

Finally, we note the significant dependence of the price of misspecification on the base purchase probability $\bar{\phi}$. Namely, both algorithms' performance takes a significant hit when $\bar{\phi}=0.05$ as compared to when $\bar{\phi} = 0.25$. This is most notable for $T = 10^3$; for instance, the performance ratio of \revedit{Greedy-MLE} goes from 0.52 when $\bar{\phi} = 0.05$ to 0.93 when $\bar{\phi} = 0.25$ under the exponential ground-truth model. This phenomenon is due to the fact that our algorithms set the initial threshold to $N = 20$, with $\phi(20) \approx\bar{\phi}$ across all ground-truth models. When $\bar{\phi}$ is small, customers purchase extremely infrequently, and therefore take a much longer amount of time to progress through their redemption cycle. As a result, the number of samples required to receive informative signals for the maximum likelihood estimation step is much higher; for \mbox{$T = 10^3$}, the algorithms often start to update the threshold only towards the end of the horizon, resulting in a low performance ratio. While this effect is also present when the ground-truth model is linear, over the tested range of $(\alpha, \beta)$, the gap between the optimal revenue and $\bar{\phi}$ is much smaller, which is why we observe higher values of $\gamma$ for $T = 10^3$. These results highlight the practical importance of choosing a tight enough upper bound $N_{\max}$ with which to initialize \revedit{Greedy-MLE} and \revedit{\DfGreedy}. Alternatively, one may also choose a random initialization for \revedit{Greedy-MLE}.

\revedit{
\begin{remark}
Another important source of misspecification that the seller may face in practice is {\it type} misspecification, where individuals of the same type exhibit heterogeneity in purchase probabilities. In Appendix \ref{sec: noise within type} we numerically study the robustness of our algorithms to this source of misspecification.
\end{remark}
}
}

\section{Conclusion}\label{sec:conclusion}
\revedit{Motivated by the practical limitations associated with varying thresholds across customers and across time, this work studies the cost of consistency in points-based rewards programs.} Our results provide a number of important managerial insights. In particular, the uniform upper bound on the \revedit{value of personalization} in our model shows that, while there does exist value to personalization, a decision-maker cannot make arbitrary gains from exploiting heterogeneity between types. Additionally, the optimality of a devaluation-free learning algorithm that changes (but never increases) the redemption threshold $O(\log T)$ times highlights that controlled experimentation suffices for learning in these settings. 

From a technical perspective, our results provide insights into the analysis of greedy algorithms for contextual bandit problems. In particular, we show that it is sufficient for contexts to be {\it Markovian}, rather than i.i.d., for greedy strategies to be optimal. This finding is of independent interest, and likely has implications in related problems. Several open questions emerge from our work. Most immediately, one interesting question is whether our lower bound can be extended to exhibit the same dependence on the proportions of each type of customer, as seen in \Cref{thm:main-thm}. Another interesting technical question is whether the dependence on the mixing time in our upper bound is optimal. One would expect that such dependence is unavoidable since the mixing time relates to the variability of the Markov chain, which in turn determines whether there is enough diversity in the contexts to forgo any forced exploration. However, it would be interesting to quantify this effect through a lower bound involving the mixing time. We expect the proof of such a lower bound to require novel ideas, potentially involving anti-concentration of non-reversible Markov chains. {Finally, we assume that types are observable in our setting, a common assumption in the revenue management literature, in addition to being practically motivated. It would however be interesting to study the case where types are a priori unknown, \revedit{and the decision-maker must determine them online using high-dimensional customer features.\footnote{\revedit{\citet{keskin2025data} recently tackled this problem within the context of retail electricity pricing.}}}}

While our model captures the core components of points-based rewards programs, there are many possible modeling extensions that would make interesting directions for future study. First, while the assumption that the price is fixed is practically motivated, it would be interesting to see how our conclusions change if the decision-maker can jointly optimize over price and redemption thresholds. We would expect this to require significant innovation to handle the dependencies between decision variables. Additionally, a reasonable practical extension of our work would be to consider a multi-product setting. {We conjecture that our analysis and main insights extend relatively easily to this setting, as long as there exists a closed-form expression for the expected revenue as a function of $\tau$.} {In this case, however, strategic considerations may become important when there are multiple products, since customers may prefer to wait until they need a high-value product before redeeming. Modeling this phenomenon would require the development of a complex behavioral model which, while interesting, is beyond the scope of the present work.} 


%
%
%

\medskip 



\noindent\textbf{Acknowledgments.} {The authors would like to thank Yeganeh Alimohammadi, Michael Choi, Xiao Lei, and Vishal Gupta, as well as participants of the Johns Hopkins Carey Business School and University of Chicago Booth OM Seminars for helpful comments and conversations on early versions of our manuscript.}


\SingleSpacedXI
\bibliographystyle{informs2014}
\bibliography{literature}

\newpage

\begin{APPENDICES}
\OneAndAHalfSpacedXI

\section{\Cref{sec:preliminaries} Omitted Proofs}

\subsection{Proof of \Cref{prop:closed-form-lr-avg-per-cust}}\label{apx:closed-form-lr-avg-per-cust}
\begin{proof}{\it Proof.}
The number of points to redemption for each customer $j$ forms a Markov chain over states $\{0,1,\ldots,N\}$, which evolves as: 
\[\tau_{j,t+1} = (\tau_{jt} - X_{jt}) \mod (N+1), \quad \forall \ t \in \mathbb{N}^+.\]

Let $P$ be the transition matrix for this Markov chain. For all $\tau \in \{0,\ldots,N\}$:
\begin{align}
\begin{cases}
P_{\tau,\tau} = 1-\purchaseprob_{k(j)}(\tau) \\
P_{\tau,(\tau-1) \text{ mod } (N+1)} = \purchaseprob_{k(j)}(\tau) \\
P_{\tau,\tau'} = 0 \quad \forall \ \tau' \not\in \{\tau,\tau-1\}.
\end{cases}
\end{align}
The associated Markov chain is finite and irreducible; hence, a stationary distribution exists and is unique. We abuse notation and let $p = (p_0,\ldots,p_N)$ denote this steady-state distribution, which satisfies:
\begin{align*}
&\begin{cases}
p_{\tau} = p_{\tau}(1-\phi_{k(j)}(\tau)) + p_{(\tau+1)\text{ mod } (N+1)} \cdot \phi_{k(j)}\big((\tau+1)\text{ mod } (N+1)\big)\\
\sum_{\tau = 0}^\threshold p_{\tau} = 1, \quad p \geq 0.
\end{cases}
\end{align*}
Simplifying, we obtain:
\begin{align*}
\begin{cases}
p_\tau = p_0\cdot \frac{\phi_{k(j)}(0)}{\phi_{k(j)}(\tau)} \quad \forall \ \tau \in [N] \\
\sum_{\tau=0}^Np_\tau = 1, \quad p\geq 0
\end{cases}
\implies \begin{cases}
p_0 = \frac{1}{\sum_{\tau'=0}^N\frac{\phi_{k(j)}(0)}{\phi_{k(j)(\tau')}}}\\
p_{\tau} = \frac{1}{\sum_{\tau'=0}^N\frac{\phi_{k(j)}(\tau)}{\phi_{k(j)(\tau')}}} \quad \forall \ \tau \in [N].
\end{cases}
\end{align*}

 Note that the steady-state probability $p_\tau$ depends only on the customer through her type $k(j)$.  Re-defining this probability as $p_{k(j)}(\tau; N)$, we obtain the second part of the proposition.

The first part of the proposition then follows by noting that the long-run average purchase probability for customer $j$ is given by:
\begin{align*}
\sum_{\tau = 0}^N p_{k(j)}(\tau; N)\cdot \phi_{k(j)}(\tau)\mathds{1}\{\tau > 0\} = \sum_{\tau=1}^N \frac{1}{\sum_{\tau'=0}^N\frac{\phi_{k(j)}(\tau)}{\phi_{k(j)}(\tau')}}\cdot \phi_{k(j)}(\tau) = \frac{N}{\sum_{\tau'=0}^N\frac{1}{\phi_{k(j)}(\tau')}}.
\end{align*}
\hfill\Halmos
\end{proof}

\subsection{Proof of \Cref{prop:special-case-noloyalty}}\label{apx:no-loyalty-cvg}

\begin{proof}{\it Proof.}
By the Stolz-Ces\`aro theorem, 
\begin{align*}
\lim_{\threshold\to\infty}R_k(\threshold) = \lim_{\threshold\to\infty} \frac{(N+1)-N}{\sum_{\tau=0}^{N+1}\frac{1}{\phi_k(\tau)}-\sum_{\tau=0}^{N}\frac{1}{\phi_k(\tau)}} = \lim_{\threshold\to\infty} \frac{1}{\frac{1}{\phi_k(N+1)}} = \noloyaltyprob_k,
\end{align*}
where the final equality follows from the assumption that $\lim_{N\to\infty}\phi_k(N) = \noloyaltyprob_k$.
\hfill\Halmos
\end{proof}

\revedit{
\subsection{Special Case: Linear Points Pressure Effect}\label{apx:special-case-lin}
Suppose $K = 1$, and let $\phi(\tau) = \bar{\phi} + (\alpha-\beta\tau)^+$, where $\bar{\phi}, \alpha$, and $\beta$ are positive scalars such that $\bar{\phi} + \alpha \leq 1$. 

\Cref{prop:remove-trivial-linear} --- whose proof we defer to Appendix \ref{apx:remove-trivial-linear} --- first establishes the intuitive fact that, when customers are highly sensitive to the distance from redemption (i.e., $\beta$ is large), the seller does not benefit from offering a rewards program.
\begin{proposition}\label{prop:remove-trivial-linear}
Suppose $\alpha \leq \beta$. Then, it is optimal to not offer a rewards program.
\end{proposition}

Given this result, in the remainder of this section we assume $\alpha > \beta$. We next show that $\alpha/\beta$, which represents the largest distance to redemption for which the points program has an effect on customers' purchase behavior, plays a critical role in determining the optimal redemption threshold. In particular, we establish that it is suboptimal to set any threshold $N >  \lfloor\frac{\alpha}{\beta}\rfloor$, conditional on offering a rewards program. We defer the proof of this fact to Appendix \ref{apx:lin-ub-on-threshold}.
\begin{proposition}\label{prop:lin-ub-on-threshold}
Let $M = \lfloor\frac{\alpha}{\beta}\rfloor$. For any $N > M$, $R(N) \leq \max\{R(M), R(+\infty)\}$. 
\end{proposition}

For simplicity we assume $N_{\max} \leq M$ in what follows. We next prove that the optimal finite threshold scales as $\Theta\left(\sqrt{\frac{\bar{\phi}+\alpha}{\beta}}\right)$. 
\begin{proposition}\label{prop:linear-opt-threshold}
Conditional on offering a loyalty program, the optimal threshold $N^*$ is such that \mbox{$N^* \in \left[{\min}\left\{N_{\max},\frac12\left(\sqrt{9+8\cdot \frac{\bar{\phi}  + \alpha}{\beta }}-5\right)\right\},\min\left\{N_{\max},\frac12\left(\sqrt{1+8\cdot \frac{\bar{\phi} +\alpha}{\beta }}-1\right)\right\}\right]$}.
\end{proposition}

We defer the proof of the proposition to Appendix \ref{apx:linear-opt-threshold}. \Cref{prop:linear-opt-threshold} implies that, for fixed $\beta$, the optimal threshold is {\it increasing} in $\bar{\phi} + \alpha$. In particular, the seller should set a higher threshold for customers who either (i) have a high baseline purchase probability $\bar{\phi}$, or (ii) are strongly impacted by the presence of a rewards program (i.e., $\alpha$ is high).  This result is intuitive: larger values of $\bar{\phi}+\alpha$ imply that the customer's expected time to redemption decreases. Hence, a larger threshold is required to avoid giving out the item for free too frequently; the seller can do so without meaningfully reducing the customer's average purchase probability throughout the cycle. On the other hand, for fixed $\bar{\phi} + \alpha$, the optimal threshold is {\it decreasing} in $\beta$, implying that the seller should set lower thresholds for customers who are more sensitive to the distance to redemption. High thresholds are inefficient for high values of $\beta$, as most of the redemption cycle is spent in states for which the purchase probability is close to $\bar{\phi}$. Hence, the seller should set a low threshold so that customers spend a larger fraction of time in the region where their purchase likelihood is meaningfully impacted.

We conclude our analysis by leveraging these results to obtain bounds on the optimal non-personalized threshold, when $K \geq 2$. 
\begin{proposition}\label{prop:lin-uniform-threshold-bound}
Suppose $K \geq 2$, and $N_{\max} \leq \max_k \lfloor\frac{\alpha_k}{\beta_k}\rfloor$. Without loss of generality, suppose moreover  $\frac{\bar{\phi}_1  + \alpha_1}{\beta_1} \leq \ldots\leq \frac{\bar{\phi}_K  + \alpha_K}{\beta_K}$. Then:
\[N^* \in \left[\min\left\{N_{\max},\frac12\left(\sqrt{9+8\cdot \frac{\bar{\phi}_1  + \alpha_1}{\beta_1}}-5\right)\right\},\min\left\{N_{\max},\frac12\left(\sqrt{1+8\cdot \frac{\bar{\phi}_K +\alpha_K}{\beta_K}}-1\right)\right\}\right].\]
\end{proposition}

The proof of \Cref{prop:lin-uniform-threshold-bound} can be found in Appendix \ref{apx:lin-uniform-threshold-bound}.

\medskip

\subsubsection{Proof of \Cref{prop:remove-trivial-linear}}\label{apx:remove-trivial-linear}

\begin{proof}{\it Proof.}
Since $\alpha \leq \beta$, $(\alpha-\beta\tau)^+=0$ for all $\tau \geq 1$. Hence, for any $N \geq 1$, we have:
\begin{align*}
R(N) &= \frac{N}{\frac{1}{\bar{\phi} + \alpha} + \sum_{\tau=1}^N \frac{1}{\bar{\phi}}} = \frac{N}{\frac{1}{\bar{\phi} + \alpha} + \frac{N}{\bar{\phi}}}.
\end{align*}
Observing that this is strictly increasing in $N$, we obtain the result.\hfill\Halmos
\end{proof}

\medskip 

\subsubsection{Proof of \Cref{prop:lin-ub-on-threshold}}\label{apx:lin-ub-on-threshold}

\begin{proof}{\it Proof.}
Fix $N \geq M$. Since $\phi(\tau) = \bar{\phi}$ for all $\tau > M$, we have:
\begin{align*}
R(N) = \frac{N}{\sum_{\tau=0}^M\frac{1}{\phi(\tau)} + \frac{N-M}{\bar{\phi}}}.
\end{align*}
For $M$ such that $\sum_{\tau=0}^M\frac{1}{\phi(\tau)}- \frac{M}{\bar{\phi}} \geq 0$, $R(N)$ is non-decreasing in $N$. Hence, $R(N) \leq R(+\infty)$. On the other hand, for $M$ such that $\sum_{\tau=0}^M\frac{1}{\phi(\tau)}- \frac{M}{\bar{\phi}} < 0$, $R(N)$ is decreasing in $N$. In this case, $R(N) < R(M)$.\hfill\Halmos
\end{proof}

\medskip

\subsubsection{Proof of \Cref{prop:linear-opt-threshold}}\label{apx:linear-opt-threshold}
\begin{proof}{\it Proof.}
For ease of notation, let $\gamma = \bar{\phi}  + \alpha$. 
By unimodality of $R(N)$, it suffices to bound $\inf\left\{N \in [N_{\max}] : R(N+1) \leq R(N) \right\}$. 
We have:
\begin{align*}
R(N+1) \leq R(N) &\iff (N+1)\sum_{\tau=0}^N\frac{1}{\phi(\tau)} \leq N \sum_{\tau=0}^{N+1}\frac{1}{\phi(\tau)} \iff \sum_{\tau=0}^N\frac{1}{\phi(\tau)}\leq \frac{N}{\phi(N+1)}.
\end{align*}
By convexity of the function $f(x) = 1/x$, for $\tau \in \{0,\ldots,N\}$, we have:
\begin{align*}
&\frac{1}{\phi(\tau)} \geq \frac{1}{\phi(N+1)}-\frac{1}{\phi(N+1)^2}\left(\phi(\tau)-\phi(N+1)\right)\\ 
\implies &\left(\sum_{\tau=0}^N\frac{1}{\phi(\tau)}\right)-\frac{N}{\phi(N+1)} \geq \frac{1}{\phi(N+1)}-\frac{1}{\phi(N+1)^2}\left(\left(\sum_{\tau=0}^N \gamma - \beta\tau\right) - (N+1)\phi(N+1)\right)\\
&= \frac{N+2}{\phi(N+1)}-\frac{1}{\phi(N+1)^2}\left(\gamma(N+1)-\beta\cdot\frac{1}{2}N(N+1)\right).
\end{align*}
Let $\widetilde{N} = \inf\left\{N \in \mathbb{N}: R(N+1) \leq R(N)\right\}$. Then, $\widetilde{N}$ is lower bounded by the largest solution to:
\begin{align}\label{eq:lb-root}
(N+2)\phi(N+1)-\gamma(N+1)+\frac12\cdot\beta N(N+1) \leq 0.
\end{align}

Suppose first that ${N}+1 \leq M$. 
In this case, the left-hand side of \eqref{eq:lb-root} is equal to:
\begin{align*}
(N+2)\left(\gamma-\beta(N+1)\right)-\gamma(N+1)+\frac12\beta N(N+1) &= \gamma - \beta(N+1)(N+2)+\frac12\beta N(N+1)\\
&=  \gamma - \beta (N+1)\left(\frac12N+2\right).
\end{align*}
Hence, we have that $\gamma - \beta (N+1)\left(\frac12N+2\right) \leq 0$ for all $N \geq \frac12\left(\sqrt{9+8\frac{\gamma}{\beta }}-5\right)$

Note that, for $\widetilde{N} = \frac12\left(\sqrt{9+8\frac{\gamma}{\beta }}-5\right)$ to be consistent, we must have $\widetilde{N} + 1 \leq M$. {Suppose this inequality does not hold. Then, $\widetilde{N} > M-1 \implies \widetilde{N} \geq M \geq N_{\max}$ by assumption. Hence, we necessarily have \mbox{$N^* = N_{\max}$} in this case.} 
We conclude that $N^* \geq \min\left\{N_{\max},\frac12\left(\sqrt{9+8\frac{\gamma}{\beta }}-5\right)\right\}$.

For the upper bound, similarly observe that, by convexity, for $\tau \in \{0,\ldots,N\}$, $\frac{1}{\phi(\tau)}$ is upper bounded by the secant line that goes through $\frac{1}{\phi(0)}$ at $\tau = 0$ and $\frac{1}{\phi(N)}$ at $\tau = N$, i.e., $\frac{1}{\phi(0)} + \frac{\frac{1}{\phi(N)}-\frac{1}{\phi(0)}}{N}\cdot \tau$. Using this bound, and the fact that $\phi(N+1)\leq \phi(N)$, we have:
\begin{align*}
\sum_{\tau=0}^N\frac{1}{\phi(\tau)}-\frac{N}{\phi(N+1)} &\leq \sum_{\tau=0}^N \left(\frac{1}{\phi(0)} + \frac{\frac{1}{\phi(N)}-\frac{1}{\phi(0)}}{N}\cdot \tau\right) - \frac{N}{\phi(N)}\\
&=\frac{N+1}{\phi(0)} + \frac12\cdot \frac{\frac{1}{\phi(N)}-\frac{1}{\phi(0)}}{N}\cdot N(N+1) - \frac{N}{\phi(N)}\\ 
&=\frac{N+1}{\phi(0)} + \frac12 \left({\frac{1}{\phi(N)}-\frac{1}{\phi(0)}}\right) (N+1) - \frac{N}{\phi(N)}\\
&=\frac12 \cdot \frac{N+1}{\phi(0)}+ \frac12\cdot \frac{1-N}{\phi(N)}\\
&=\frac12\cdot \frac{N+1}{\gamma} + \frac12\cdot \frac{1-N}{\gamma-\beta N}.
\end{align*}
Hence, $\widetilde{N}$ is upper bounded by the largest solution to:
\begin{align*}
(N+1)(\gamma-\beta N)+(1-N)\gamma \leq 0,
\end{align*}
i.e., $N = \frac12\left(\sqrt{1+\frac{8\gamma}{\beta }}-1\right)$. In this case, we have $N^* \leq \min\left\{N_{\max},\frac12\left(\sqrt{1+\frac{8\gamma}{\beta }}-1\right)\right\}$. \hfill\Halmos
\end{proof}

\medskip 

\subsubsection{Proof of \Cref{prop:lin-uniform-threshold-bound}}\label{apx:lin-uniform-threshold-bound}

\begin{proof}{\it Proof.}
This again follows from unimodality of the type-specific revenue functions $R_k(N)$. In particular, by \Cref{prop:linear-opt-threshold}, for all $k \in [K]$, $N_k^* \geq \min\left\{N_{\max},\frac12\left(\sqrt{9+8\cdot \frac{\bar{\phi}_1  + \alpha_1}{\beta_1}}-5\right)\right\} =: \underline{N}_1$. Hence, $R_k(\underline{N}_1) \geq R_k(\underline{N}_1-1)$ for all $k \in [K]$, which then implies that $R(\underline{N}_1) \geq R(\underline{N}_1-1)$. Similarly, $N_k^* \leq\min\left\{N_{\max},\frac12\left(\sqrt{1+8\cdot \frac{\bar{\phi}_K +\alpha_K}{\beta_K}}-1\right)\right\} =: \underline{N}_K$ for all $k \in [K]$. This implies that $R_k(\underline{N}_K + 1) \leq R_k(\underline{N}_K)$ for all $k$, and therefore \mbox{$R(\underline{N}_K + 1) \leq R(\underline{N}_K)$}.

Putting these two facts together, we obtain our bounds on $N^*$.\hfill\Halmos
\end{proof}
}

\section{\Cref{sec:complete-info} Omitted Proofs}

\subsection{Proof of \Cref{thm:price-of-fairness}}\label{apx:price-of-fairness}
\begin{proof}{\it Proof.}
To prove the result, we bound the inverse of \revedit{VoP}, which we denote by $\gamma = \frac{\mathcal{R}^{\text{non-pers}}}{\mathcal{R}^{\text{pers}}}$. Let \mbox{$N^* \in \argmax_{N\in[\maxthreshold]\cup\{+\infty\}}\sum_{k\in[K]}\rho_kR_k(N)$} be an optimal non-personalized threshold, breaking ties arbitrarily. 
Finally, for ease of notation we let \mbox{$\mathcal{R}_k^{\text{pers}}=\rho_k\cdot\frac{N_k^*}{\sum_{\tau=0}^{N_k^*}\frac{1}{\phi_k(\tau)}}$} be the long-run average revenue associated with type $k$ under their optimal personalized threshold, weighted by the fraction of type $k$ individuals $\rho_k$. Note that \mbox{$\mathcal{R}^{\text{pers}} = \sum_{k\in[K]}\mathcal{R}_k^{\text{pers}}$}.

Fix $k \in [K]$, and let $\mathbf{N} = (N_k^*, N_k^*,\ldots,N_k^*)$ be the vector that sets the same threshold $N_k^*$ for each type $j\in[K]$. By optimality of $N^*$, we have:
\begin{align}
\mathcal{R}^{\text{non-pers}} \geq R(\mathbf{N}) &= \sum_{j\in[K]}\rho_j\cdot \frac{N_k^*}{\sum_{\tau=0}^{N_k^*}\frac{1}{\phi_j(\tau)}} \notag \\
&= \left(\sum_{j < k}\rho_j\cdot \frac{N_k^*}{\sum_{\tau=0}^{N_k^*}\frac{1}{\phi_j(\tau)}}\right) + \mathcal{R}_k^{\text{pers}} + \left(\sum_{j > k}\rho_j\cdot \frac{N_k^*}{\sum_{\tau=0}^{N_k^*}\frac{1}{\phi_j(\tau)}}\right)  \label{eq:pers-bound-00}\\
&\geq \mathcal{R}_k^{\text{pers}} + \left(\sum_{j > k}\rho_j\cdot \frac{N_k^*}{\sum_{\tau=0}^{N_k^*}\frac{1}{\phi_j(\tau)}}\right) \notag\\ 
&=  \mathcal{R}_k^{\text{pers}} + \left(\sum_{j > k}\mathcal{R}_j^{\text{pers}}\cdot\frac{N_k^*}{N_j^*}\cdot\frac{\sum_{\tau=0}^{N_j^*}\frac{1}{\phi_j(\tau)}}{\sum_{\tau=0}^{N_k^*}\frac{1}{\phi_j(\tau)}}\right), \label{eq:pers-bound-1}
\end{align}
where the first equality applies \Cref{prop:closed-form-lr-avg-per-cust} to $R(\mathbf{N})$, the first inequality follows from trivially lower bounding the first term in \eqref{eq:pers-bound-00} by 0, and the final equality multiplies and divides each term in the summand by $\mathcal{R}_j^{\text{pers}} = \rho_j\cdot\frac{N_j^*}{\sum_{\tau=0}^{N_j^*}\frac{1}{\phi_j(\tau)}}$.

We first focus on bounding the ratio $\frac{\sum_{\tau=0}^{N_j^*}\frac{1}{\phi_j(\tau)}}{\sum_{\tau=0}^{N_k^*}\frac{1}{\phi_j(\tau)}}$, for all $j > k$. Since $\phi_j(\tau)$ is non-increasing, we have that $\phi_j(\tau)\leq \phi_j(N_k^*+1)$ for all $\tau \geq N_k^*+1$. Therefore:
\begin{align*}
\sum_{\tau=0}^{N_j^*}\frac{1}{\phi_j(\tau)} \geq \sum_{\tau=0}^{N_k^*}\frac{1}{\phi_j(\tau)} + (N_j^*-N_k^*)\cdot\frac{1}{\phi_j(N_k^*+1)}.
\end{align*}
Similarly, since $\phi_j(\tau) \geq \phi_j(N_k^*)$ for all $\tau \leq N_k^*$:
\begin{align*}
\sum_{\tau=0}^{N_k^*}\frac{1}{\phi_j(\tau)} \leq (N_k^*+1)\cdot\frac{1}{\phi_j(N_k^*)}.
\end{align*}
Putting these two bounds together yields:
\begin{align*}
\frac{\sum_{\tau=0}^{N_j^*}\frac{1}{\phi_j(\tau)}}{\sum_{\tau=0}^{N_k^*}\frac{1}{\phi_j(\tau)}} &\geq  \frac{\sum_{\tau=0}^{N_k^*}\frac{1}{\phi_j(\tau)} + (N_j^*-N_k^*)\cdot\frac{1}{\phi_j(N_k^*+1)}}{\sum_{\tau=0}^{N_k^*}\frac{1}{\phi_j(\tau)}} \\
&\geq 1 + \frac{(N_j^*-N_k^*)\cdot\frac{1}{\phi_j(N_k^*+1)}}{(N_k^*+1)\cdot\frac{1}{\phi_j(N_k^*)}}\\
&= 1 + \frac{(N_j^*-N_k^*)\cdot\phi_j(N_k^*)}{(N_k^*+1)\cdot\phi_j(N_k^*+1)}\\
&\geq 1 + \frac{N_j^*-N_k^*}{N_k^*+1}\\
&= \frac{N_j^*+1}{N_k^*+1},
\end{align*}
where the second inequality follows from the fact that $\phi_j(\tau) \geq \phi_j(N_k^*)$ for all $\tau \leq N_k^*$, and the final inequality similarly uses $\phi_j(N_k^*) \geq \phi_j(N_k^*+1)$.

Plugging this into \eqref{eq:pers-bound-1}, we obtain that, for all $k \in [K]$
\begin{align*}
&\mathcal{R}^{\text{non-pers}} \geq \mathcal{R}_k^\text{pers} + \sum_{j > k} \mathcal{R}_j^{\text{pers}}\cdot\frac{N_k^*}{N_k^*+1}\cdot\frac{N_j^*+1}{N_j^*}.
\end{align*}
Dividing both sides by $\mathcal{R}^{\text{pers}} = \sum_{j\in[K]}\mathcal{R}_j^{\text{pers}}$ and taking the maximum over all $k \in [K]$, we obtain:
\begin{align}\label{eq:non-pers-better-than-max}
\gamma \geq \max_{k\in[K]}\left\{\frac{\mathcal{R}_k^{\text{pers}}+\frac{N_k^*}{N_k^*+1}\sum_{j > k}\frac{N_j^*+1}{N_j^*}\mathcal{R}_j^{\text{pers}}}{\sum_{j\in[K]}\mathcal{R}_j^{\text{pers}}}\right\}.
\end{align}

\Cref{lem:bounding-max} lower bounds the right-hand side of \eqref{eq:non-pers-better-than-max}, exclusively as a function of the optimal personalized thresholds $N_k^*$. We defer its proof to Appendix \ref{apx:non-pers-better-than-max}.
\begin{lemma}\label{lem:bounding-max}
\[\max_{k\in[K]}\left\{\frac{\mathcal{R}_k^{\text{pers}}+\frac{N_k^*}{N_k^*+1}\sum_{j > k}\frac{N_j^*+1}{N_j^*}\mathcal{R}_j^{\text{pers}}}{\sum_{j\in[K]}\mathcal{R}_j^{\text{pers}}}\right\} \geq \frac{1}{K-\sum_{k=1}^{K-1}\frac{N_k^*}{N_k^*+1}\cdot\frac{N_{k+1}^*+1}{N_{k+1}^*}}.\]
\end{lemma}

\revedit{The following upper bound on the denominator completes the proof of the theorem. We defer its proof to Appendix \ref{apx:lb-telescoping-prod}.
\begin{lemma}\label{lem:lb-telescoping-prod}
For all $(N_1,\ldots,N_K)$ such that $1 \leq N_1 \leq N_2 \leq \ldots \leq N_K$,
\[K-\sum_{k=1}^{K-1}\frac{N_k}{N_k+1}\cdot\frac{N_{k+1}+1}{N_{k+1}} \leq \min\left\{1 + \frac{1}{N_1} - \frac{1}{N_K},K-(K-1)2^{-1/(K-1)},\frac{\pi^2}{6}\right\}.\]
\end{lemma}
}

\revedit{
We complete the proof of the theorem by establishing tightness when (i) $K = 2$ and (ii) $K\to+\infty$. We defer the proof of the below lemma to Appendix \ref{apx:pers-tight}.
\begin{lemma}\label{lem:pers-tight}
For $K = 2$, there exists an instance such that $\text{\revedit{VoP}} = K-(K-1)2^{-1/(K-1)} = 3/2$. Moreover, there exists a sequence of instances parameterized by $K$ such that the \text{\revedit{VoP}} of this sequence of instances converges to $\frac{\pi^2}{6}$ as $K\to+\infty$.
\end{lemma}
}
\hfill\Halmos
\end{proof}

\subsubsection{Proof of \Cref{lem:bounding-max}}\label{apx:non-pers-better-than-max}

\begin{proof}{\it Proof.}
We seek to lower bound the following quantity:
\[\max_{k\in[K]}\left\{\frac{\mathcal{R}_k^{\text{pers}}+\frac{N_k^*}{N_k^*+1}\sum_{j > k}\frac{N_j^*+1}{N_j^*}\mathcal{R}_j^{\text{pers}}}{\sum_{j\in[K]}\mathcal{R}_j^{\text{pers}}}\right\}.\]
Observe that the minimum of the quantity of interest is attained when all $K$ terms are equal. Moreover, the denominator has no dependence on $k$, so it suffices to find the minimum-value solution such that all $K$ numerators are equal. We prove by induction that equality holds uniquely for all $(\mathcal{R}_1^{\text{pers}},\ldots,\mathcal{R}_K^{\text{pers}})$ satisfying:
\begin{align}\label{eq:equal}
\mathcal{R}_k^{\text{pers}} = \left(1-\frac{N_k^*}{N_k^*+1}\cdot\frac{N_{k+1}^*+1}{N_{k+1}^*}\right)\mathcal{R}_K^{\text{pers}} \quad \ \forall \ k \in [K-1].
\end{align}

\noindent\textbf{Base case: $k=K-1$.} In this case, we seek to solve:
\begin{align*}
&\rpers{K-1} + \frac{\npers{K-1}}{\npers{K-1}+1}\cdot\frac{\npers{K}+1}{\npers{K}}\cdot\rpers{K} = \rpers{K} \\
\iff &\rpers{K-1} = \left(1-\frac{\npers{K-1}}{\npers{K-1}+1}\cdot\frac{\npers{K}+1}{\npers{K}}\right)\rpers{K},
\end{align*}
which completes the proof of the base case.

\medskip 

\noindent\textbf{Inductive step.} Fix $k \in \{2,\ldots,K-1\}$, and suppose \eqref{eq:equal} holds for all $k' \geq k$. We prove that it also holds for $k-1$. Again, we seek to solve:
\begin{align*}
&\rpers{k-1}+\frac{\npers{k-1}}{\npers{k-1}+1}\sum_{j > k-1}\frac{N_j^*+1}{N_j^*}\mathcal{R}_j^{\text{pers}} = \rpers{k}+\frac{\npers{k}}{\npers{k}+1}\sum_{j>k}\frac{N_j^*+1}{N_j^*}\mathcal{R}_j^{\text{pers}}\\
\iff &\rpers{k-1} = \rpers{K}\Bigg[1-\frac{\npers{k}}{\npers{k}+1}\frac{\npers{k+1}+1}{\npers{k+1}} + \frac{\npers{k}}{\npers{k}+1}\sum_{j>k}\frac{\npers{j}+1}{\npers{j}}\left(1-\frac{\npers{j}}{\npers{j}+1}\cdot\frac{\npers{j+1}+1}{\npers{j+1}}\right)\\&\qquad\qquad\qquad-\frac{\npers{k-1}}{\npers{k-1}+1}\sum_{j>k-1}\frac{\npers{j}+1}{\npers{j}}\left(1-\frac{\npers{j}}{\npers{j}+1}\cdot\frac{\npers{j+1}+1}{\npers{j+1}}\right)\Bigg],
\end{align*}
where the second line follows from the inductive hypothesis. Further simplifying, we have:
\begin{align*}
\rpers{k-1} &= \rpers{K}\Bigg[1-\frac{\npers{k-1}}{\npers{k-1}+1}\sum_{j > k-1}\frac{\npers{j}+1}{\npers{j}} + \frac{\npers{k-1}}{\npers{k-1}+1}\sum_{j>k}\frac{\npers{j}+1}{\npers{j}} \Bigg]\\
&=\rpers{K}\Bigg[1-\frac{\npers{k-1}}{\npers{k-1}+1}\cdot\frac{\npers{k}+1}{\npers{k}}\Bigg],
\end{align*}
which completes the proof of the fact that the minimum of $\max_{k\in[K]}\left\{\frac{\mathcal{R}_k^{\text{pers}}+\frac{N_k^*}{N_k^*+1}\sum_{j > k}\frac{N_j^*+1}{N_j^*}\mathcal{R}_j^{\text{pers}}}{\sum_{j\in[K]}\mathcal{R}_j^{\text{pers}}}\right\}$ is achieved at $(\rpers{1},\ldots,\rpers{K})$ satisfying \eqref{eq:equal}. Using this fact, we have:
\begin{align*}
\max_{k\in[K]}\left\{\frac{\mathcal{R}_k^{\text{pers}}+\frac{N_k^*}{N_k^*+1}\sum_{j > k}\frac{N_j^*+1}{N_j^*}\mathcal{R}_j^{\text{pers}}}{\sum_{j\in[K]}\mathcal{R}_j^{\text{pers}}}\right\} &\geq \frac{\rpers{K}}{\rpers{K}\left(1+\sum_{k =1}^{K-1}\left(1-\frac{\npers{k}}{\npers{k}+1}\cdot\frac{\npers{k+1}+1}{\npers{k+1}}\right)\right)}\\
&\geq \frac{1}{K-\sum_{k=1}^{K-1}\frac{\npers{k}}{\npers{k}+1}\cdot\frac{\npers{k+1}+1}{\npers{k+1}}}.
\end{align*}
\hfill\Halmos
\end{proof}

\revedit{
\subsubsection{Proof of \Cref{lem:lb-telescoping-prod}}\label{apx:lb-telescoping-prod}

\begin{proof}{\it Proof.} We prove the three bounds separately.

\smallskip 

\noindent\textbf{Bound 1: $K - \sum_{k=1}^{K-1}\frac{N_k}{N_k+1}\cdot\frac{N_{k+1}+1}{N_{k+1}}\leq 1 + \frac{1}{N_1}-\frac{1}{N_K}$.} We have:
\begin{align}
    K - \sum_{k=1}^{K-1}\frac{N_k}{N_k+1}\cdot\frac{N_{k+1}+1}{N_{k+1}}  &=  K - \sum_{k=1}^{K-1} \left(1- \frac{1}{N_k+1}\right)\left(1+\frac{1}{N_{k+1}}\right) \notag \\
    &= 1- \sum_{k=1}^{K-1} \left(-\frac{1}{N_k+1} + \frac{1}{N_{k+1}} - \frac{1}{N_{k+1}(N_k+1)}\right) \label{eq:use-in-next-bound}
    \\ &\leq 1- \sum_{k=1}^{K-1} \left(-\frac{1}{N_k+1} + \frac{1}{N_{k+1}} - \frac{1}{N_{k}(N_k+1)}\right) \notag
    \\ & = 1- \sum_{k=1}^{K-1} \left(-\frac{1}{N_k+1} + \frac{1}{N_{k+1}} - \left(\frac{1}{N_{k}} - \frac{1}{N_k+1}\right) \right) \notag
    \\&=  1- \sum_{k=1}^{K-1} \left( \frac{1}{N_{k+1}} - \frac{1}{N_{k}} \right) \notag
    \\& = 1 +\frac{1}{N_{1}} - \frac{1}{N_K}, \notag
\end{align}
where the inequality uses the fact that $N_{k+1} \geq N_k$.

\smallskip 

\noindent\textbf{Bound 2: $K - \sum_{k=1}^{K-1}\frac{N_k}{N_k+1}\cdot\frac{N_{k+1}+1}{N_{k+1}} \leq K-(K-1)2^{-1/(K-1)}$.} 

It suffices to show that $\sum_{k=1}^{K-1}\frac{N_k}{N_k+1}\cdot\frac{N_{k+1}+1}{N_{k+1}} \geq (K-1)2^{-1/(K-1)}$. By the AM-GM inequality:
\begin{align*}
\sum_{k=1}^{K-1}\frac{N_k}{N_k+1}\cdot\frac{N_{k+1}+1}{N_{k+1}} &\geq (K-1)\left(\prod_{k=1}^{K-1}\frac{\frac{N_k}{N_k+1}}{\frac{N_{k+1}}{N_{k+1}+1}}\right)^{1/(K-1)}\\
&=(K-1)\left(\frac{\frac{N_1}{N_1+1}}{\frac{N_K}{N_K+1}}\right)^{1/(K-1)}\\
&\geq (K-1)\left(\frac12\right)^{1/(K-1)},
\end{align*}
where the final inequality follows from the fact that $N_1 \geq 1$.

\smallskip 

\noindent\textbf{Bound 3: $K - \sum_{k=1}^{K-1}\frac{N_k}{N_k+1}\cdot\frac{N_{k+1}+1}{N_{k+1}}\leq \frac{\pi^2}{6}$.} By \eqref{eq:use-in-next-bound}, we have:
\begin{align}\label{eq:plug-back-in-for-pi-6}
K - \sum_{k=1}^{K-1}\frac{N_k}{N_k+1}\cdot\frac{N_{k+1}+1}{N_{k+1}}  &= 1+ \sum_{k=1}^{K-1} \frac{N_{k+1}-N_k}{(N_k+1)N_{k+1}},
\end{align}
where we use the convention that $\frac{N_{k+1}-N_k}{(N_k+1)N_{k+1}} = 0$ for all $k$ such that $N_k = N_{k+1} = +\infty$.

We upper bound the above via the following auxiliary lemma, whose proof we defer to Appendix \ref{apx:pof-tighter-lem}.
\begin{lemma}\label{lem:pof-tighter-lem}
For all $k \in [K-1]$ such that $N_{k+1} < +\infty$ and $N_{k+1} \geq N_k+1$,
\[\frac{N_{k+1}-N_k}{(N_k+1)N_{k+1}} \leq \sum_{i=N_k+1}^{N_{k+1}}\frac{1}{i^2}.\]
\end{lemma}
Let $K^*  = \min\left\{k \in [K-1]: N_{k+1} = +\infty\right\}$. (If this latter set is empty, we adopt the convention that \mbox{$K^* = K$}.)
 Applying \Cref{lem:pof-tighter-lem} above, we have:
\begin{align}\label{eq:tigher-bound}
\sum_{k=1}^{K-1}\frac{N_{k+1}-N_k}{(N_k+1)N_{k+1}} &= \left(\sum_{k=1}^{K^*-1}\frac{N_{k+1}-N_k}{(N_k+1)N_{k+1}}\right) + \frac{1}{N_{K^*}+1} \notag \\
&\leq \left(\sum_{k=1}^{K^*-1}\mathds{1}\{N_{k+1} \geq N_k+1\}\sum_{i=N_k+1}^{N_{k+1}} \frac{1}{i^2}\right) + \frac{1}{N_{K^*}+1} \notag \\
&= \left(\sum_{i=N_1+1}^{N_{K^*}} \frac{1}{i^2}\right) + \frac{1}{N_{K^*}+1} \\
&\leq \sum_{i=2}^{+\infty}\frac{1}{i^2} \notag \\
&= \frac{\pi^2}{6}-1, \notag 
\end{align}
where the last inequality follows from the fact that $\left(\sum_{i=N_1+1}^{N_{K^*}} \frac{1}{i^2}\right) + \frac{1}{N_{K^*}+1}$ is increasing in $N_{K^*}$ and decreasing in $N_1$. Plugging this back into \eqref{eq:plug-back-in-for-pi-6}, we obtain the bound.
\hfill\Halmos
\end{proof}
}

\revedit{
\subsubsection{Proof of \Cref{lem:pers-tight}}\label{apx:pers-tight}
\begin{proof}{\it Proof.} We define a sequence of instances parameterized by $K$. Let $\phi_K(\tau) = 1$ for all $\tau \in \mathbb{N}$ and $\bar{\phi}_K = 1$. For $k \in [K-1]$, let $\bar{\phi}_k = 0$, with 
\[\phi_k(\tau) = \begin{cases}
1 \quad& \forall \ \tau \in \{0,\ldots,k\} \\
0 \quad& \forall \ \tau > k.
\end{cases}
\]
Suppose moreover that
\begin{align*}
\rho_k = \begin{cases}
\frac12\left(\frac1k-\frac{1}{k+1}\right) \quad &\forall \ k \leq K-2\\ 
\frac12 \cdot \frac{1}{K-1} \quad &k = K-1 \\
\frac12 \quad &k = K.
\end{cases}
\end{align*}


Suppose $N_{\max} > K$. It is easy to see that the optimal personalized scheme sets $N_k^* = k$ for all $k \in [K-1]$, and $N_K^* = +\infty$. Hence,
\begin{align*}
\mathcal{R}^{\text{pers}} &= \rho_K + \sum_{k \leq K-1}\rho_k \cdot \frac{k}{k+1} \\
&= \frac12 + \left(\frac12\sum_{k\leq K-2}\frac{1}{k(k+1)}\cdot \frac{k}{k+1}\right) + \frac12 \cdot \frac{1}{K-1}\cdot\frac{K-1}{K} \\
&= \frac12\left(1+\sum_{k\leq K-2}\frac{1}{(k+1)^2} + \frac1K\right) = \frac12\left(1+\sum_{k=2}^{K-1}\frac{1}{k^2} + \frac1K\right) 
\end{align*}

Additionally, the optimal uniform scheme necessarily sets $N^* \in \{1,\ldots,K-1\}\cup\{+\infty\}$. Fix $j \in [K-1]$. Then, the uniform scheme which sets $N = j$ for all customers obtains revenue:
\begin{align*}
R(N) &= \sum_{k \geq j}\rho_k\cdot \frac{j}{j+1}\\
&= \frac{j}{j+1}\cdot \frac12 \left(\sum_{k = j}^{K-2}\left(\frac1k-\frac{1}{k+1}\right) + \frac{1}{K-1} + 1\right)\\
&= \frac12 \cdot \frac{j}{j+1}\cdot \left(1 + \frac1j\right)\\
&= \frac12.
\end{align*}
Similarly, the uniform scheme which sets $N = +\infty$ obtains $R(N) = 1/2$, since only type $K$ purchases under this scheme. 

Therefore, $\mathcal{R}^{\text{non-pers}} = \frac12$, and $\revedit{VoP} = 1 + \sum_{k=2}^{K-1} \frac{1}{k^2} + \frac1K$. This yields $\revedit{VoP} = 3/2$ when $K = 2$, and $\lim_{K\to+\infty }\revedit{VoP} = \frac{\pi^2}{6}$, proving tightness of the bounds for both extremes.
\hfill\Halmos
\end{proof}
}

\revedit{
\subsubsection{Proof of \Cref{lem:pof-tighter-lem}}\label{apx:pof-tighter-lem}
\begin{proof}{\it Proof.}
The fact follows from the following inequality. For all $a, b \in \mathbb{N}$ such that $a \leq b < +\infty$,
\begin{align}\label{eq:tight-inequality}
\frac{b-a+1}{ab} \leq \sum_{i=a}^b\frac{1}{i^2}.
\end{align}
Letting $a = N_k+1$ and $b = N_{k+1}$, the lemma is shown.

Fix $a \in \mathbb{N}$. We prove \eqref{eq:tight-inequality} by induction on $b$.

\textbf{Base case: $b = a$.} The fact trivially holds with equality.

\textbf{Inductive step:} Suppose there exists $b^*$ such that \eqref{eq:tight-inequality} holds for all $b \leq b^*-1$. We prove that it then holds for $b^*$. By the inductive hypothesis:
\begin{align*}
\sum_{i=a}^{b^*} \frac{1}{i^2} &\geq \frac{b^*-1-a+1}{a(b^*-1)} + \frac{1}{(b^*)^2} \\
&= \frac{b^*-a}{a(b^*-1)} +  \frac{1}{(b^*)^2}.
\end{align*}
We have:
\begin{align*}
\frac{b^*-a}{a(b^*-1)} +  \frac{1}{(b^*)^2} \geq \frac{b^*-a+1}{ab^*} &\iff \frac{(b^*)^3-a(b^*)^2+ab^*-a}{(b^*-1)b^*} \geq b^*-a+1 \\
&\iff (b^*)^3-a(b^*)^2+ab^*-a \geq (b^*)^3-(b^*)^2-a(b^*)^2+ab^*+(b^*)^2-b^*\\
&\iff -a \geq -b^*,
\end{align*}
which holds by assumption. \hfill\Halmos
\end{proof}
}

\revedit{
\subsection{Computational Complexity of Determining an Optimal Threshold}\label{apx:computational-complexity}
\begin{proposition}\label{prop:unimodality}
Suppose $K = 1$. Assuming all values of $\phi(\tau)$ and all partial sums $\sum_{i=0}^\tau \frac{1}{\phi(i)}$, for \mbox{$\tau\in \{0,\ldots,N_{\max}\}$}, are precomputed, the optimal threshold can be computed in $O(\log(N_{\max}))$ operations.
\end{proposition}

\begin{proof}{\it Proof.}
We prove that $R(N)$ is unimodal, i.e., $R(N+1)-R(N) \leq 0 \implies R(N+2)-R(N+1) \leq 0$. We have:
\begin{align*}
R(N+1)-R(N) = \frac{N+1}{\sum_{\tau=0}^{N+1} \frac{1}{\phi(\tau)}}-\frac{N}{\sum_{\tau=0}^{N} \frac{1}{\phi(\tau)}}.
\end{align*}
Hence,
\begin{align*}
R(N+1)-R(N) \leq 0 &\iff (N+1)\sum_{\tau=0}^{N} \frac{1}{\phi(\tau)} - N\sum_{\tau=0}^{N+1} \frac{1}{\phi(\tau)} \leq 0\\
&\iff \sum_{\tau=0}^N\frac{1}{\phi(\tau)}-\frac{N}{\phi(N+1)} \leq 0.
\end{align*}
Similarly, we have:
\begin{align*}
R(N+2)-R(N+1) \leq 0 &\iff \sum_{\tau=0}^{N+1}\frac{1}{\phi(\tau)}-\frac{N+1}{\phi(N+2)} \leq 0 \\
&\iff \sum_{\tau=0}^{N}\frac{1}{\phi(\tau)} + \frac{1}{\phi(N+1)}-\frac{N+1}{\phi(N+2)} \leq 0.
\end{align*}
Using the fact that $\phi(N+2) \leq \phi(N+1)$:
\begin{align*}
 \sum_{\tau=0}^{N}\frac{1}{\phi(\tau)} + \frac{1}{\phi(N+1)}-\frac{N+1}{\phi(N+2)} \leq \sum_{\tau=0}^{N}\frac{1}{\phi(\tau)}-\frac{N}{\phi(N+1)}.
\end{align*}
Therefore, $R(N+1)-R(N) \leq 0 \implies R(N+2)-R(N+1) \leq 0$, as claimed. This implies that $\arg\max_{N\in [N_{\max}]} R(N) = \inf\{N \in [N_{\max}]: R(N+1)-R(N) \leq 0\}$ can be computed via binary search over $N \in [N_{\max}]$, resulting in a computational complexity of $O(\log(N_{\max}))$, assuming all partial sums and values of $\phi(\tau)$ are precomputed. \hfill\Halmos
\end{proof}
}

\revedit{
\subsection{Impact of Personalization on Customer Utility}\label{apx:price-of-personalization}

In this section we study the impact of personalization on a natural proxy for customer utility, the long-run fraction of time a customer receives a free product under the rewards program. For type $k \in [K]$, let $\mathcal{U}_k(N)$ denote this quantity, for a fixed threshold $N$. The following proposition provides a closed-form expression for $\mathcal{U}_k(N)$. We defer its proof to Appendix \ref{apx:utility-proxy}.
\begin{proposition}\label{prop:utility-proxy}
Fix type $k \in [K]$ and threshold $N \in [N_{\max}]$. The long-run fraction of time a type-$k$ customer receives a free product is given by
\begin{align*}
\mathcal{U}_k(N) &= \frac{1}{\sum_{\tau=0}^N\frac{1}{\phi_k(\tau)}}.
\end{align*}
\end{proposition}

For this utility proxy, we measure the impact of personalization by computing, for the most disadvantaged type, how much more frequently the optimal uniform scheme gives out a free item than the optimal personalized scheme does.\footnote{\revedit{In the case where there are multiple revenue-maximizing thresholds, we assume the seller breaks ties in favor of customer utility.}} We henceforth refer to this metric as the {\it price of personalization}.
\begin{definition}[Price of Personalization]
The price of personalization of a given instance is given by:
\begin{align*}
PoP = \max_{k \in [K]}\frac{\mathcal{U}_k(N^*)}{\mathcal{U}_k(N_k^*)},
\end{align*}
where we use the convention that $N_1^* \leq N_2^* \leq \ldots \leq N_K^*$.
\end{definition}

\Cref{prop:pop-bound} provides an instance-dependent lower bound on the price of personalization, and implies that while the {value} of personalization from the seller's perspective is bounded, from the customer's perspective, the price of personalization can be arbitrarily high. We defer the proof of this fact to Appendix \ref{apx:pop-bound}.

\begin{proposition}\label{prop:pop-bound}
For any instance defined by optimal personalized thresholds $(N_k^*)_{k\in[K]}$ and optimal uniform threshold $N^*$, the price of personalization is lower bounded by
\[PoP \geq \frac{N_K^*+1}{N^*+1}.\]
Moreover, there exists an instance for which the price of personalization is unbounded.
\end{proposition}

We conclude this section by exploring the price of personalization numerically, for the same setup as the one used in \Cref{ssec:pof-numerics}. In \Cref{fig:apof_wpof_hist} we plot a histogram of the price of personalization for the 10,000 randomly generated instances described in \Cref{eq:two-type-setup}. These results confirm the intuition that personalization can significantly harm a subset of customers in the population. In particular, on average, the price of personalization is approximately 2.73; in words, this means that the most disadvantaged type redeems almost three times as frequently under the optimal uniform scheme as under the optimal personalized scheme. Moreover, the median price of personalization is approximately 2.16, and in over 25\% of instances the PoP exceeds 3.08. In our experiments, the most disadvantaged type of customer is the ``frequent'' type-2 customer who, under the optimal personalized scheme, is allocated a high redemption threshold (frequently $N_2^* = N_{\max} = 20$). In contrast, the optimal uniform scheme allocates a much lower redemption threshold (frequently $N^* \leq 3$), which explains the high price of personalization.  


\begin{figure}
\centering
\includegraphics[scale=0.45]{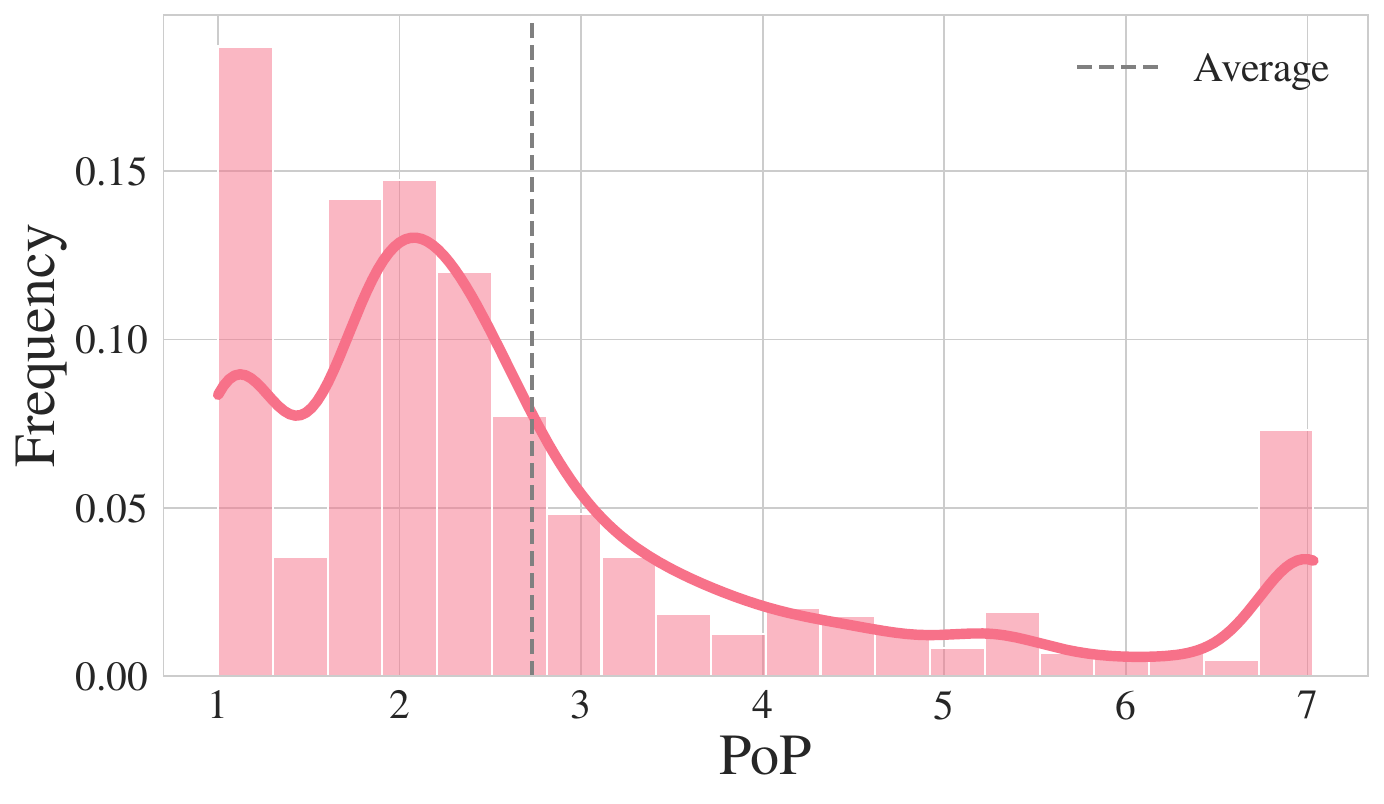}
\caption{\centering\revedit{Distribution of the price of personalization across all 10,000 randomly generated instances.}
\label{fig:apof_wpof_hist}}
\end{figure}


\Cref{tab:u-rho1}, which reports the average and maximum PoP across all randomly generated instances as a function of the proportion of type-1 customers $\rho_1$, shows that the price of personalization depends significantly on the heterogeneity of the population. In particular, when $\rho_1$ is small, personalization is not detrimental to the customer population. This observation is intuitive: when the fraction of type-1 customers is small, the optimal uniform threshold is $N^* \approx N_2^* \gg N_1^*$. Hence, the optimal uniform scheme generates similar utility for type-2 customers as the optimal personalized scheme. As $\rho_1$ increases, however, the optimal uniform threshold decreases. In this case, personalization becomes more detrimental to type-2 customers. Indeed, in the worst case, when $\rho_1 = 0.9$ and type-1 customers dominate, under the optimal uniform scheme, $N^* \approx 1 \ll N_2^* \approx 20$, and type-2 customers on average redeem the free item over four times as frequently as under the personalized scheme.

\begin{table}
\centering
\revedit{
\begin{tabular}{@{}cccccccccc@{}}
\toprule
$\rho_1$         & {0.1} & {0.2} & {0.3} & {0.4} & {0.5} & {0.6} & {0.7} & {0.8} & {0.9} \\ \midrule
{Average PoP} & 1.0408       & 1.0922       & 1.1743       &1.4815       &2.7262       & 3.3550       &4.0537       & 4.2530       & 4.3651       \\
{Max. PoP}     & 1.3333       & 1.6247       &2.2950       & 5.2719       & 7.0309       & 10.5       & 10.5461       & 10.5461       & 10.5461       \\ \bottomrule
\end{tabular}
}
\caption{\revedit{Dependence of the price of personalization on fraction of type-1 customers across all 10,000 randomly generated instances.}}
\label{tab:u-rho1}
\end{table}

\subsubsection{Proof of \Cref{prop:utility-proxy}}\label{apx:utility-proxy}
\begin{proof}{\it Proof.}
Recall, $p_k(\tau;N)$ denotes the long-run average fraction of time a type-$k$ customer spends with $\tau$ points remaining to redemption, given threshold $N$. Noting that a customer receives a free product if and only if they are in state $\tau = 0$ and they make a purchase, we have:
\begin{align*}
\mathcal{U}_k(N) &= \phi_k(0) \cdot p_k(0; N) = \phi_k(0)\cdot\frac{1}{\phi_k(0)\sum_{\tau'=0}^N\frac{1}{\phi_k(\tau')}}=\frac{1}{\sum_{\tau'=0}^N\frac{1}{\phi_k(\tau')}},
\end{align*}
where the second equality follows from \Cref{prop:closed-form-lr-avg-per-cust}. \hfill\Halmos
\end{proof}

\subsubsection{Proof of \Cref{prop:pop-bound}}\label{apx:pop-bound}
\begin{proof}{\it Proof.}
Fix any $k \in [K]$ such that $N^* \leq N_k^*$. (Note that such a $k$ must exist by unimodality of $R_k(N)$, established in the proof of \Cref{prop:unimodality}.) By \Cref{prop:utility-proxy},
\begin{align*}
\frac{\mathcal{U}_k(N^*)}{\mathcal{U}_k(N_k^*)} &= \frac{\sum_{\tau=0}^{N_k^*}\frac{1}{\phi_k(\tau)}}{\sum_{\tau=0}^{N^*}\frac{1}{\phi_k(\tau)}}\\
&= 1 + \frac{\sum_{\tau=N^*+1}^{N_k^*}\frac{1}{\phi_k(\tau)}}{\sum_{\tau=0}^{N^*}\frac{1}{\phi_k(\tau)}}\\
&\geq 1+\frac{\frac{1}{\phi_k(N^*)}(N_k^*-N^*)}{\frac{1}{\phi_k(N^*)}(N^*+1)} \\
&=\frac{N_k^*+1}{N^*+1}, 
\end{align*}
where the inequality follows from monotonicity of $\phi_k(\cdot)$. 
The proof of the lower bound is completed by observing that, for all $k$ such that $N_k^* < N^*$, $\frac{\mathcal{U}_k(N^*)}{\mathcal{U}_k(N_k^*)} \leq 1$. Hence, the most disadvantaged type must be such that $N_k^* \geq N^*$. Using the convention that $N_1^* \leq N_2^* \leq \ldots \leq N_K^*$, we obtain our final bound.

The instance for which the price of personalization is unbounded is precisely the tight, two-type instance from \Cref{thm:price-of-fairness}, in which  the optimal personalized scheme sets $N_2^* = +\infty$ for the ``frequent'' type of customer. In contrast, there exists an optimal uniform threshold for which $N^* = 1$, resulting in an unbounded price of personalization. We refer the reader to the proof of \Cref{lem:pers-tight} for the formal construction.\hfill\Halmos
\end{proof}
}

\medskip

\section{\Cref{sec:learning} Omitted Proofs}

\subsection{Proof of \Cref{thm: lower bound}}\label{apx:lb}

\begin{proof}{\it Proof.}
We construct two instances, both of which assume the population is homogeneous (i.e., \mbox{$K = 1$}). As a result, we omit the dependence on $k$ in the remainder of the proof. 

Fix $\Delta \in (0,1/3]$. Both of our instances have $N_{\max} = 2$, and are respectively defined by parameters $\beta = (\beta_1,\beta_2)$ and $\beta' = (\beta_1',\beta_2')$ such that
\begin{align*}
&\phi(\tau) = \beta_1-\beta_2\tau \qquad \text{and} \qquad \phi'(\tau) = \beta_1' - \beta_2'\tau \quad \forall \ \tau \in \{0,1,2\},
\end{align*}
with $\bar{\phi} = \bar{\phi}' = 1/100$, $\beta_1 = \beta_1' = 1/2$, $\beta_2 = \sqrt{\frac{1-\Delta}{8}}-\frac12$ and $\beta_2' = \sqrt{\frac{1+\Delta}{8}}-\frac12$.

It is straightforward to verify that \Cref{asp:glm} holds for both instances, for any $\Delta \in (0,\frac13]$. Moreover, for clarity of exposition we abuse notation and let $R(N;\beta)$ and $R(N;\beta')$ respectively denote the long-run average revenue for the instances defined by $\beta$ and $\beta'$. We rely on the following lemma to bound the regret for each instance, deferring its proof to Appendix \ref{apx:rev-diff}.
\begin{lemma}\label{lem:rev-diff}
For the two instances defined above, the following hold:
\begin{align*}
&R(1;\beta)-R(2;\beta) = \frac{\Delta  \sqrt{1-\Delta}}{2\prns{\sqrt{1-\Delta} +\sqrt{2}} \prns{ \sqrt{2-2\Delta}-\Delta} } > 0\\
&R(1;\beta')-R(2;\beta') = \frac{- \Delta  \sqrt{1+\Delta}}{2\prns{\sqrt{\Delta+1} +\sqrt{2}} \prns{\Delta + \sqrt{2+2\Delta}} } < 0.
\end{align*}
Moreover, for both instances, $R(2;\beta)-\bar{\phi} > 0$ and $R(1;\beta')-\bar{\phi}' > 0$.
\end{lemma}

\Cref{lem:rev-diff} implies that $N^* = 1$ for the first instance and $N^* = 2$ for the second, with the no-loyalty option yielding the least revenue in both cases. 

We now introduce some additional notation. Fix a policy $\pi$. {Since $\policy$ is fixed, we remove the dependence of all quantities on $\policy$ in the notation throughout this proof.} Let $J_n(t) = \sum_{s=1}^t \mathds{1}\braces{N_t^\pi = n}$ be the number of times threshold $n$ was chosen by $\policy$ by the end of round $t$.  We use  $\pr_\beta$ to denote the probability measure on the $\sigma$-algebra generated by the random trajectory $(N_t,\tau_{jt}, X_{jt}, \ \forall \ t \in [T], j \in \mathcal{M})$ induced by $\policy$ over the $T$ rounds of interaction when the true parameter is $\beta$, and $\pr_{\beta'}$ the probability measure when the true parameter is $\beta'$. Finally, we abuse notation and let $\expect_{\beta}\bracks{\cdot}$ denote the expectation when the true parameter is $\beta$, and define $\text{Reg}_T(\policy, \beta) = \expect_\beta\bracks{MT\mathcal{R}^{\text{non-pers}}-\sum_{t\in [T]} MR(N_t)}$ to be the expected regret of $\policy$ under $\beta$ (and analogously for $\beta'$). 

When the true parameter is $\beta$, \Cref{lem:rev-diff} implies that $\pi$ incurs regret in all rounds $t$ such that \mbox{$N_t \in \{2, + \infty\}$}. Since $R(2;\beta)>\bar{\phi}$, this implies: 
\begin{align}\label{eq:threshold-1-loss}
\text{Reg}_T(\policy, \beta) &\geq |R(1;\beta)-R(2;\beta)|\cdot M\cdot \expect_\beta\bracks{J_2(T)+J_{\infty}(T)} \notag \\
&\geq |R(1;\beta)-R(2;\beta)|\cdot\frac{MT}{2} \cdot \pr_\beta\prns{J_1(T)\le\frac{T}{2}},
\end{align}
where the final inequality follows from Markov's inequality.

Similarly, when the true parameter is $\beta'$, \Cref{lem:rev-diff} implies that $\pi$ incurs regret in all rounds $t$ such that $N_t \in \{1, +\infty\}$. Since $R(1;\beta') > \bar{\phi}$, we have in this case:
\begin{align}\label{eq:threshold-2-loss}
         \text{Reg}_T(\policy, \beta')&\ge |R(1;\beta')-R(2;\beta')|\cdot M\cdot \expect_{\beta'}\bracks{J_1(T)+J_{\infty}(T)} \notag \\
         &\geq |R(1;\beta')-R(2;\beta')|\cdot\frac{MT}{2} \cdot \pr_{\beta'}\prns{J_1(T)>\frac{T}{2}},
\end{align}
where the final inequality uses the loose lower bound $J_{\infty}(T) \geq 0$, and similarly uses Markov's inequality.

We bound $|R(1;\beta)-R(2;\beta)| $ and $|R(1;\beta')-R(2;\beta')|$, again using \Cref{lem:rev-diff}. Namely:
    \begin{align*}
      \abs{R(1;\beta) - R(2;\beta) } &= \abs{\frac{\Delta  \sqrt{1-\Delta}}{2\prns{\sqrt{1-\Delta} +\sqrt{2}} \prns{ \sqrt{2-2\Delta}-\Delta} }} \\
      &= \abs{\frac{\Delta}{2\prns{\sqrt{1-\Delta}+\sqrt{2}} \prns{ \sqrt{2}-\Delta/\sqrt{1-\Delta}} }} \\
      &\geq \frac{\Delta}{2(2+\sqrt{2})}\\
      &\geq \frac{\Delta}{10},
    \end{align*}
    where the first inequality uses the fact that $2\prns{\sqrt{1-\Delta} +\sqrt{2}} \prns{ \sqrt{2}-\Delta/\sqrt{1-\Delta}}$ is maximized at $\Delta = 0$ over the range $[0,1/2]$. Similarly,
    \begin{align*}
    \abs{R(1;\beta') - R(2;\beta') } 
    &= \abs{\frac{- \Delta  \sqrt{1+\Delta}}{2\prns{\sqrt{\Delta+1} +\sqrt{2}} \prns{\Delta + \sqrt{2+2\Delta}} }} \\
    &= \abs{\frac{- \Delta }{2\prns{\sqrt{\Delta+1} +\sqrt{2}} \prns{\Delta/\sqrt{1+\Delta} + \sqrt{2}} }} \\
    &\geq \frac{\Delta}{5+8/\sqrt{3}}\\
    &\geq \frac{\Delta}{10},
    \end{align*}
     where the first inequality uses the fact that $2\prns{\sqrt{\Delta+1} +\sqrt{2}} \prns{\Delta/\sqrt{1+\Delta} + \sqrt{2}} $ is maximized at $\Delta = 1/2$ over the range $[0,1/2]$.

     Plugging these two bounds into \eqref{eq:threshold-1-loss} and \eqref{eq:threshold-2-loss}, respectively, and summing, we obtain:
     \begin{align*}
\text{Reg}_T(\pi,\beta) + \text{Reg}_T(\pi,\beta') &\geq \frac{\Delta MT}{20}\left(\pr_{\beta}\prns{J_1(T)\leq\frac{T}{2}}+\pr_{\beta'}\prns{J_1(T)>\frac{T}{2}}\right).
\end{align*}
Let $D\prns{\pr_\beta, \pr_{\beta'}}$ denote the relative entropy between the measures $\pr_\beta$ and $\pr_{\beta'}$. By the Bretagnolle-Huber inequality \citep[Theorem 14.2]{lattimore2020bandit}, we have:
     \begin{align}\label{eq:to-combine}
         \text{Reg}_T(\policy, \beta) + \text{Reg}_T(\policy, \beta')
         &\geq \frac{\Delta M T}{40}\exp\prns{-D\prns{\pr_\beta, \pr_{\beta'}}}.
     \end{align}

    Hence, it remains to upper bound the relative entropy $D\prns{\pr_\beta, \pr_{\beta'}}$. Let $p_\beta$ and $p_{\beta'}$ respectively denote the probability mass functions on any realized trajectory of policy $\pi$ under parameters $\beta$ and $\beta'$. 
    Moreover, let $\mathcal{H}_t = (N_1, \tau_{11}, \dots, \tau_{M1}, X_{11}, \dots, X_{M1}, \dots, N_t, \tau_{1t}, \dots, \tau_{Mt}, X_{1t}, \dots, X_{Mt})$ be the random history of all thresholds, points to redemption, and purchase decisions of every customer up until the end of period $t$. We use the conventions that $h_t$ denotes a realization of $\mathcal{H}_t$, and  $\mathcal{H}_0$ is the empty set.
    By definition,
    \begin{align*}
        D\prns{\pr_\beta, \pr_{\beta'}} &= \sum_{h_T}p_\beta (h_T)\log\prns{\frac{p_\beta (h_T)}{p_{\beta'}(h_T)}}= \expect_\beta \bracks{\log\prns{\frac{p_\beta (\mathcal{H}_T)}{p_{\beta'}(\mathcal{H}_T)}}},
    \end{align*}
    where the summation in the first equality is over all possible trajectories, with the convention $0\log(\cdot) = 0$. 
    
    Since the $M$ customers are independent, by the chain rule, we can write $p_\beta$ as:
    \begin{align*}
         & p_\beta (h_T) \\
        &= \prod_{t=1}^T \bracks{p_\beta \prns{n_t\mid h_{t-1}} \prns{\prod_{j=1}^M p_\beta \prns{\tau_{jt} \mid n_t, \tau_{1t},\dots, \tau_{j,t-1}, h_{t-1}} }\prns{\prod_{j=1}^M p_{\beta}\prns{x_{jt}\mid n_t, \tau_{1t},\dots, \tau_{Mt}, x_{1t}, \dots, x_{j,t-1}, h_{t-1}}} }, 
    \end{align*}
    where we abuse notation slightly to also let $p_\beta$ denote the conditional probability masses. 

    Note that the points to redemption $\tau_{jt}$ depends only on the history through $N_t$, $N_{t-1},\tau_{j,t-1}$, and $X_{j,t-1}$,\footnote{{Here, $\tau_{jt}$ has dependence on both $N_t$ and $N_{t-1}$ since $\pi$ may vary the threshold in the middle of a customer's redemption cycle.}} and the decision $X_{jt}$ only depends on the points to redemption $\tau_{jt}$. Therefore: 
    \begin{align*}
     p_\beta (h_T) &= \prod_{t=1}^T \bracks{p_\beta \prns{n_t\mid h_{t-1}} \prns{\prod_{j=1}^M p_\beta \prns{\tau_{jt} \mid n_t, n_{t-1},\tau_{j,t-1},x_{j,t-1}} }\prns{\prod_{j=1}^M p_{\beta}\prns{x_{jt}\mid \tau_{jt}} } }.
    \end{align*}
 Similarly,
\begin{align}\label{eq:pbetaprime}
         p_{\beta'} (h_T) 
        &= \prod_{t=1}^T \bracks{p_{\beta'} \prns{n_t\mid h_{t-1}} \prns{\prod_{j=1}^M p_{\beta'} \prns{\tau_{jt} \mid n_t, n_{t-1},\tau_{j,t-1},x_{j,t-1}} }\prns{\prod_{j=1}^M p_{\beta'}\prns{x_{jt}\mid \tau_{jt}} } }
\end{align}
Note that, given $h_{t-1}$, the threshold $N_t$ is solely determined by the fixed policy $\policy$, and is therefore independent of $\beta'$. Similarly, $\tau_{jt}$ is a deterministic function of $n_t$, $n_{t-1},\tau_{j,t-1}$, and $x_{j,t-1}$, with no dependence on the true underlying parameter $\beta'$. We apply these two facts to \eqref{eq:pbetaprime} to conclude that:
\begin{align*}
        p_{\beta'} (h_T)  &= \prod_{t=1}^T \bracks{p_\beta \prns{n_t\mid h_{t-1}} \prns{\prod_{j=1}^M p_\beta \prns{\tau_{jt} \mid n_t, n_{t-1},\tau_{j,t-1},x_{j,t-1}} }\prns{\prod_{j=1}^M p_{\beta'}\prns{x_{jt}\mid \tau_{jt}} } },
\end{align*}
Taking the ratio of $p_\beta(h_T)$ and $p_{\beta'}(h_T)$ and taking the log:
\begin{align*}
    \log\prns{\frac{p_\beta (h_T)}{p_{\beta'}(h_T)}} &= \log\prns{\prod_{t=1}^T \prod_{j=1}^M \frac{p_\beta (x_{jt}\mid \tau_{jt})}{p_{\beta'}(x_{jt} \mid \tau_{jt})}} \\
    &= \sum_{t=1}^T \sum_{j=1}^M \log \frac{p_\beta\prns{x_{jt}\mid \tau_{jt}}}{p_{\beta'}\prns{x_{jt}\mid \tau_{jt}}}.
\end{align*}
Taking expectations on both sides, we have:
\begin{align}\label{eq:tot-rel-entropy}
     D\prns{\pr_\beta, \pr_{\beta'}} = \sum_{t=1}^T \sum_{j=1}^M \expect_\beta \bracks{\log \frac{p_\beta\prns{X_{jt}\mid \tau_{jt}}}{p_{\beta'}\prns{X_{jt}\mid \tau_{jt}}}}.
\end{align}
Since $X_{jt}$ is a Bernoulli random variable, we have:
\begin{align}\label{eq:apply-inverse-pinsk}
    \expect_\beta \bracks{\log \frac{p_\beta\prns{X_{jt}\mid \tau_{jt}}}{p_{\beta'}\prns{X_{jt}\mid \tau_{jt}}}} 
    &= \expect_\beta \bracks{D\prns{\text{Ber}(\beta_1+\beta_2\tau_{jt}), \text{Ber}(\beta_1'+\beta_2'\tau_{jt})}}.
\end{align}
The following lemma upper bounds the relative entropy of these two Bernoulli random variables.

\smallskip

\begin{lemma}[Reverse Pinsker's Inequality, Lemma 6.3 of \cite{csiszar2006context}]\label{lemma: inverse Pinsker}
     The relative entropy between Bernoulli distributions with respective parameters $p\in(0,1)$ and $q\in(0,1/2]$ satisfies:
    \begin{align*}
        D\prns{\text{Ber}(p), \text{Ber}(q)} \le \frac{2}{q}(p-q)^2.
    \end{align*}
\end{lemma}

\smallskip

Noting that $0< \beta_1'+\beta_2'\tau_{jt} \leq \frac12$ for $\tau_{jt} \leq 2$, we apply \Cref{lemma: inverse Pinsker} to \eqref{eq:apply-inverse-pinsk} to obtain:
\begin{align*}
\expect_\beta \bracks{\log \frac{p_\beta\prns{X_{jt}\mid \tau_{jt}}}{p_{\beta'}\prns{X_{jt}\mid \tau_{jt}}}} 
    &\leq \expect_\beta \bracks{\frac{2}{\beta_1'+\beta_2'\tau_{jt}}\prns{\beta_2-\beta_2'}^2\tau_{jt}^2} \tag{Since $\beta_1=\beta_1'$}
\end{align*}
Note that the function $f(x) = \frac{x^2}{\beta_1'+\beta_2'x}$ is increasing in $x$ for $x \leq -2\beta_1'/\beta_2' = -\frac{1}{\sqrt{(1+\Delta)/8}-1/2}$. For \mbox{$\Delta\in(0,1/3]$, $-\frac{1}{\sqrt{(1+\Delta)/8}-1/2} \geq 2$}, which implies that $f(x)$ is increasing for $x \leq 2$. Using this to upper bound the above, we have:
\begin{align*}
  \expect_\beta \bracks{\log \frac{p_\beta\prns{X_{jt}\mid \tau_{jt}}}{p_{\beta'}\prns{X_{jt}\mid \tau_{jt}}}}
    &\leq \frac{8}{\beta_1'+2\beta_2'}\prns{\beta_2-\beta_2'}^2 
    \\
    &= \frac{2\prns{1-\sqrt{1-\Delta^2}}}{\sqrt{(1+\Delta)/2}-1/2} \tag{By definition of $\beta_2,\beta_1',\beta_2'$}\\
    &= \frac{2\Delta^2}{\prns{\sqrt{(1+\Delta)/2}-1/2}\prns{1+\sqrt{1-\Delta^2}}} \\
    &\leq 2(1+\sqrt{2})\Delta^2,
\end{align*}
where the two equalities follow from algebra, and the last inequality follows from the fact that $\prns{\sqrt{(1+\Delta)/2}-1/2}\prns{1+\sqrt{1-\Delta^2}}$ is minimized at $\Delta = 0$ over the range $[0,1/2]$.

Plugging this back into \eqref{eq:tot-rel-entropy}, we have:
\begin{align*}
     D\prns{\pr_\beta, \pr_{\beta'}} \le  2(1+\sqrt{2})\Delta^2 M T.
\end{align*}

Applying this to \eqref{eq:to-combine}:
 \begin{align*}
         \text{Reg}_T(\policy, \beta) + \text{Reg}_T(\policy, \beta')\ge \frac{\Delta M T}{40}\exp\prns{-2(1+\sqrt{2})\Delta^2 M T},
     \end{align*}
whose maximum is achieved at $\Delta = \frac{1}{2\sqrt{(1+\sqrt{2})M T}}$. For this value of $\Delta$, then:
\begin{align*}
&\text{Reg}_T(\policy, \beta) + \text{Reg}_T(\policy, \beta')\ge \frac{\sqrt{MT}}{80\sqrt{1+\sqrt{2}}}\exp\prns{-1/2}\\
\implies &\max\{\text{Reg}_T(\policy, \beta),\text{Reg}_T(\policy, \beta')\} \geq \frac{\sqrt{MT}}{160\sqrt{1+\sqrt{2}}}\exp\prns{-1/2}.
\end{align*}
\hfill\Halmos
\end{proof}

\subsubsection{Proof of \Cref{lem:rev-diff}}\label{apx:rev-diff}
\begin{proof}{\it Proof.}
Consider first the instance for which the true parameter is $\beta$. We have:
    \begin{align*}
       R(1;\beta) - R(2;\beta) &= \frac{1}{\frac{1}{\beta_1} + \frac{1}{\beta_1+\beta_2}} - \frac{2}{\frac{1}{\beta_1} + \frac{1}{\beta_1+\beta_2} +\frac{1}{\beta_1 + 2\beta_2}} \\
      &= -\frac{\beta_1(\beta_1+\beta_2)(\beta_1^2+4\beta_1\beta_2+2\beta_2^2)}{(2\beta_1+\beta_2)(3\beta_1^2+6\beta_1\beta_2+2\beta_2^2)}\\
     &= \frac{\Delta  \sqrt{1-\Delta}}{2\prns{\sqrt{1-\Delta} +\sqrt{2}} \prns{ \sqrt{2-2\Delta}-\Delta} } \\
     &> 0,
    \end{align*}
    where the third equality follows from plugging in the definitions of $\beta_1$ and $\beta_2$.

    We now argue that setting $N = 2$ (weakly) improves upon the no-loyalty option under $\beta$. Noting that $\phi(\tau)\ge1/20$ for all $\tau\in\{0,1,2\}$, we have:
\[
    R(2;\beta)
    =
    \frac{2}{\sum_{\tau=0}^2 1/\phi(\tau)}
    \ge
    \frac{2}{3}\cdot\frac{1}{20}
    >
    \frac{1}{100}
    =
    \bar\phi.
\]

    Similarly, for the second instance:
    \begin{align*}
   R(1;\beta') - R(2;\beta')
   &= -\frac{\beta_1'(\beta_1'+\beta_2')\prns{(\beta_1')^2+4\beta_1'\beta_2'+2(\beta_2')^2}}{(2\beta_1'+\beta_2')\prns{3(\beta_1')^2+6\beta_1'\beta_2'+2(\beta_2')^2}}\\
    &= \frac{- \Delta  \sqrt{1+\Delta}}{2\prns{\sqrt{\Delta+1} +\sqrt{2}} \prns{\Delta + \sqrt{2+2\Delta}} } \\
    &< 0.
    \end{align*}
    It is similarly easy to verify that $\phi'(\tau) \geq \frac{1}{20}$ for all $\tau \in \{0,1,2\}$. Then, comparing $N=1$ and the no-loyalty
option, we have:
\[
    R(1;\beta')
    =
    \frac{1}{\sum_{\tau=0}^1 1/\phi'(\tau)}
    \ge
    \frac{1}{2}\cdot\frac{1}{20}
    >
    \frac{1}{100}
    =
    \bar\phi.
\]
    \hfill\Halmos
\end{proof}

\section{\Cref{sec:ub} Omitted Proofs}

\subsection{Proof of \Cref{prop:tmix-ub}}\label{apx:tmix-ub}

{
\begin{proof}{\it Proof.}
Fix an individual of any type $k \in [K]$ and a redemption threshold $N \leq N_{\max}$. Since our bound holds uniformly for all $k$ and $N$, throughout the proof we suppress the dependence of all quantities on $k$ and $N$. We moreover abuse notation and let $t_{mix} = t_{mix,k}(N)$.

Our proof follows similar lines as the proof used to bound the mixing time of a random walk on the cycle (Section 5.3.2. in \citet{levin2017markov}). Specifically, let the coupling $(X_t, Y_t)_{t=0}^{\infty}$ be such that $(X_t)$ and $(Y_t)$ are Markov chains representing the number of points to redemption. Let $P$ be the corresponding transition matrix. Recall, from \Cref{prop:closed-form-lr-avg-per-cust}, $P$ is defined as:
\begin{align}\label{prop:transition-matrix}
\begin{cases}
P_{\tau,\tau} = 1-\purchaseprob(\tau) \\
P_{\tau,(\tau-1) \text{ mod } (N+1)} = \purchaseprob(\tau) \\
P_{\tau,\tau'} = 0 \quad \forall \ \tau' \not\in \{\tau,\tau-1\}.
\end{cases}
\end{align}

Let $X_0 = x \in \{0,\ldots,N\}$ and $Y_0 = y \in \{0,\ldots,N\}$, with $x > y$ without loss of generality. We assume $X_t$ and $Y_t$ move independently in each period until the two chains collide, at which point they make identical moves in all future periods. Note that this is trivially a valid coupling; we abuse notation and let $\tau_{\text{couple}} = \inf\{t \geq 0: X_t = Y_t\}$. By Corollary 5.5 in \citet{levin2017markov}, 
\begin{align}\label{eq:tmix-to-collision}
t_{mix} \leq 4\max_{x,y}\mathbb{E}_{x,y}[\tau_{\text{couple}}],
\end{align}
where $\mathbb{E}_{x,y}[\cdot]$ is used to denote the expectation of the random variable given that $X_0 = x$ and $Y_0 = y$. 

In order to bound $\mathbb{E}_{x,y}[\tau_{\text{couple}}]$, we define $D_t$ to be the clockwise distance from $X_t$ to $Y_t$ on the $(N+1)$-cycle. Then, for all $t \geq 0$:
\begin{align}\label{eq:cw-dist}
D_{t+1}-D_t = \begin{cases}
+1 \quad &\text{if } X_{t+1} = X_t \text{ and } Y_{t+1} = Y_t-1 \text{ mod } N+1 \\
-1 &\text{if } X_{t+1} = X_t-1 \text{ mod } N+1 \text{ and } Y_{t+1} = Y_t\\
0 &\text{if } X_{t+1}=X_t \text{ and } Y_{t+1} = Y_t\\
0 &\text{if } X_{t+1}=X_t-1 \text{ mod } N+1 \text{ and } Y_{t+1} = Y_t-1 \text{ mod } N+1
\end{cases}
\end{align}
\Cref{fig:coupling} illustrates this construction. 

\begin{figure}
\centering
\begin{tikzpicture}[x=0.75pt,y=0.75pt,yscale=-1,xscale=1]

\draw   (116,106) .. controls (116,92.19) and (127.19,81) .. (141,81) .. controls (154.81,81) and (166,92.19) .. (166,106) .. controls (166,119.81) and (154.81,131) .. (141,131) .. controls (127.19,131) and (116,119.81) .. (116,106) -- cycle ;
\draw   (227,106) .. controls (227,92.19) and (238.19,81) .. (252,81) .. controls (265.81,81) and (277,92.19) .. (277,106) .. controls (277,119.81) and (265.81,131) .. (252,131) .. controls (238.19,131) and (227,119.81) .. (227,106) -- cycle ;
\draw   (338,105) .. controls (338,91.19) and (349.19,80) .. (363,80) .. controls (376.81,80) and (388,91.19) .. (388,105) .. controls (388,118.81) and (376.81,130) .. (363,130) .. controls (349.19,130) and (338,118.81) .. (338,105) -- cycle ;
\draw   (449,104) .. controls (449,90.19) and (460.19,79) .. (474,79) .. controls (487.81,79) and (499,90.19) .. (499,104) .. controls (499,117.81) and (487.81,129) .. (474,129) .. controls (460.19,129) and (449,117.81) .. (449,104) -- cycle ;
\draw    (166,106) -- (225,106) ;
\draw [shift={(227,106)}, rotate = 180] [color={rgb, 255:red, 0; green, 0; blue, 0 }  ][line width=0.75]    (10.93,-3.29) .. controls (6.95,-1.4) and (3.31,-0.3) .. (0,0) .. controls (3.31,0.3) and (6.95,1.4) .. (10.93,3.29)   ;
\draw    (277,106) -- (336,106) ;
\draw [shift={(338,106)}, rotate = 180] [color={rgb, 255:red, 0; green, 0; blue, 0 }  ][line width=0.75]    (10.93,-3.29) .. controls (6.95,-1.4) and (3.31,-0.3) .. (0,0) .. controls (3.31,0.3) and (6.95,1.4) .. (10.93,3.29)   ;
\draw    (388,105) -- (447,105) ;
\draw [shift={(449,105)}, rotate = 180] [color={rgb, 255:red, 0; green, 0; blue, 0 }  ][line width=0.75]    (10.93,-3.29) .. controls (6.95,-1.4) and (3.31,-0.3) .. (0,0) .. controls (3.31,0.3) and (6.95,1.4) .. (10.93,3.29)   ;
\draw    (474,129) .. controls (430.44,227.01) and (128.13,174.08) .. (140.53,132.26) ;
\draw [shift={(141,131)}, rotate = 114.34] [color={rgb, 255:red, 0; green, 0; blue, 0 }  ][line width=0.75]    (10.93,-3.29) .. controls (6.95,-1.4) and (3.31,-0.3) .. (0,0) .. controls (3.31,0.3) and (6.95,1.4) .. (10.93,3.29)   ;

\draw (135,97.4) node [anchor=north west][inner sep=0.75pt]    {$3$};
\draw (246,96.4) node [anchor=north west][inner sep=0.75pt]    {$2$};
\draw (358,97.4) node [anchor=north west][inner sep=0.75pt]    {$1$};
\draw (469,96.4) node [anchor=north west][inner sep=0.75pt]    {$0$};
\draw (133,62.4) node [anchor=north west][inner sep=0.75pt]  [font=\small,color={rgb, 255:red, 7; green, 62; blue, 214 }  ,opacity=1 ]  {$X_{t}$};
\draw (357,61.4) node [anchor=north west][inner sep=0.75pt]  [font=\small,color={rgb, 255:red, 7; green, 62; blue, 214 }  ,opacity=1 ]  {$Y_{t}$};
\draw (128,43.4) node [anchor=north west][inner sep=0.75pt]  [font=\small,color={rgb, 255:red, 245; green, 166; blue, 35 }  ,opacity=1 ]  {$X_{t+1}$};
\draw (462,43.4) node [anchor=north west][inner sep=0.75pt]  [font=\small,color={rgb, 255:red, 245; green, 166; blue, 35 }  ,opacity=1 ]  {$Y_{t+1}$};
\draw (240,62.4) node [anchor=north west][inner sep=0.75pt]  [font=\small,color={rgb, 255:red, 110; green, 201; blue, 0 }  ,opacity=1 ]  {$X_{t+1}$};
\draw (461,61.4) node [anchor=north west][inner sep=0.75pt]  [font=\small,color={rgb, 255:red, 110; green, 201; blue, 0 }  ,opacity=1 ]  {$Y_{t+1}$};
\draw (240,43.4) node [anchor=north west][inner sep=0.75pt]  [font=\small,color={rgb, 255:red, 189; green, 16; blue, 224 }  ,opacity=1 ]  {$X_{t+1}$};
\draw (354,43.4) node [anchor=north west][inner sep=0.75pt]  [font=\small,color={rgb, 255:red, 189; green, 16; blue, 224 }  ,opacity=1 ]  {$Y_{t+1}$};
\end{tikzpicture}
\caption{Illustration of the $(X_t,Y_t)$ coupling for $N = 3$. At time $t$, shown in blue, the clockwise distance $D_t = 2$. In green, we have shown a realization where both $X_t$ and $Y_t$ have decreased by 1, in which case the clockwise distance has not changed and $D_{t+1} = 2$.  In orange, $X_{t+1} = X_t$ and $Y_{t+1} = Y_t-1$. Therefore, the clockwise distance increases by 1 and $D_{t+1} = 3$. Finally, in purple, $X_t$ has decreased by 1, but $Y_t$ has not moved. Therefore, $D_{t+1}$ decreases by 1, with $D_{t+1} = 1$.}\label{fig:coupling}
\end{figure}

Notice that the process $(D_t)$ is a random walk on $\{0,\ldots,N+1\}$. Moreover, $X_t$ and $Y_t$ colliding is equivalent to this walk becoming absorbed at either 0 or $N+1$, as illustrated in \Cref{fig:coupling-2}.

\begin{figure}
\centering
\begin{tikzpicture}[x=0.75pt,y=0.75pt,yscale=-1,xscale=1]

\draw   (116,106) .. controls (116,92.19) and (127.19,81) .. (141,81) .. controls (154.81,81) and (166,92.19) .. (166,106) .. controls (166,119.81) and (154.81,131) .. (141,131) .. controls (127.19,131) and (116,119.81) .. (116,106) -- cycle ;
\draw   (227,106) .. controls (227,92.19) and (238.19,81) .. (252,81) .. controls (265.81,81) and (277,92.19) .. (277,106) .. controls (277,119.81) and (265.81,131) .. (252,131) .. controls (238.19,131) and (227,119.81) .. (227,106) -- cycle ;
\draw   (338,105) .. controls (338,91.19) and (349.19,80) .. (363,80) .. controls (376.81,80) and (388,91.19) .. (388,105) .. controls (388,118.81) and (376.81,130) .. (363,130) .. controls (349.19,130) and (338,118.81) .. (338,105) -- cycle ;
\draw   (449,104) .. controls (449,90.19) and (460.19,79) .. (474,79) .. controls (487.81,79) and (499,90.19) .. (499,104) .. controls (499,117.81) and (487.81,129) .. (474,129) .. controls (460.19,129) and (449,117.81) .. (449,104) -- cycle ;
\draw    (166,106) -- (225,106) ;
\draw [shift={(227,106)}, rotate = 180] [color={rgb, 255:red, 0; green, 0; blue, 0 }  ][line width=0.75]    (10.93,-3.29) .. controls (6.95,-1.4) and (3.31,-0.3) .. (0,0) .. controls (3.31,0.3) and (6.95,1.4) .. (10.93,3.29)   ;
\draw    (277,106) -- (336,106) ;
\draw [shift={(338,106)}, rotate = 180] [color={rgb, 255:red, 0; green, 0; blue, 0 }  ][line width=0.75]    (10.93,-3.29) .. controls (6.95,-1.4) and (3.31,-0.3) .. (0,0) .. controls (3.31,0.3) and (6.95,1.4) .. (10.93,3.29)   ;
\draw    (388,105) -- (447,105) ;
\draw [shift={(449,105)}, rotate = 180] [color={rgb, 255:red, 0; green, 0; blue, 0 }  ][line width=0.75]    (10.93,-3.29) .. controls (6.95,-1.4) and (3.31,-0.3) .. (0,0) .. controls (3.31,0.3) and (6.95,1.4) .. (10.93,3.29)   ;
\draw    (474,129) .. controls (430.44,227.01) and (128.13,174.08) .. (140.53,132.26) ;
\draw [shift={(141,131)}, rotate = 114.34] [color={rgb, 255:red, 0; green, 0; blue, 0 }  ][line width=0.75]    (10.93,-3.29) .. controls (6.95,-1.4) and (3.31,-0.3) .. (0,0) .. controls (3.31,0.3) and (6.95,1.4) .. (10.93,3.29)   ;

\draw (135,97.4) node [anchor=north west][inner sep=0.75pt]    {$3$};
\draw (246,96.4) node [anchor=north west][inner sep=0.75pt]    {$2$};
\draw (358,97.4) node [anchor=north west][inner sep=0.75pt]    {$1$};
\draw (469,96.4) node [anchor=north west][inner sep=0.75pt]    {$0$};
\draw (131,62.4) node [anchor=north west][inner sep=0.75pt]  [font=\small,color={rgb, 255:red, 7; green, 62; blue, 214 }  ,opacity=1 ]  {$X_{t}$};
\draw (356,62.4) node [anchor=north west][inner sep=0.75pt]  [font=\small,color={rgb, 255:red, 7; green, 62; blue, 214 }  ,opacity=1 ]  {$Y_{t}$};
\draw (125,41.4) node [anchor=north west][inner sep=0.75pt]  [font=\small,color={rgb, 255:red, 7; green, 62; blue, 214 }  ,opacity=1 ]  {$\textcolor[rgb]{0.96,0.65,0.14}{X_{t+1}}$};
\draw (463,60.4) node [anchor=north west][inner sep=0.75pt]  [font=\small,color={rgb, 255:red, 7; green, 62; blue, 214 }  ,opacity=1 ]  {$\textcolor[rgb]{0.96,0.65,0.14}{Y_{t+1}}$};
\draw (105,22.4) node [anchor=north west][inner sep=0.75pt]  [font=\small,color={rgb, 255:red, 208; green, 2; blue, 27 }  ,opacity=1 ]  {$\textcolor[rgb]{0.74,0.06,0.88}{X_{t+2}}$};
\draw (143,22.4) node [anchor=north west][inner sep=0.75pt]  [font=\small,color={rgb, 255:red, 7; green, 62; blue, 214 }  ,opacity=1 ]  {$\textcolor[rgb]{0.74,0.06,0.88}{Y_{t+2}}$};
\end{tikzpicture}
\caption{Illustration of collision at $t+2$, for $N = 3$. In period $t$, $D_t = 2$. Over this sample path, $X_t$ stays fixed between $t$ and $t+2$, whereas $Y_t$ decreases by 1 in each period. By \Cref{eq:cw-dist}, then, $D_t$ increases by 1 in each period, until it reaches $D_{t+2} = 4$. Since $X_{s} = Y_s$ for all $s \geq t+2$ by construction, we also have $D_s = 4$ for all $s \geq t+2$, indicating absorption of the random walk $(D_t)$ at the state $N+1$.}\label{fig:coupling-2}
\end{figure}
Let $d$ denote the clockwise distance between the initial states $x$ and $y$. Formally, then:
\begin{align*}
\mathbb{E}_{x,y}[\tau_{\text{couple}}] = \mathbb{E}_{d}\bigg[\inf\big\{t \geq 0: D_t \in \{0,N+1\}\big\}\bigg].
\end{align*}
We upper bound this quantity by considering the absorbing random walk $\widetilde{D}_t$ with initial state $d$ such that, for all $t$:
\begin{align}\label{eq:new-rw-drift}
\widetilde{D}_{t+1}-\widetilde{D}_t = \begin{cases}
+1 \quad&\text{with probability } (1-\mu_{\max})\mu_{\min} \text{ if } \widetilde{D}_t \in \{1,\ldots,N\}\\
-1 &\text{with probability } (1-\mu_{\max})\mu_{\min} \text{ if } \widetilde{D}_t \in \{1,\ldots,N\} \\
0 &\text{otherwise}.
\end{cases}
\end{align}
Noting that this random walk maximizes the probability of staying in the same state in each period, subject to $(\phi(X_t), \phi(Y_t)) \in [\mu_{\min},\mu_{\max}]^2$ for all $t$, it follows that
\begin{align}\label{eq:ub-rws}
\mathbb{E}_{d}\bigg[\inf\big\{t \geq 0: D_t \in \{0,N+1\}\big\}\bigg] \leq \mathbb{E}_{d}\bigg[\inf\big\{t \geq 0: \widetilde{D}_t \in \{0,N+1\}\big\}\bigg] \quad \forall \ d \in \{0,\ldots,N+1\}.
\end{align}
The following lemma explicitly characterizes the expected absorption time for $\widetilde{D}_t$, given its initial state $d$. We defer its proof to Appendix \ref{apx:rw-absorp-time}.
\begin{lemma}\label{lem:rw-absorp-time}
For all $d \in \{0,\ldots,N+1\}$,
\begin{align}\label{eq:rw-absorp-time}
\mathbb{E}_{d}\bigg[\inf\big\{t \geq 0: \widetilde{D}_t \in \{0,N+1\}\big\}\bigg] = \frac{d(N-d+1)}{2(1-\mu_{\max})\mu_{\min}}.
\end{align}
\end{lemma}
Note that the right-hand side of \eqref{eq:rw-absorp-time} attains its maximum at $d = \frac{N+1}{2}$. Using this in \eqref{eq:ub-rws}, we obtain:
\begin{align*}
&\mathbb{E}_{x,y}[\tau_{\text{couple}}] = \mathbb{E}_{d}\bigg[\inf\big\{t \geq 0: D_t \in \{0,N+1\}\big\}\bigg] \leq \frac{(N+1)^2}{8(1-\mu_{\max})\mu_{\min}}\\
\implies &t_{mix} \leq  \frac{(N+1)^2}{2(1-\mu_{\max})\mu_{\min}},
\end{align*}
by \eqref{eq:tmix-to-collision}. Using the upper bound $N \leq N_{\max}$, we obtain the result.
\hfill\Halmos
\end{proof}
}

\subsubsection{Proof of \Cref{lem:rw-absorp-time}}\label{apx:rw-absorp-time}

{
\begin{proof}{\it Proof.}
For ease of notation, let $\widetilde{t}_d = \mathbb{E}_{d}\bigg[\inf\big\{t \geq 0: \widetilde{D}_t \in \{0,N+1\}\big\}\bigg]$, and $\widetilde{\mu} = (1-\mu_{\max})\mu_{\min}$. By \eqref{eq:new-rw-drift}, $\widetilde{t}_d$ is the solution to the following recurrence relation:
\begin{align}\label{eq:new-rw}
\begin{cases}
\widetilde{t}_d = 1 + \widetilde{\mu}\widetilde{t}_{d+1} + \widetilde{\mu}\widetilde{t}_{d-1} + (1-2\widetilde{\mu})\widetilde{t}_d \quad \forall \ d \in \{1,\ldots,N\}\\
\widetilde{t}_0 = \widetilde{t}_{N+1} = 0.
\end{cases}
\end{align}
Note that $\frac{d(N-d+1)}{2\widetilde{\mu}}$ trivially satisfies the boundary conditions of \eqref{eq:new-rw}. We now verify that it satisfies the recurrence relation, which simplifies to:
\begin{align}\label{eq:rec-relation}
2\widetilde{\mu}\widetilde{t}_d = 1+\widetilde{\mu}\left(\widetilde{t}_{d+1}+\widetilde{t}_{d-1}\right).
\end{align}
We have:
\begin{align*}
&2\widetilde{\mu}\frac{d(N-d+1)}{2\widetilde{\mu}} = d(N-d+1),
\end{align*}
and
\begin{align*}
&1+\widetilde{\mu}\left(\frac{(d+1)(N-(d+1)+1)}{2\widetilde{\mu}}+\frac{(d-1)(N-(d-1)+1)}{2\widetilde{\mu}}\right)\\ &= 1 + \frac12 \left((d+1)(N-d)+(d-1)(N-d+2)\right)\\
&=d(N-d+1).
\end{align*}
Therefore, $\frac{d(N-d+1)}{2\widetilde{\mu}}$ satisfies the recurrence relation \eqref{eq:rec-relation}. It is moreover not difficult to see that this solution is unique, which proves the claim.
\hfill\Halmos
\end{proof}
}

\revedit{
\subsection{Proof of \Cref{thm:main-thm}}\label{ssec:proof-of-main-thm}

Before proving the theorem, we provide explicit instantiations of $C_1,\ldots,C_5$ for the construction of the epoch schedule and termination thresholds. Let $\sigma = \frac12$, $C_0 =  \frac{512 G_\mu^2 \sigma^2(1+\threshold_{\max}^2)}{\kappa^4}$ and {$C_{\lambda} = \frac{\mu_{\min}^2}{12\mu_{\max}^2}$}. Then, $C_1,\ldots,C_5$ are defined as follows: 
\begin{align*}
{C_1 = \frac{48}{C_{\lambda}}, \quad C_2 = \frac{8C_0}{\rho_{\min}C_{\lambda}}, \quad C_3 = \frac{2C_0}{\rho_{\min}C_{\lambda}}, \quad C_4 = \frac{810N_{\max}^4}{\rho_{\min}C_{\lambda}^2}, \quad  C_5 = \sum_{k\in[K]}\frac{3\mu_{\max}^2L_\mu\sigma}{\mu_{\min}^2\kappa}\sqrt{\frac{2\rho_k(1+N_{\max}^2)}{C_{\lambda}}}}.
\end{align*}

These constants give rise to the following schedule and termination thresholds:
\begin{align}
    &T_1 = \max\braces{\frac{{48}}{(1-2^{-1/\tmixhat})C_\lambda},\frac{2C_0 (4+\log(1/\delta))}{\rho_{\min} MC_\lambda }, \frac{810N_{\max}^4 {\tmixhat} \log(1/\delta)}{\rho_{\min} M C_\lambda^2}}, \quad T_h = 2^{h-1}T_1 \ \forall \ h \in [H(T)]  \label{eq: epoch}\\
    &\Delta_h = \sum_{k\in[K]}\frac{3\mu_{\max}^2L_\mu\sigma}{\mu_{\min}^2\kappa}\sqrt{\frac{2\rho_k\log(1/\delta)(1+N_{\max}^2)}{C_{\lambda}MT_{h-1}}}, \quad \forall \ h \in \{2,\ldots,H(T)\}. \label{eq: termination thresholds}
\end{align}

For ease of notation, we define $\alpha = 2^{-1/\tmixhat}$, and omit the dependence of all quantities on $\pi$ throughout the proofs of all remaining results.

\smallskip 

\begin{proof}{\it Proof of Theorem~\ref{thm:main-thm}.}
We partition the proof into two cases, depending on whether or not the no-loyalty option is optimal. Recall that in Algorithm~\ref{alg:greedy} the no-loyalty option is selected for all epochs after which the termination condition is satisfied. Let $h_{\infty} = \inf\{h \geq 2: R(+\infty) > R(N_h;\hat{\beta}^{(h)})+\Delta_h\}$ be the epoch in which the termination condition is satisfied, where we use the convention that $h_{\infty} = H(T)+1$ if  \mbox{$R(+\infty) \leq R(N_h;\hat{\beta}^{(h)})+\Delta_h$} for all $h \in \{2,\ldots,H(T)\}$.

\smallskip 

\noindent\textbf{Case 1: $R(+\infty) < R(N^*)$.} We first bound our algorithm's regret as a function of the loss incurred from greedily selecting the threshold in each epoch with respect to the estimated parameters $\hat{\beta}^{(h)}$, as opposed to the true (unknown) parameters $\beta$. Since $N_t^\pi = N_h$ for all $t \in \mathcal{T}_h$, we have:
\begin{align}
 \regret(\policy, M, T)
&= MTR(N^*)- \sum_{h\in[H(T)]}MT_hR(N_h) \notag \\
&\leq MT_1\mu_{\max} + \sum_{h=2}^{H(T)}MT_h\left(R(N^*)-R(N_h)\right) \notag \\
&= MT_1\mu_{\max} + \sum_{h=2}^{h_{\infty}-1}MT_h\left(R(N^*;\beta)-R(N_h;\hat{\beta}^{(h)})+R(N_h;\hat{\beta}^{(h)})-R(N_h;\beta)\right) \notag \\
&\quad + \sum_{h=h_{\infty}}^{H(T)}MT_h\left(R(N^*)- R(+\infty)\right) \notag \\
&\leq MT_1\mu_{\max} + \sum_{h=2}^{h_{\infty}-1}MT_h\left(R(N^*;\beta)-R(N^*;\hat{\beta}^{(h)})+R(N_h;\hat{\beta}^{(h)})-R(N_h;\beta)\right) \label{eq:because-of-greedy} \\
&\quad + \sum_{h=h_{\infty}}^{H(T)}MT_h\left(R(N^*)-R(+\infty)\right) \notag \\
&\leq MT_1\mu_{\max} + 2\sum_{h=2}^{h_{\infty}-1}MT_h\max_{N\in[N_{\max}]}\abs{R(N;\beta)-R(N;\hat{\beta}^{(h)})} \notag \\
&\qquad+ \sum_{h=h_{\infty}}^{H(T)}MT_h\left(R(N^*)-R(+\infty)\right), \label{eq:new-first-step}
\end{align}
where the first inequality uses the trivial bound $R(N) \leq \mu_{\max}$, for all $N\in [N_{\max}]$, and the next equality uses the fact that once the termination condition is satisfied, the average revenue is $R(+\infty)$. Moreover, \Cref{eq:because-of-greedy} follows from the fact that, for $h \leq h_{\infty}-1$, $N_h$ is chosen greedily with respect to $\hat{\beta}_h$, therefore \mbox{$R(N_h;\hat{\beta}_h) \geq R(N^*;\hat{\beta}_h)$}.

\Cref{eq:new-first-step} shows the two sources of loss accumulated by \Cref{alg:greedy}: (i) the loss incurred from optimizing according to $\hat{\beta}^{(h)}$ instead of $\beta$, and (ii) the loss incurred from incorrect early termination. The bulk of the proof lies in bounding the first source of loss. In particular, \Cref{lem:bounded-reward-diff} below establishes that, for all $h$, with sufficiently high probability this loss is upper bounded by $\Delta_h$. The vanishing construction of $\Delta_h$ will then guarantee that our algorithm does not lose too much from mis-estimation in each period.
\begin{restatable}{lemma}{MainLemma}\label{lem:bounded-reward-diff}
Fix $h \in \{2,\ldots,h_{\infty}{\wedge H(T)}\}$. Under the epoch schedule given in \Cref{eq: epoch}, with probability at least $1-7K\delta$,
\begin{align}\label{eq:good-event}
\max_{N\in[N_{\max}]}\abs{R(N;\beta)-R(N;\hat{\beta}^{(h)})} \leq \Delta_h.
\end{align}
\end{restatable}

\smallskip

We defer the proof of the lemma to Appendix \ref{ssec:bounded-reward-diff-proof}, and proceed to use this fact to show our algorithm's regret bound. We define the following ``good event'':
\[\mathcal{E} = \bigg\{\max_{N\in[N_{\max}]}\abs{R(N;\beta)-R(N;\hat{\beta}^{(h)})} \leq \Delta_h \ \forall \ h \leq h_{\infty}\bigg\},\]
which holds with probability at least $1-7K\delta H(T)$, by \Cref{lem:bounded-reward-diff}. Then, by \Cref{eq:new-first-step}, under event $\mathcal{E}$, we have:
\begin{align}\label{eq:delta-intuition}
\regret(\policy,M,T) \leq MT_1\mu_{\max}+2\sum_{h=2}^{h_{\infty}-1}MT_h\Delta_h+\sum_{h=h_{\infty}}^{H(T)}MT_h\left(R(N^*)-R(+\infty)\right).
\end{align}

Suppose the termination condition was satisfied despite the fact that it is optimal to choose a loyalty option, i.e., \mbox{$h_{\infty} \leq H(T)$}. Under $\mathcal{E}$, \mbox{$R(N^*) \leq R(N^*;\hat{\beta}^{(h_{\infty})})+\Delta_h$}. Moreover, since the termination condition was satisfied at $h_{\infty}$, \mbox{$R(+\infty) > R(N^*;\hat{\beta}^{(h_{\infty})})+\Delta_h$}. Putting these two inequalities together, it must be that $R(+\infty) > R(N^*)$, a contradiction. We conclude that the termination condition is never satisfied under $\mathcal{E}$, implying that:
\begin{align*}
\regret(\policy,M,T) \leq MT_1\mu_{\max}+2\sum_{h=2}^{H(T)}MT_h\Delta_h.
\end{align*}

\medskip

\noindent\textbf{Case 2: $R(+\infty) \geq R(N^*)$.} 
{In this case, we have: 
\begin{align*}
 \regret(\policy, M, T)
&= MTR(+\infty)- \sum_{h\in[H(T)]}MT_hR(N_h)\\
&\leq MT_1\mu_{\max} + \sum_{h=2}^{H(T)}MT_h\left(R(+\infty)-R(N_h)\right)\\
&=MT_1\mu_{\max} + \sum_{h=2}^{H(T)}MT_h\left(R(+\infty)-R(N_h;\hat{\beta}^{(h)})+R(N_h;\hat{\beta}^{(h)})-R(N_h;\beta)\right).
\end{align*}
Note that our algorithm only incurs regret for $h < h_{\infty}$, where $R(+\infty) \leq R(N_h;\hat{\beta}^{(h)})+\Delta_h$ by construction. Using this condition above, we obtain, {on the event $\mathcal{E}$}:
\begin{align*}
 \regret(\policy, M, T)&\leq MT_1\mu_{\max} + \sum_{h=2}^{h_{\infty}-1}MT_h\left(\Delta_h + R(N_h;\hat{\beta}^{(h)})-R(N_h;\beta)\right)\\
 &\leq MT_1\mu_{\max} + 2\sum_{h=2}^{h_{\infty}-1}MT_h\Delta_h,
\end{align*}
by \Cref{lem:bounded-reward-diff}.\footnote{Note that \Cref{lem:bounded-reward-diff} holds whether or not $R(+\infty) < R(N^*)$.}
}

Therefore, in both cases we have:
\begin{align*}
\regret(\pi,M,T) &\leq MT_1\mu_{\max} + 2\sum_{h=2}^{H(T)}MT_h\Delta_h \\
&=MT_1\mu_{\max} + 2\sum_{h=2}^{H(T)}MT_h\left(\sum_{k\in[K]}\frac{3\mu_{\max}^2L_\mu\sigma}{\mu_{\min}^2\kappa}\sqrt{\frac{2\rho_k\log(1/\delta)(1+N_{\max}^2)}{C_{\lambda}MT_{h-1}}}\right)\\
&\leq MT_1\mu_{\max} + \frac{12\mu_{\max}^2L_\mu\sigma}{\mu_{\min}^2\kappa}\cdot \sqrt{\frac{\log(1/\delta)(1+N_{\max}^2)}{C_{\lambda}}}\cdot\sum_{h=2}^{H(T)}\sum_{k\in[K]}\sqrt{{\rho_kMT_h}}\\
&=MT_1\mu_{\max} + \frac{{12\mu_{\max}^3L_\mu}\sqrt{{3\log(1/\delta)(1+N_{\max}^2)}}}{\mu_{\min}^3\kappa} \cdot\sum_{h=2}^{H(T)}\sum_{k\in[K]}\sqrt{{\rho_kMT_h}},
\end{align*}
where the second inequality uses the fact that $T_{h-1}\geq T_h/2$ for all $h$, and the final equality plugs in the definition of $C_{\lambda}=\frac{\mu_{\min}^2}{12\mu_{\max}^2}$ and $\sigma = 1/2$.
\hfill\Halmos\end{proof}

\subsubsection{Proof of \Cref{lem:bounded-reward-diff}.}\label{ssec:bounded-reward-diff-proof}

In this section we prove \Cref{lem:bounded-reward-diff}, the driver of all of our results.
\begin{proof}{\it Proof.}
Fix $h \leq h_{\infty}\wedge H(T)$. \Cref{lem:rev-diff-to-est-error} first establishes that the loss incurred from optimizing with respect to the incorrect parameters can be written as a function of the MLE estimation error.

\begin{lemma}\label{lem:rev-diff-to-est-error}
For all $k \in [K]$,
\begin{align}\label{eq:per-type-rew-to-est-error}
        \abs{R_k(\threshold;\hat{\beta}_k^{(h)}) - R_k(\threshold;\beta_k)}  
 \le  \frac{\mu_{\max}^2 L_{\mu}}{\mu_{\min}^2 (N+1)} \sum_{\tau = 0}^N\abs{ (\hat{\beta}_{k,1}^{(h)}-\beta_{k,1}) + (\hat{\beta}_{k,2}^{(h)}-\beta_{k,2})\tau}. 
\end{align}
\end{lemma}

\Cref{lem:rev-diff-to-est-error} leverages the closed-form expression of $R_k(\threshold;\beta_k)$ derived in \Cref{prop:closed-form-lr-avg-per-cust} and Lipschitz continuity of $\mu$. We defer its proof to Appendix \ref{apx:rev-diff-to-est-error}.

To bound this cumulative estimation error (i.e., the right-hand side of \Cref{eq:per-type-rew-to-est-error}), we introduce some additional notation. For any epoch $h < h_{\infty}$, consider the type-$k$ samples observed in epoch $h$, and let \mbox{$V_k^{(h)} = \begin{pmatrix}
        \rho_k M T_h & \sum_{j \in \mathcal{M}_k}\sum_{t\in\tcal_h} \tau_{jt}\\
        \sum_{j \in \mathcal{M}_k}\sum_{t\in\tcal_h} \tau_{jt} \, & \sum_{j \in \mathcal{M}_k}\sum_{t\in\tcal_h} \tau_{jt}^2
    \end{pmatrix}$} be the associated design matrix. The following lemma leverages recent results from the generalized linear contextual bandits literature \citep{li2017provably} by interpreting the points to redemption $\tau_{jt}$ as ``contexts,'' and establishing that bounding our algorithm's estimation error reduces to lower bounding the minimum eigenvalue of the design matrix, denoted by $\lambda_{\min}(V_k^{(h)})$.

\begin{lemma}\label{lem:connecting-to-min-eval}
Fix $\delta>0$, $h \in [h_{\infty}-1]$, and $k\in[K]$. 
If
\begin{align}\label{eq:eval-lb}
    \lambda_{\min}(V_k^{(h)}) \ge C_0\prns{4 + \log \frac{1}{\delta}},
\end{align}
then, with probability at least $1-3\delta$, the maximum likelihood estimator satisfies:
\begin{align}\label{eq:est-error-bound}
    \abs{(\hat{\beta}_{k, 1}^{(h+1)}-\beta_{k,1}) + (\hat{\beta}_{k, 2}^{(h+1)}-\beta_{k,2})\tau} \le  {\frac{3\sigma}{\kappa}}\sqrt{ \frac{\log(1/\delta) (1+\tau^2)}{\lambda_{\min}(V_k^{(h)})}} \quad {\forall \ \tau \leq N_{\max}}.
\end{align}
\end{lemma}

\smallskip
We defer the proof of \Cref{lem:connecting-to-min-eval} to Appendix \ref{apx:connecting-to-min-eval}.
Applying \Cref{lem:connecting-to-min-eval} to \Cref{eq:per-type-rew-to-est-error} when \Cref{eq:eval-lb} holds, we obtain that, with probability at least $1-3\delta$,
\begin{align}\label{eq:est-error-to-min-eval}
        \abs{R_k(\threshold;\hat{\beta}_k^{(h)}) - R_k(\threshold;\beta_k)}  
 \le  \frac{\mu_{\max}^2 L_{\mu}}{\mu_{\min}^2}\cdot{\frac{3\sigma}{\kappa}}\cdot\sqrt{\frac{\log(1/\delta)(1+N_{\max}^2)}{\lambda_{\min}(V_k^{(h-1)})}}, 
\end{align}
where we have additionally used the trivial upper bound $\tau \leq N_{\max}$.

{
\Cref{lem:eval-lb} below establishes that $\lambda_{\min}(V_k^{(h)})$ grows linearly in the type-$k$ sample size of epoch $h$, $\rho_k M T_h$, with probability that is (i) exponentially {\it decreasing} in this sample size, and (ii) exponentially {\it increasing} in the maximum mixing time of the underlying Markov chain $t_{mix}$. By construction of our epoch schedule, we will then use this linear lower bound to show that \Cref{eq:eval-lb} holds with high probability, as $\rho_k MT_h$ grows large. 
}

\begin{lemma}\label{lem:eval-lb}
Fix $h \in [h_{\infty}-1]$, $k\in [K]$. Under the epoch schedule described in \Cref{eq: epoch}: 
\begin{align}\label{eq:final-eval-lb}
\PP\prns{\lambda_{\min}(V_k^{(h)}) \le  \frac{C_\lambda \rho_k M T_h}{2} } \le 4\exp\prns{-\frac{\rho_k M  T_hC_{\lambda}^2}{810  N_{\max}^4 t_{mix}}}.
\end{align}
Moreover,
\begin{align}
 \frac{C_\lambda \rho_k M T_h}{2} \ge C_0\prns{4 + \log \frac{1}{\delta}}.
\end{align}
\end{lemma}

\smallskip

The following lemma results from applying \Cref{lem:eval-lb} to \Cref{eq:est-error-to-min-eval}, and applying a union bound. We defer its algebraic proof to Appendix \ref{apx:getting-there}.

\begin{lemma}\label{cor:getting-there}
For all $N \leq N_{\max}$, with probability at least $1-3\delta K - 4 \sum_{k\in[K]}\exp\prns{-\frac{\rho_k M  T_1C_{\lambda}^2}{810  N_{\max}^4 \tmixhat}}$,
\begin{align*}
\abs{R(N;\beta)-R(N;\hat{\beta}^{(h)})} &\leq \sum_{k\in[K]}\frac{\mu_{\max}^2 L_{\mu}}{\mu_{\min}^2}\cdot{\frac{3\sigma}{\kappa}}\cdot\sqrt{\frac{2\log(1/\delta)(1+N_{\max}^2)\rho_k}{C_{\lambda}MT_{h-1}}} =: \Delta_h.
\end{align*}
\end{lemma}

The result then follows by using the fact that $T_1 \geq \frac{810N_{\max}^4\tmixhat\log(1/\delta)}{\rho_kMC_\lambda^2}$ by construction, which gives that the required bound holds with probability at least $1-3\delta K -4 \delta K = 1-7\delta K$.
\hfill\Halmos
\end{proof}

}

\subsection{\Cref{lem:bounded-reward-diff} Auxiliary Results}

\subsubsection{Proof of \Cref{lem:rev-diff-to-est-error}}\label{apx:rev-diff-to-est-error}

\begin{proof}{\it Proof.}
Fix $k \in [K]$. By \Cref{prop:closed-form-lr-avg-per-cust}, we have:
\begin{align*}
     \abs{R_k(\threshold;\hat{\beta}_k^{(h)}) - R_k(\threshold;\beta_k)}  
    & = \abs{\frac{N}{\sum_{\tau = 0}^N \frac{1}{\mu_k(\hat{\beta}_{k,1}^{(h)} + \hat{\beta}_{k,2}^{(h)}\tau)}} 
    -  \frac{N}{\sum_{\tau = 0}^N \frac{1}{\mu_k(\beta_{k,1} + \beta_{k,2}\tau)}}}\\
    & = \abs{\frac{N \sum_{\tau = 0}^N \frac{\mu_k(\hat{\beta}_{k,1}^{(h)} + \hat{\beta}_{k,2}^{(h)}\tau) - \mu_k(\beta_{k,1} + \beta_{k,2}\tau)}{\mu_k(\hat{\beta}_{k,1}^{(h)} + \hat{\beta}_{k,2}^{(h)}\tau)\mu_k(\beta_{k,1} + \beta_{k,2}\tau)}}{\prns{\sum_{\tau = 0}^N \frac{1}{\mu_k(\hat{\beta}_{k,1}^{(h)} + \hat{\beta}_{k,2}^{(h)}\tau)}}\prns{\sum_{\tau = 0}^N \frac{1}{\mu_k(\beta_{k,1} + \beta_{k,2}\tau)}}}}\\
    & \leq \frac{\mu_{\max}^2}{\mu_{\min}^2}\cdot\frac{N}{(N+1)^2}\cdot \abs{\sum_{\tau = 0}^N \mu_k(\hat{\beta}_{k,1}^{(h)} + \hat{\beta}_{k,2}^{(h)}\tau) - \mu_k(\beta_{k,1} + \beta_{k,2}\tau)}\\
    & \leq \frac{\mu_{\max}^2}{\mu_{\min}^2}\cdot\frac{1}{N+1}\cdot \abs{\sum_{\tau = 0}^N \mu_k(\hat{\beta}_{k,1}^{(h)} + \hat{\beta}_{k,2}^{(h)}\tau) - \mu_k(\beta_{k,1} + \beta_{k,2}\tau)}
\end{align*}
where the first inequality follows from $\mu_k(\hat{\beta}_{k,1}^{(h)} + \hat{\beta}_{k,2}^{(h)}\tau)\mu_k(\beta_{k,1} + \beta_{k,2}\tau) \geq \mu_{\min}^2$ for all $\tau$ in the numerator, and $\prns{\sum_{\tau = 0}^N \frac{1}{\mu_k(\hat{\beta}_{k,1}^{(h)} + \hat{\beta}_{k,2}^{(h)}\tau)}}\prns{\sum_{\tau = 0}^N \frac{1}{\mu_k(\beta_{k,1} + \beta_{k,2}\tau)}} \geq \frac{(N+1)^2}{\mu_{\max}^2}$ in the denominator. Moreover, using the fact that $\mu_k(\cdot)$ is $L_{\mu}$-Lipschitz, we obtain: 
\begin{align*}
        \abs{R_k(\threshold;\hat{\beta}_k^{(h)}) - R_k(\threshold;\beta_k)}  
 \le  \frac{\mu_{\max}^2 L_{\mu}}{\mu_{\min}^2 (N+1)} \sum_{\tau = 0}^N\abs{ (\hat{\beta}_{k,1}^{(h)}-\beta_{k,1}) + (\hat{\beta}_{k,2}^{(h)}-\beta_{k,2})\tau}. 
\end{align*}
\hfill\Halmos
\end{proof}

\subsubsection{Proof of \Cref{lem:connecting-to-min-eval}}\label{apx:connecting-to-min-eval}

\begin{proof}{\it Proof.}
The result follows from Theorem 1 in \citet{li2017provably}, which we state for completeness in \Cref{thm: general mle} (with a general setup provided in Appendix \ref{apx:mle}).

{For any individual $j \in \mathcal{M}_k$, $\tau_{jt}$ corresponds to a feature in our setting, and $X_{jt}$ is its corresponding observation. Given $\tau_{jt}$, $X_{jt} \sim \bern\left(\mu_{k(j)}(\beta_{k(j),1}+\beta_{k(j),2}\tau_{jt})\right)$. Therefore, it is in the exponential family, and the noise is sub-Gaussian with \mbox{$\sigma = 1/2$} \citep{blm}.}

Note that \Cref{thm: general mle} requires the features to have $\ell_2$-norm in $[0,1]$. Therefore, to apply the result we define  $Z_{jt} = \prns{1/\sqrt{1+N^2_{\max}}, \tau_{jt}/\sqrt{1+N^2_{\max}}}$ to be the normalized feature vectors, with \mbox{$\theta^* = \prns{\sqrt{1+N_{\max}^2}\cdot \beta_{k,1}, \sqrt{1+N_{\max}^2}\cdot \beta_{k,2}}$}, and $d = 2$.  Then, applying \Cref{thm: general mle}, we have that if
\begin{align*}
    \lambda_{\min}\prns{\frac{V_k^{(h)}}{1+N_{\max}^2}} \ge \frac{512 G_\mu^2 \sigma^2}{\kappa^4}\prns{4 + \log \frac{1}{\delta}},
\end{align*}
then, with probability at least $1-3\delta$, the MLE computed by our algorithm at the beginning of epoch $h+1$ satisfies, for all $\tau$,
\begin{align*}
    \abs{(\hat{\beta}_{k, 1}^{(h+1)}-\beta_{k,1}) + (\hat{\beta}_{k, 2}^{(h+1)}-\beta_{k,2})\tau}\le  \frac{3\sigma}{\kappa}\sqrt{\log(1/\delta)}\sqrt{(1, \tau) \prns{V_k^{(h)}}^{-1} \begin{pmatrix} 1 \\ \tau\end{pmatrix}}.
\end{align*}
The conclusion follows by noticing that
    \begin{align*}
     \sqrt{(1, \tau)  \prns{V_k^{(h)}}^{-1} \begin{pmatrix} 1 \\ \tau\end{pmatrix}} \le \sqrt{ \frac{1+\tau^2}{\lambda_{\min}(V_k^{(h)})}}. 
    \end{align*}
\hfill\Halmos\end{proof}

\subsubsection{Proof of \Cref{lem:eval-lb}}\label{apx:eval-lb}

The proof of \Cref{lem:eval-lb} leverages the following concentration bounds on the Markov chain representing customers' points to redemption. We defer its proof to Appendix \ref{apx:yichun-bounds}.

\begin{proposition}\label{prop:yichun-bounds}
    Fix type $k\in[K]$ and epoch $h \in [h_{\infty}-1]$. Let $\alpha = 2^{-1/\tmixhat}$. The following high-probability bounds hold, for any $\epsilon > 0$:
    \begin{enumerate}
\item Average time to redemption:\[\PP\prns{\abs{\sum_{j \in \mathcal{M}_k}\sum_{t\in\tcal_h} \tau_{jt} - \rho_k M T_h \exptimeredeem{k}{\threshold_h}} \ge \rho_k M \left( T_h \epsilon + \frac{{4} N_h}{1-\alpha}\right) \mid N_h} \le 2\exp\prns{-\frac{2\rho_k M  T_h \epsilon^2}{45 N_h^2  t_{mix}}}.\]
\item Average squared time to redemption:
    \[ \PP\prns{\abs{\sum_{j \in \mathcal{M}_k}\sum_{t\in\tcal_h} \tau_{jt}^2 - \rho_k M T_h \exptimeredeemsq{k}{\threshold_h}} \ge \rho_k M \prns{T_h \epsilon + \frac{{4} N_h^2}{1-\alpha}} \mid N_h} \le 2\exp\prns{-\frac{2\rho_k M  T_h\epsilon^2}{45  N_h^4  t_{mix}}}.\]
\item Average revenue:
    \[\pr\prns{\abs{\sum_{j\in \mathcal{M}_k} \sum_{t\in \mathcal{T}_h} \bracks{R_k(\threshold_h) -\purchaseprob_k(\tau_{jt}){\mathds{1}\{\tau_{jt} > 0\}}}} \ge \rho_k M \prns{T_h \epsilon + \frac{{4} }{1-\alpha}}} \le 2\exp\prns{-\frac{2\rho_k M  T_h\epsilon^2}{45    t_{mix}}}.\]
    \end{enumerate}
\end{proposition}

\begin{proof}{\it Proof of \cref{lem:eval-lb}.}
Fix $h \in [h_{\infty}-1]$ and $k\in [K]$. 
Solving the characteristic equation of $V_k^{(h)}$, i.e., \mbox{$\abs{\lambda I - V_k^{(h)}} = 0$}, the two eigenvalues of $V_k^{(h)}$ are
\begin{align*}
   \frac{\rho_k M T_h + \sum_{j \in \mathcal{M}_k}\sum_{t\in\tcal_h} \tau_{jt}^2 }{2} \pm \sqrt{\frac{\prns{ \rho_k M T_h-\sum_{j \in \mathcal{M}_k}\sum_{t\in\tcal_h} \tau_{jt}^2 }^2}{4} + \prns{\sum_{j \in \mathcal{M}_k}\sum_{t\in\tcal_h} \tau_{jt}}^2}. 
\end{align*}
Therefore,
\begin{align}\label{eq:mix-0-a}
    \lambda_{\min}(V_k^{(h)}) 
    & = \frac{\rho_k M T_h + \sum_{j \in \mathcal{M}_k}\sum_{t\in\tcal_h} \tau_{jt}^2 }{2} - \sqrt{\frac{\prns{ \rho_k M T_h-\sum_{j \in \mathcal{M}_k}\sum_{t\in\tcal_h} \tau_{jt}^2 }^2}{4} + \prns{\sum_{j \in \mathcal{M}_k}\sum_{t\in\tcal_h} \tau_{jt}}^2}.
\end{align}
Multiplying and dividing the right-hand side by \[\frac{\rho_k M T_h + \sum_{j \in \mathcal{M}_k}\sum_{t\in\tcal_h} \tau_{jt}^2 }{2} + \sqrt{\frac{\prns{ \rho_k M T_h-\sum_{j \in \mathcal{M}_k}\sum_{t\in\tcal_h} \tau_{jt}^2 }^2}{4} + \prns{\sum_{j \in \mathcal{M}_k}\sum_{t\in\tcal_h} \tau_{jt}}^2} \ \ ,\]
we obtain:
\begin{align}
    \lambda_{\min}(V_k^{(h)})& = \frac{\rho_k M T_h\prns{\sum_{j \in \mathcal{M}_k}\sum_{t\in\tcal_h} \tau_{jt}^2} - \prns{\sum_{j \in \mathcal{M}_k}\sum_{t\in\tcal_h} \tau_{jt}}^2 }{\frac{\rho_k M T_h + \sum_{j \in \mathcal{M}_k}\sum_{t\in\tcal_h} \tau_{jt}^2 }{2} + \sqrt{\frac{\prns{ \rho_k M T_h-\sum_{j \in \mathcal{M}_k}\sum_{t\in\tcal_h} \tau_{jt}^2 }^2}{4} + \prns{\sum_{j \in \mathcal{M}_k}\sum_{t\in\tcal_h} \tau_{jt}}^2}} \notag.
\end{align}
By Jensen's inequality, $\prns{\sum_{j \in \mathcal{M}_k}\sum_{t\in\tcal_h} \tau_{jt}}^2 \le \rho_k M T_h \prns{\sum_{j \in \mathcal{M}_k}\sum_{t\in\tcal_h} \tau_{jt}^2}$. We upper bound the final term in the denominator using this fact and simplify, obtaining:
\begin{align}\label{eq:mix-0}
    \lambda_{\min}(V_k^{(h)}) &\geq \frac{\rho_k M T_h\prns{\sum_{j \in \mathcal{M}_k}\sum_{t\in\tcal_h} \tau_{jt}^2} - \prns{\sum_{j \in \mathcal{M}_k}\sum_{t\in\tcal_h} \tau_{jt}}^2}{\rho_k M T_h  + \sum_{j \in \mathcal{M}_k}\sum_{t\in\tcal_h} \tau_{jt}^2} \notag \\
    &\geq \frac{\prns{\sum_{j \in \mathcal{M}_k}\sum_{t\in\tcal_h} \tau_{jt}^2}  - \frac{1}{\rho_k M T_h}\prns{\sum_{j \in \mathcal{M}_k}\sum_{t\in\tcal_h} \tau_{jt}}^2}{1 + \threshold_h^2},
\end{align}
where the last inequality follows from the trivial upper bound $\tau_{jt} \leq N_h$, and by dividing numerator and denominator by $\rho_kMT_h$.

{
We introduce some additional notation. Recall, for $N \in [N_{\max}]$, $\statprob{k}{\tau}{N}$ is used to denote the steady-state probability that a type-$k$ customer has $\tau$ points to redemption remaining, given threshold $N$. Let \mbox{$\exptimeredeem{k}{\threshold}  = \sum_{\tau=0}^{\threshold} \tau \statprob{k}{\tau}{N}$} be the expected points to redemption for this chain in steady state, with \mbox{$\exptimeredeemsq{k}{\threshold} = \sum_{\tau = 0}^{\threshold} \tau^2 \statprob{k}{\tau}{N}$}. We will show that the numerator is ``close'' to its steady-state expectation, $\rho_kMT_h\left(\exptimeredeemsq{k}{\threshold}-\left(\exptimeredeem{k}{\threshold}\right)^2\right)$. It will then suffice to derive a constant lower bound on the steady-state variance of the underlying Markov chain.
}

Having outlined our approach, observe that \eqref{eq:mix-0} implies that, for all $\epsilon > 0$:
\begin{align}\label{eq:eval-prob-bound}
 & \PP\prns{\lambda_{\min}(V_k^{(h)}) \le \frac{\rho_k M T_h}{1+N_h^2}\prns{\exptimeredeemsq{k}{\threshold_h} -\prns{\exptimeredeem{k}{\threshold_h}}^2 - \frac{{12}N_h^2}{(1-\alpha)T_h} - (2N_h+1)\epsilon} \mid N_h} \notag  \\
 &\leq \PP\Bigg(\Big(\sum_{j \in \mathcal{M}_k}\sum_{t\in\tcal_h} \tau_{jt}^2\Big)  - \frac{1}{\rho_k M T_h}\Big(\sum_{j \in \mathcal{M}_k}\sum_{t\in\tcal_h} \tau_{jt}\Big)^2 \notag \\
 &\qquad \qquad  \le {\rho_k M T_h}\prns{\exptimeredeemsq{k}{\threshold_h} -\prns{\exptimeredeem{k}{\threshold_h}}^2 - \frac{{12}N_h^2}{(1-\alpha)T_h} - (2N_h+1)\epsilon} \mid N_h\Bigg) \notag \\
 &= \PP\Bigg(\left(\sum_{j \in \mathcal{M}_k}\sum_{t\in\tcal_h} \tau_{jt}^2-{\rho_k M T_h}\exptimeredeemsq{k}{\threshold_h}\right)  -\left( \frac{1}{\rho_k M T_h}\Big(\sum_{j \in \mathcal{M}_k}\sum_{t\in\tcal_h} \tau_{jt}\Big)^2 - \rho_kMT_h \prns{\exptimeredeem{k}{\threshold_h}}^2 \right) \notag \\
 &\qquad \qquad  \le -{\rho_k M T_h}\prns{\frac{{12}N_h^2}{(1-\alpha)T_h}+ (2N_h+1)\epsilon} \mid N_h\Bigg) \notag \\
 &\leq \PP\Bigg(\sum_{j \in \mathcal{M}_k}\sum_{t\in\tcal_h} \tau_{jt}^2-{\rho_k M T_h}\exptimeredeemsq{k}{\threshold_h} \leq -\rho_kM\left(\frac{{4}N_h^2}{1-\alpha}+T_h\epsilon\right) \mid N_h\Bigg) \notag \\
 & \qquad + \PP\left(\rho_kMT_h \prns{\exptimeredeem{k}{\threshold_h}}^2-\frac{1}{\rho_k M T_h}\Big(\sum_{j \in \mathcal{M}_k}\sum_{t\in\tcal_h} \tau_{jt}\Big)^2 \leq -2N_h\rho_kM\left(\frac{{4}N_h}{1-\alpha}+T_h\epsilon\right) \mid N_h\right),
\end{align}
where the equality follows from some re-arranging, and the final inequality follows from a union bound. Upper bounding \eqref{eq:eval-prob-bound} further, it suffices to bound the distance of the points to redemption and squared points to redemption from their respective means. Namely, to bound:
\begin{align}\label{eq:bounding-two-terms}
&\PP\left(\abs{\sum_{j \in \mathcal{M}_k}\sum_{t\in\tcal_h} \tau_{jt}^2-{\rho_k M T_h}\exptimeredeemsq{k}{\threshold_h}} \geq \rho_kM\left(\frac{{4}N_h^2}{1-\alpha}+T_h\epsilon\right) \mid N_h\right) \notag \\
 & \qquad + \PP\left(\abs{\rho_kMT_h \prns{\exptimeredeem{k}{\threshold_h}}^2-\frac{1}{\rho_k M T_h}\Big(\sum_{j \in \mathcal{M}_k}\sum_{t\in\tcal_h} \tau_{jt}\Big)^2}  \geq 2N_h\rho_kM\left(\frac{{4}N_h}{1-\alpha}+T_h\epsilon\right) \mid N_h\right).
\end{align}

By Part 2 of \Cref{prop:yichun-bounds}, the first term is upper bounded by: 
\begin{align}\label{eq:mix-2}
\PP\prns{\abs{\sum_{j \in \mathcal{M}_k}\sum_{t\in\tcal_h} \tau_{jt}^2 - \rho_k M T_h \exptimeredeemsq{k}{\threshold_h}} \ge \rho_k M \prns{T_h \epsilon + \frac{{4} N_h^2}{1-\alpha}} \mid N_h} \le 2\exp\prns{-\frac{2\rho_k M  T_h\epsilon^2}{45  N_h^4  t_{mix}}}.
\end{align}

We now bound the second term in \eqref{eq:bounding-two-terms}.
Note that
\begin{align}\label{eq:diff-of-square}
 & \abs{\frac{1}{\rho_k M T_h}\prns{\sum_{j \in \mathcal{M}_k}\sum_{t\in\tcal_h} \tau_{jt}}^2  - \rho_k M T_h \prns{\exptimeredeem{k}{\threshold_h}}^2} \notag \\
 & =\frac{1}{\rho_k M T_h}\abs{\prns{\sum_{j \in \mathcal{M}_k}\sum_{t\in\tcal_h} \tau_{jt}}^2  -  \bigg({\rho_k M T_h\exptimeredeem{k}{\threshold_h}}\bigg)^2} \notag \\ 
 & =\frac{1}{\rho_k M T_h} \abs{\sum_{j \in \mathcal{M}_k}\sum_{t\in\tcal_h} \tau_{jt} - \rho_k M T_h \exptimeredeem{k}{\threshold_h}} \prns{\sum_{j \in \mathcal{M}_k}\sum_{t\in\tcal_h} \tau_{jt} + \rho_kMT_h\exptimeredeem{k}{\threshold_h}} \notag \\
 & \leq  \frac{1}{\rho_k M T_h} \abs{\sum_{j \in \mathcal{M}_k}\sum_{t\in\tcal_h} \tau_{jt} - \rho_k M T_h \exptimeredeem{k}{\threshold_h}}\cdot 2\rho_kMT_hN_h  \notag \\
 & = 2 \threshold_h  \abs{\sum_{j \in \mathcal{M}_k}\sum_{t\in\tcal_h} \tau_{jt} - \rho_k M T_h \exptimeredeem{k}{\threshold_h}},
\end{align}
where the inequality follows from loosely upper bounding $\tau_{jt}$ by $N_h$, for all $t$. This implies:
\begin{align}\label{eq:mix-1}
   &\pr\prns{\abs{\frac{1}{\rho_k M T_h}\prns{\sum_{j \in \mathcal{M}_k}\sum_{t\in\tcal_h} \tau_{jt}}^2  - \rho_k M T_h \prns{\exptimeredeem{k}{\threshold_h}}^2}\ge 2 N_h \rho_k M \prns{T_h \epsilon + \frac{{4} N_h}{1-\alpha}} \mid N_h} \notag \\
   &\leq \pr\prns{ \abs{\sum_{j \in \mathcal{M}_k}\sum_{t\in\tcal_h} \tau_{jt} - \rho_k M T_h \exptimeredeem{k}{\threshold_h}}\ge  \rho_k M \prns{T_h \epsilon + \frac{{4} N_h}{1-\alpha}} \mid N_h} \notag \\
   &\leq 2\exp\prns{-\frac{2\rho_k M  T_h\epsilon^2}{45  N_h^2  t_{mix}}},
\end{align}
where the final inequality follows from Part 1 of \Cref{prop:yichun-bounds}.

\medskip

\noindent\textbf{Putting it all together.} Applying \eqref{eq:mix-1} and \eqref{eq:mix-2} to \eqref{eq:bounding-two-terms}, we obtain:
\begin{align}\label{eq:high-prob-ub-lam}
    & \PP\prns{\lambda_{\min}(V_k^{(h)}) \le \frac{\rho_k M T_h}{1+N_h^2}\prns{\exptimeredeemsq{k}{\threshold_h} -\prns{\exptimeredeem{k}{\threshold_h}}^2 - \frac{{12}N_h^2}{(1-\alpha)T_h} - (2N_h+1)\epsilon} \mid N_h} \notag \\
    & \leq2\exp\prns{-\frac{2\rho_k M  T_h\epsilon^2}{45  N_h^2  t_{mix}}} + 2\exp\prns{-\frac{2\rho_k M  T_h\epsilon^2}{45 N_h^4  t_{mix}}} \notag \\
    & \leq 4\exp\prns{-\frac{2\rho_k M  T_h\epsilon^2}{45  N_{\max}^4  t_{mix}}},
\end{align}
where the last inequality uses the fact that $N_h \leq N_{\max}$.

To complete the proof of the lemma, it suffices to show that this high-probability lower bound on $\lambda_{\min}(V_k^{(h)})$ is indeed linear in $\rho_k M T_h$. Equivalently, it suffices to show that, given our definition of $T_h$, there exists $\epsilon > 0$ such that the expression
\begin{align}\label{eq:to-lower-bound}
\exptimeredeemsq{k}{N} -\prns{\exptimeredeem{k}{N}}^2 - \frac{{12}N^2}{(1-\alpha)T_h} - (2N+1)\epsilon
\end{align}
is lower bounded by a constant, for all $N \in [N_{\max}]$.

{
Note that the first two terms in \eqref{eq:to-lower-bound} correspond to the steady-state variance of the underlying Markov chain, as alluded to above. This reinforces the intuition that it is the variability in customers' natural redemption cycles that allows for effective learning. \Cref{lem:sample-variance-bound} below provides a uniform lower bound on this variance. We defer its proof to Appendix \ref{apx:sample-variance-bounds}.
}

\begin{lemma}\label{lem:sample-variance-bound}
For all $k \in [K], N\in[N_{\max}]$,
\[\exptimeredeemsq{k}{\threshold} -\prns{\exptimeredeem{k}{\threshold}}^2 \geq \frac{\mu_{\min}^2}{12 \mu_{\max}^2}\cdot {N(N+2)}.\]
\end{lemma}

We use \Cref{lem:sample-variance-bound} in the high-probability lower bound on $\lambda_{\min}(V_k^{(h)})$, as follows:
\begin{align}\label{eq:getting-to-linear}
 &\frac{\rho_k M T_h}{1+N_h^2}\prns{\exptimeredeemsq{k}{\threshold_h} -\prns{\exptimeredeem{k}{\threshold_h}}^2 - \frac{{12}N_h^2}{(1-\alpha)T_h} - (2N_h+1)\epsilon} \notag \\ &\geq \rho_k M T_h \cdot \min_{N\in[N_{\max}]} \bigg\{\frac{\mu_{\min}^2}{12 \mu_{\max}^2}\cdot {\frac{N(N+2)}{1+N^2}} - \frac{{12}N^2}{(1-\alpha)T_h(1+N^2)} - \frac{(2N+1)\epsilon}{1+N^2}\bigg\}\notag \\
&\geq \rho_k M T_h \cdot \bigg(C_{\lambda} - \frac{{12}}{(1-\alpha)T_h} - \frac{3\epsilon}{2}\bigg),
\end{align}
where the second inequality uses the fact that {$\frac{N(N+2)}{1+N^2} \geq 1$, and $C_{\lambda} = \frac{\mu_{\min}^2}{12\mu_{\max}^2}$}. It moreover uses the upper bounds $\frac{N^2}{N^2+1}\le 1$ and \mbox{$\frac{2N+1}{1+N^2}\le 3/2$}, for all $N \in [N_{\max}]$. 

Letting $\epsilon=C_{\lambda}/6$, and noting that $T_h \geq T_1 \geq \frac{{48}}{(1-\alpha)C_\lambda}$ by construction, we have:
\begin{align*}
\eqref{eq:getting-to-linear} &\geq \rho_k M T_h \cdot \bigg(C_{\lambda} - \frac{C_{\lambda}}{4} - \frac{C_{\lambda}}{4}\bigg) = \frac{C_{\lambda}\rho_kMT_h}{2}.
\end{align*}

Applying this to \eqref{eq:high-prob-ub-lam}, we obtain:
\begin{align*}
&\PP\left(\lambda_{\min}(V_k^{(h)}) \le   \frac{C_\lambda \rho_k M T_h}{2} {\mid N_h} \right) \\ &\leq \PP\prns{\lambda_{\min}(V_k^{(h)}) \le \frac{\rho_k M T_h}{1+N_h^2}\prns{\exptimeredeemsq{k}{\threshold_h} -\prns{\exptimeredeem{k}{\threshold_h}}^2 - \frac{{12}N_h^2}{(1-\alpha)T_h} - (2N_h+1)\epsilon} {\mid N_h}} \\& \le 4\exp\prns{-\frac{2\rho_k M  T_h(C_{\lambda}/6)^2}{45  N_{\max}^4  t_{mix}}}\\
&=4\exp\prns{-\frac{2\rho_k M  T_h(C_{\lambda}/6)^2}{45  N_{\max}^4  t_{mix}}}\\
&=4\exp\prns{-\frac{\rho_k M  T_hC_{\lambda}^2}{810  N_{\max}^4  t_{mix}}}.
\end{align*}

{
By the law of total probability, we then have that
\begin{align*}
\PP\left(\lambda_{\min}(V_k^{(h)}) \le   \frac{C_\lambda \rho_k M T_h}{2} \right) \leq 4\exp\prns{-\frac{\rho_k M  T_hC_{\lambda}^2}{810  N_{\max}^4  t_{mix}}}.
\end{align*}
}

Finally, by definition of the epoch schedule (see \Cref{eq: epoch}), we have
\begin{align*}
  \frac{C_\lambda \rho_k M T_h}{2} \ge  &   \frac{C_\lambda \rho_k M T_1}{2}
  \ge C_0(4+\log(1/\delta)).
\end{align*}
\hfill\Halmos\end{proof}

\subsubsection{Proof of \Cref{lem:sample-variance-bound}}\label{apx:sample-variance-bounds}

\begin{proof}{\it Proof.}
By definition,
\begin{align*}
\exptimeredeem{k}{\threshold} &= \sum_{\tau=0}^N\tau \statprob{k}{\tau}{N} = \sum_{\tau=0}^N \tau \cdot \frac{\frac{1}{\phi_k(\tau)}}{\sum_{\tau'=0}^N \frac{1}{\phi_k(\tau')}},
\end{align*}
where the second equality follows from \Cref{prop:closed-form-lr-avg-per-cust}, and
\begin{align*}
\exptimeredeemsq{k}{\threshold} &= \sum_{\tau=0}^N\tau^2 \statprob{k}{\tau}{N} = \sum_{\tau=0}^N \tau^2 \cdot \frac{\frac{1}{\phi_k(\tau)}}{\sum_{\tau'=0}^N \frac{1}{\phi_k(\tau')}}.
\end{align*}
We then have
\begin{align*}
 \exptimeredeemsq{k}{\threshold} -  \prns{\exptimeredeem{k}{\threshold}}^2 & =\prns{\frac{1}{\sum_{\tau=0}^N \frac{1}{\phi_k(\tau)}}}^2 \bracks{\prns{\sum_{\tau=0}^N \frac{1}{\phi_k(\tau)}}\prns{\sum_{\tau=0}^N \frac{\tau^2}{\phi_k(\tau)}} - \prns{\sum_{\tau=0}^N \frac{\tau}{\phi_k(\tau)}}^2} \\
 &\geq \frac{\mu_{\min}^2}{(N+1)^2}\bracks{\prns{\sum_{\tau=0}^N \frac{1}{\phi_k(\tau)}}\prns{\sum_{\tau=0}^N \frac{\tau^2}{\phi_k(\tau)}} - \prns{\sum_{\tau=0}^N \frac{\tau}{\phi_k(\tau)}}^2}.
\end{align*}
Writing out the summations explicitly, we get
\begin{align*}
 \prns{\sum_{\tau=0}^N \frac{1}{\phi_k(\tau)}}\prns{\sum_{\tau=0}^N \frac{\tau^2}{\phi_k(\tau)}} = \sum_{\tau_1 = 0}^{N-1} \sum_{\tau_2 = \tau_1 + 1}^N \frac{\tau_1^2 + \tau_2^2}{\phi_k(\tau_1)\phi_k(\tau_2)} + \sum_{\tau=0}^N \frac{\tau^2}{\phi_k^2(\tau)},
\end{align*}
and 
\begin{align*}
\prns{\sum_{\tau=0}^N \frac{\tau}{\phi_k(\tau)}}^2 =     \sum_{\tau_1 = 0}^{N-1} \sum_{\tau_2 = \tau_1 + 1}^N \frac{2\tau_1\tau_2}{\phi_k(\tau_1)\phi_k(\tau_2)} + \sum_{\tau=0}^N \frac{\tau^2}{\phi_k^2(\tau)}.
\end{align*}
Putting everything together,
\begin{align*}
 \exptimeredeemsq{k}{\threshold} -  \prns{\exptimeredeem{k}{\threshold}}^2 &\geq\frac{\mu_{\min}^2}{(N+1)^2} \prns{ \sum_{\tau_1 = 0}^{N-1} \sum_{\tau_2 = \tau_1 + 1}^N \frac{(\tau_1 - \tau_2)^2}{\phi_k(\tau_1)\phi_k(\tau_2)}} \\
 &\geq \frac{\mu_{\min}^2}{(N+1)^2 \mu_{\max}^2} \prns{ \sum_{\tau_1 = 0}^{N-1} \sum_{\tau_2 = \tau_1 + 1}^N (\tau_1 - \tau_2)^2} \\
 & =  \frac{\mu_{\min}^2}{(N+1)^2 \mu_{\max}^2} \prns{ \sum_{\tau_1 = 0}^{N-1} \sum_{j = 1}^{N-\tau_1} j^2} \\
 & = \frac{\mu_{\min}^2}{(N+1)^2 \mu_{\max}^2} \prns{ \sum_{\tau = 1}^{N} (N+1-\tau)\tau^2}\\
 & = \frac{\mu_{\min}^2}{(N+1)^2 \mu_{\max}^2} \prns{\frac{1}{12}N(N+1)^2(N+2)}\\
 & = \frac{\mu_{\min}^2}{12 \mu_{\max}^2}\cdot N(N+2).
\end{align*}
\hfill\Halmos\end{proof}

\subsubsection{Proof of \Cref{cor:getting-there}}\label{apx:getting-there}
\begin{proof}{\it Proof.}
Applying \Cref{eq:final-eval-lb} in \Cref{lem:eval-lb} to \Cref{eq:est-error-to-min-eval}, we obtain 
\begin{align}\label{eq:getting-there}
        \abs{R_k(\threshold;\hat{\beta}_k^{(h)}) - R_k(\threshold;\beta_k)}  
 &\le  \frac{\mu_{\max}^2 L_{\mu}}{\mu_{\min}^2}\cdot\frac{3\sigma}{\kappa}\cdot\sqrt{\frac{\log(1/\delta)(1+N_{\max}^2)}{C_{\lambda}\rho_kMT_{h-1}/2}} \notag\\
 &=  \frac{\mu_{\max}^2 L_{\mu}}{\mu_{\min}^2}\cdot\frac{3\sigma}{\kappa}\cdot\sqrt{\frac{2\log(1/\delta)(1+N_{\max}^2)}{C_{\lambda}\rho_kMT_{h-1}}},
\end{align}
{with probability $\zeta_k = 1-3\delta-4\exp\prns{-\frac{\rho_k M  T_hC_{\lambda}^2}{810  N_{\max}^4 t_{mix}}}$}.

Taking a union bound over all $k \in [K]$, \Cref{eq:getting-there} implies that, for all $N \leq N_{\max}$:
\begin{align*}
\abs{R(N;\beta)-R(N;\hat{\beta}^{(h)})} &\leq \sum_{k\in[K]}\rho_k\cdot \frac{\mu_{\max}^2 L_{\mu}}{\mu_{\min}^2}\cdot\frac{3\sigma}{\kappa}\cdot\sqrt{\frac{2\log(1/\delta)(1+N_{\max}^2)}{C_{\lambda}\rho_kMT_{h-1}}} \\
&=\sum_{k\in[K]}\frac{\mu_{\max}^2 L_{\mu}}{\mu_{\min}^2}\cdot\frac{3\sigma}{\kappa}\cdot\sqrt{\frac{2\log(1/\delta)(1+N_{\max}^2)\rho_k}{C_{\lambda}MT_{h-1}}}\\
&=\Delta_h,
\end{align*}
with probability at least 
\begin{align*}
\sum_{k\in[K]}\zeta_k &= 1-3\delta K - 4 \sum_{k\in[K]}\exp\prns{-\frac{\rho_k M  T_hC_{\lambda}^2}{810  N_{\max}^4 t_{mix}}}\\
&\geq 1-3\delta K - 4 \sum_{k\in[K]}\exp\prns{-\frac{\rho_k M  T_1C_{\lambda}^2}{810  N_{\max}^4 \tmixhat}},
\end{align*}
where the inequality follows from the fact that $T_h \geq T_1$ and $t_{mix} \leq \tmixhat$.
\hfill\Halmos\end{proof}

\revedit{
\subsection{Proof of \Cref{cor:exp-regret}}\label{apx:exp-regret}

\begin{proof}{\it Proof.}
By \Cref{thm:main-thm}, with probability at least $1-7KH(T)\delta$,
\begin{align*}
\text{Regret}(\pi,M,T) &= {O}\left(MT_1 + N_{\max}\left(\sum_{h=2}^{H(T)}\sqrt{T_h}\right)\sqrt{M\log(1/\delta)}\right).
\end{align*}
Letting $\delta = 1/T$ and using the fact that $T_h = 2^{h-1}T_1$ for all $h$, this implies:
\begin{align}\label{eq:regret-for-cor}
\text{Regret}(\pi,M,T) &= {O}\left(MT_1 + N_{\max}\sqrt{T_1}2^{H(T)/2}\sqrt{M\log T}\right).
\end{align}
We first bound $T_1$. By \eqref{eq: epoch}:
\begin{align}\label{eq:t1-bound}
T_1 \sim \max\braces{\frac{{1}}{1-2^{-1/\tmixhat}},\frac{N_{\max}^2\log T}{M}, \frac{N_{\max}^4 {\tmixhat} \log T}{M}},
\end{align}
where the notation $\sim$ hides constants that are independent of $\tmixhat, T, N_{\max}$, and $M$. Moreover, by \Cref{prop:tmix-ub}, $\tmixhat = \Theta(N_{\max}^2)$,\footnote{\revedit{While Proposition \revedit{2} only establishes that $\tmixhat=O(N_{\max}^2)$, it is relatively easy to construct a lazy, directed random walk on a $(N_{\max}+1)$-cycle whose mixing time scales as $\Omega(N_{\max}^2)$. We leave the proof of this fact --- which is essentially an application of the Central Limit Theorem --- to the reader.}} which implies that 
\begin{align}\label{eq:t1-bound-2}
T_1 = \Theta\left(\frac{N_{\max}^6\log T}{M}\right).
\end{align}
To bound $H(T)$, observe that $T_h = 2^{h-1}T_1$ implies $H(T) = O\left(\log(T/T_1)\right)$. Plugging in the bound on $T_1$, we have:
\begin{align}\label{eq:ht-bound}
H(T) = O\left(\log\left(\frac{MT}{N_{\max}^6}\right)\right).
\end{align}
Plugging \eqref{eq:t1-bound-2} and \eqref{eq:ht-bound} into \eqref{eq:regret-for-cor}, we obtain:
\begin{align*}
\text{Regret}(\pi,M,T) = \widetilde{O}\left(N_{\max}^6 + N_{\max}^4\cdot\left(\frac{MT}{N_{\max}^6}\right)^{1/2}\right) = \widetilde{O}\left(N_{\max}^6 + N_{\max}\sqrt{MT}\right),
\end{align*}
where the $\widetilde{O}$ notation hides polylogarithmic factors of $M,T$.
\smallskip 
Taking an expectation over the good and bad events, we have:
\begin{align*}
\mathbb{E}_{\pi}\left[\regret(\pi,M,T)\right] = \widetilde{O}\left(N_{\max}^6 + N_{\max}\sqrt{MT} + M\right).
\end{align*}
\hfill\Halmos
\end{proof}
}

\subsection{Bound on the Mixing Loss of \revedit{Greedy-MLE}}\label{apx:mixing-loss}

In this section we analyze the mixing loss of \Cref{alg:greedy}.

\begin{theorem}\label{thm:mixing-loss}
Fix $\delta \in (0,1)$, and {let $\tmixhat$ be any known upper bound on $t_{mix}$.} Under the epoch schedule defined in \Cref{eq: epoch}, with probability at least $1-KH(T)\delta$, \Cref{alg:greedy} guarantees:
\begin{align}
&\mixingloss \leq \frac{{4} M H(T)}{1-2^{-1/\tmixhat}} + \sqrt{\frac{45t_{mix}}{2}} \prns{\sum_{k=1}^K \sqrt{\rho_k}}\prns{\sum_{h=1}^{H(T)}  \sqrt{T_h }}\sqrt{M\log(2/\delta)}.
\end{align}
\end{theorem}

\begin{proof}{\it Proof.}
We similarly define $\alpha = 2^{-1/\tmixhat}$, and omit the dependence of all quantities on $\pi$ throughout. 

As in the proof of \Cref{thm:main-thm}, let $h_{\infty} = \inf\{h \geq 2: R(+\infty) > R(N_h;\hat{\beta}^{(h)})+\Delta_h\}$ be the epoch in which the termination condition was satisfied, with $h_{\infty} = H(T)+1$ if  $R(+\infty) \leq R(N_h;\hat{\beta}^{(h)})+\Delta_h$ for all $h \in \{2,\ldots,H(T)\}$.

Under \Cref{alg:greedy}, the mixing loss is given by:
\begin{align*}
    \mixingloss &=\sum_{h\in[H(T)]} \sum_{t\in\tcal_h} \sum_{k\in[K]}  \sum_{j\in \mathcal{M}_k}  \big[R_k(\threshold_h) -\purchaseprob_k(\tau_{jt}){\mathds{1}\{\tau_{jt} > 0\}}\big].
\end{align*}
Fix $k \in [K], h \in[h_{\infty}-1]$, and let $\epsilon = \sqrt{\frac{45t_{mix}\log(2/\delta)}{2\rho_k M T_h}}$. By Part 3 of \cref{prop:yichun-bounds}, we have:
\begin{align*}
     \pr\prns{\abs{\sum_{j\in \mathcal{M}_k} \sum_{t\in \mathcal{T}_h} \bracks{R_k(\threshold_h) -\purchaseprob_k(\tau_{jt}){\mathds{1}\{\tau_{jt} > 0\}}}} \ge \sqrt{\frac{45t_{mix}\log(2/\delta)\rho_k M T_h}{2}} + \frac{{4} \rho_k M }{1-\alpha}} \le \delta.
\end{align*}
Consider now $h \in \{h_{\infty},\ldots,H(T)\}$. Since our algorithm sets $N_h = +\infty$ in this case, we have \mbox{$R_k(N_h) = \phi_k(\tau_{jt})\mathds{1}\{\tau_{jt}>0\} = \bar{\phi}_k$}, which implies zero mixing loss over these epochs.  

Therefore, union bounding over all $k \in [K]$ and $h \in [h_{\infty}-1]$, we have that, with probability at least $1-KH(T)\delta$,
\begin{align*}
    \mixingloss \le \frac{{4} M H(T)}{1-\alpha} + \sqrt{\frac{45t_{mix}\log(2/\delta) M }{2}} \prns{\sum_{k=1}^K \sqrt{\rho_k}}\prns{\sum_{h=1}^{H(T)}  \sqrt{T_h }}.
\end{align*}
\hfill\Halmos\end{proof}

\section{\Cref{sec:extension} Omitted Proofs}

\subsection{Proof of \Cref{thm:extension}}\label{apx:extension}

\begin{proof}{\it Proof.}
As in the proof of \Cref{thm:main-thm}, we let \mbox{$h_{\infty} = \inf\{h \geq 2: R(+\infty) > R(N_h;\hat{\beta}^{(h)})+3\Delta_h\}$} be the epoch in which the termination condition was satisfied, with $h_{\infty} = H(T)+1$ if  \mbox{$R(+\infty) \leq R(N_h;\hat{\beta}^{(h)})+3\Delta_h$} for all $h \in \{2,\ldots,H(T)\}$. We moreover restrict our analysis to the ``good event'' $\mathcal{E}$, defined as:
 \[\mathcal{E} = \bigg\{\max_{N\in[N_{\max}]}\abs{R(N;\beta)-R(N;\hat{\beta}^{(h)})} \leq \Delta_h \ \forall \ h \leq h_{\infty} \wedge H(T) \bigg\}.\]
Note that $\mathcal{E}$ still holds with probability at least $1-7K\delta H(T)$ under \Cref{alg:non-increasing}, as a corollary of \Cref{lem:bounded-reward-diff}. This follows from the fact that \Cref{lem:bounded-reward-diff} is a statement about the quality of the MLE $\hat{\beta}^{(h)}$, and is not specific to the greedy decisions made in \Cref{alg:greedy}. 

\medskip 

\noindent\textbf{Case 1: $R(+\infty) < R(N^*)$.} 
As before, we have:
\begin{align}\label{eq:ext-regret-decomp}
  \regret(\pi, M, T)  &= MTR(N^*) - M\sum_{h\in[H(T)]}T_hR(N_h) \notag \\
    &\leq MT_1\mu_{\max} + M\left(\sum_{h=2}^{h_{\infty}-1}T_h\left(R(N^*)-R(N_h)\right)\right) + M\big(R(N^*)-R(+\infty)\big)\left(\sum_{h=h_{\infty}}^{H(T)}T_h\right).
\end{align}

Fix $h \leq h_{\infty}-1$. By definition of the good event $\mathcal{E}$: 
\begin{align}\label{eq:ext-regret-step-1}
    R(N_h) \geq  R(N_h; \hat{\beta}^{(h)}) -\gaph \geq  \max_{N\in\mathcal{N}_{h-1}} R(N; \hat{\beta}^{(h)}) -3\gaph,
\end{align}
where the second inequality follows from the definition of the consideration set (see \eqref{eq:consideration-set}).  We relate \eqref{eq:ext-regret-step-1} to the revenue under $N^*$ by first establishing that $N^*$ is never eliminated before termination. We defer its proof to Appendix \ref{apx:never-eliminate-opt}.
\begin{lemma}\label{lem:never-eliminate-opt}
Under event $\mathcal{E}$, $N^*\in\mathcal{N}_h$ for all $h \leq h_{\infty}\wedge H(T)$.
\end{lemma}

By \Cref{lem:never-eliminate-opt}, then, the estimated revenue under $N^*$ in epoch $h$ must be dominated by the greedy optimal decision. We formalize this below:
\begin{align*}
&\max_{N \in \mathcal{N}_{h-1}} R(N;\hat{\beta}^{(h)}) \geq R(N^*;\hat{\beta}^{(h)}) \geq R(N^*;\beta)-\gaph\\
\implies &R(N_h) \geq R(N^*;\beta)-4\gaph,
\end{align*}
where the second inequality follows from $\mathcal{E}$, and the implication follows from \eqref{eq:ext-regret-step-1}. We plug this lower bound into the regret decomposition shown in \eqref{eq:ext-regret-decomp}, and obtain:
\begin{align}\label{eq:ext-almost-there}
  \regret(\pi, M, T) 
    &\leq MT_1\mu_{\max} + M\left(\sum_{h=2}^{h_{\infty}-1}T_h\cdot 4\gaph\right) + M\big(R(N^*)-R(+\infty)\big)\left(\sum_{h=h_{\infty}}^{H(T)}T_h\right).
\end{align}
We conclude the proof of the regret bound by arguing that, under event $\mathcal{E}$, the no-loyalty termination condition is never satisfied (i.e., $R(+\infty) \leq R(N_h;\hat{\beta}^{(h)}) + 3\gaph$ for all $h \in [H(T)]$). This follows from a similar argument as the one used in the proof of \Cref{thm:main-thm}. Namely, suppose for contradiction that the termination condition was satisfied for some $h_{\infty} \leq H(T)$. Then, at $h = h_{\infty}$ we would have:
\begin{align*}
R(+\infty) &> R(N_{h};\hat{\beta}^{(h)})+{3}\gaph \\
&\geq \left(\max_{N\in\mathcal{N}_{h-1}}R(N;\hat{\beta}^{(h)})-2\gaph\right)+{3}\gaph \\
&\geq R(N^*;\hat{\beta}^{(h)}){+\gaph} \\
&\geq R(N^*),
\end{align*}
a contradiction.

Using this in \eqref{eq:ext-almost-there}, we obtain:
\begin{align*}
\regret(\pi,M,T) &\leq MT_1\mu_{\max}+4M\sum_{h=2}^{H(T)}T_h\gaph.
\end{align*}

\medskip

\noindent\textbf{Case 2: $R(+\infty) \geq R(N^*)$.} In this case, we have:
\begin{align*}
\regret(\pi,M,T) &= MTR(+\infty)-M\sum_{h\in[H(T)]}T_hR(N_h)\\
&\leq MT_1\mu_{\max}+M\sum_{h=2}^{h_{\infty}-1}T_h\left(R(+\infty)-R(N_h)\right) \\
&= MT_1\mu_{\max}+M\sum_{h=2}^{h_{\infty}-1}T_h\left(R(+\infty)-R(N_h;\hat{\beta}^{(h)})+R(N_h;\hat{\beta}^{(h)})-R(N_h)\right) \\
&\leq MT_1\mu_{\max}+M\sum_{h=2}^{h_{\infty}-1}T_h\cdot 4\gaph,
\end{align*}
where the final inequality follows from the fact that, for all $h \leq h_{\infty}-1$, $R(+\infty)-R(N_h;\hat{\beta}^{(h)}) \leq 3\gaph$, and moreover under $\mathcal{E}$, $R(N_h;\hat{\beta}^{(h)})-R(N_h)\leq \gaph$.

\medskip

Thus, we have established that, in both cases:
\begin{align*}
\regret(\pi,M,T) &\leq MT_1\mu_{\max} + 4M\sum_{h=2}^{H(T)}T_h\gaph \\
&= MT_1\mu_{\max}+4M\sum_{h=2}^{H(T)}T_h\cdot   \left(\sum_{k\in[K]}\frac{3\mu_{\max}^2L_\mu\sigma}{\mu_{\min}^2\kappa}\sqrt{\frac{2\rho_k\log(1/\delta)(1+N_{\max}^2)}{C_{\lambda}MT_{h-1}}}\right)\\
&\leq MT_1\mu_{\max}+\frac{48\mu_{\max}^3L_\mu\sigma\sqrt{3\log(1/\delta)(1+N_{\max}^2)}}{\mu_{\min}^3\kappa}\left(\sum_{k\in[K]}\sqrt{\rho_k}\right)\left(\sum_{h=2}^{H(T)}\sqrt{T_h}\right)\sqrt{M},
\end{align*}
where the final equality follows from $T_{h-1}\geq T_h/2$. \hfill\Halmos
\end{proof}

\subsubsection{Proof of \Cref{lem:never-eliminate-opt}}\label{apx:never-eliminate-opt}
\begin{proof}{\it Proof.}
We prove this by induction. Note that $N^* \in \mathcal{N}_1$ by definition, since $\mathcal{N}_1 = [N_{\max}]$. Suppose now that $N^* \in \mathcal{N}_{h'}$ for all $h' \leq h_0$, for some $h_0 < h_{\infty}\wedge H(T)$. We show that $N^* \in \mathcal{N}_{h_0+1}$. To see this, note that under $\mathcal{E}$:
\begin{align*}
R(N^*;\hat{\beta}^{(h_0+1)}) &\geq R(N^*;\beta)-\Delta_{h_0+1}\\
&\geq \max_{N\in\mathcal{N}_{h_0}} R(N;\beta) - \Delta_{h_0+1} \\
&\geq \max_{N\in\mathcal{N}_{h_0}} R(N;\hat{\beta}^{(h_0+1)}) - 2\Delta_{h_0+1},
\end{align*}
where the second inequality follows from optimality of $N^*$ under the true parameters $\beta$, and the third inequality again follows from the conditioning on $\mathcal{E}$. As a result, $N^*$ is necessarily included in $\mathcal{N}_{h_0+1}$, by \Cref{eq:consideration-set}. 
\hfill\Halmos
\end{proof}

\medskip

\section{Results on Markov Chain Concentration}

\subsection{Known Results}

We rely on the following theorems for many of our results.
\begin{theorem}[Corollary 2.10 and Remark 2.11 of \citet{paulin2015concentration}] \label{thm: markov concentration}
Consider a uniformly ergodic Markov chain $X_1, \dots, X_n$ with state space $\Omega$ and mixing time $t_{mix}$.
Let $f$ be a non-negative, bounded function on $\Omega$ such that $0\le f(x)\le F$ for any $x\in\Omega$.
Then, for any $\epsilon>0$,
\begin{align*}
    \pr\prns{\abs{\sum_{i=1}^n f(X_i) - \expect\bracks{\sum_{i=1}^n f(X_i)}}\ge \epsilon} \le 2\exp\prns{-\frac{2\epsilon^2}{9 nF^2  t_{mix}}}.
\end{align*}
\end{theorem}

\medskip

\begin{proposition}[Hoeffding bound, Proposition 2.5 of \citet{wainwright2019high}]\label{prop: hoeffding}
Suppose that variables $Z_i, i=1, \dots, n,$ are independent, and $Z_i$ has mean $\mu_i$ and sub-Gaussian parameter $\sigma_i$. Then for any $\epsilon >  0$, we have
\begin{align*}
    \pr\prns{ \abs{\sum_{i=1}^n (Z_i-\mu_i)}\ge \epsilon} \le 2\exp\prns{-\frac{\epsilon^2}{2\sum_{i=1}^n \sigma_i^2}}.
\end{align*}
\end{proposition}

\medskip 

\begin{proposition}[Note 2, Chapter 5.4 of \citet{lattimore2020bandit}]\label{prop: tail bound to subgaussian}
    Let $Z$ be a zero-mean random variable. Moreover, suppose there exists $\sigma > 0$ such that, for any $\epsilon>0$,
    \begin{align*}
        \pr\prns{\abs{Z}\ge \epsilon}\le 2\exp\prns{-\frac{\epsilon^2}{2\sigma^2}}.
    \end{align*}
    Then, $Z$ is $\sqrt{5}\sigma$-sub-Gaussian.
\end{proposition}

\subsection{Proof of \Cref{prop:yichun-bounds}}\label{apx:yichun-bounds}

The proof of \Cref{prop:yichun-bounds} relies on the following closed-form convergence theorem for the Markov chain representing a type-$k$ customer's points to redemption, in any epoch $h$. 

\begin{proposition}\label{cor:convergence-thm}
Fix type $k\in[K]$ and epoch $h \in [h_{\infty}-1]$. For any $\tmixhat \geq t_{mix}$:
\begin{align*}
\max_{\tau_0\in\{0,\ldots,N_h\}}\sum_{\tau=0}^{N_h}\abs{P_k^t(\tau_0, \tau; N_h)-\statprob{k}{\tau}{N_h}} \leq 4\cdot\left(2^{-1/\tmixhat}\right)^t.
\end{align*}
\end{proposition}

\begin{proof}{\it Proof.}
The proof of this result is adapted from \citet{levin2017markov}.  For ease of notation, throughout the proof we omit the dependence of all quantities on $N_h$. Since the Markov chain governing a customer's points to redemption is finite, irreducible and aperiodic, by Equation (4.33) in Section 4.5 of \citet{levin2017markov}, for any positive integer $\ell$, $d_k(\ell t_{mix,k}) \leq 2^{-\ell}.$ 

For any $t \geq 0$, let $\ell(t) = \sup\left\{\ell \in \mathbb{N}: t \geq \ell t_{mix,k}\right\}$. Since $d_k(\cdot)$ is non-increasing (see Exercise 4.2. in \citet{levin2017markov}), we have that, for any $t \geq 0$:
\begin{align*}
d_k(t) \leq d_k(\ell(t) t_{mix,k}) \leq 2^{-\ell(t)} \leq 2^{-(t/t_{mix,k}-1)},
\end{align*}
where the third inequality follows from the fact that $\ell(t) + 1 > \frac{t}{t_{mix,k}}$ by definition. Using the $L_1$-characterization of the TV distance, this implies:
\begin{align*}
&\frac12\max_{\tau_0\in\{0,\ldots,N_h\}}\sum_{\tau=0}^{N_h}|P_k^t(\tau_0,\tau)-p_k(\tau)| \leq 2\cdot\left(2^{-1/t_{mix,k}}\right)^t\\
\implies &\max_{\tau_0\in\{0,\ldots,N_h\}}\sum_{\tau=0}^{N_h}|P_k^t(\tau_0,\tau)-p_k(\tau)|  \leq 4\cdot\left(2^{-1/t_{mix,k}}\right)^t \leq 4\cdot\left(2^{-1/\tmixhat}\right)^t,
\end{align*}
where the final inequality follows from the fact that $\tmixhat \geq t_{mix} \geq t_{mix,k}$ by definition.
\hfill\Halmos
\end{proof}

We use this explicit convergence theorem to derive  \Cref{prop:yichun-bounds}.

\begin{proof}{\it Proof of \Cref{prop:yichun-bounds}.}
Each of these three facts is a corollary of the following general result, whose proof we defer to the end of the section.

\begin{lemma}\label{prop: general concentration}
 Let $f(\tau)$ be any function such that $0\le f(\tau)\le F$ for all $\tau\le N_{\max}$. For any $\epsilon>0$,
 \begin{align*}
     \PP\prns{\abs{\sum_{j \in \mathcal{M}_k}\sum_{t\in\tcal_h} f(\tau_{jt}) - \rho_k M T_h \sum_{\tau=0}^{N_h}\statprob{k}{\tau}{N_h}f(\tau)} \ge \rho_k M T_h \epsilon + \frac{{4}F}{1-\alpha}\rho_k M \mid N_h} \\ \le  2\exp\prns{-\frac{2\rho_k M  T_h\epsilon^2}{45  F^2  t_{mix}}}.
 \end{align*}
\end{lemma}

Part 1 applies \Cref{prop: general concentration} to $f(\tau) = \tau$, with $F = N_h$. Part 2 applies \Cref{prop: general concentration} to $f(\tau) = \tau^2$, with $F = N_h^2$. Part 3 follows from the fact that $R_k(\threshold_h) = \sum_{\tau = 1}^{N_h} \statprob{k}{\tau}{N_h}\phi_k(\tau)$ by definition, and applies \Cref{prop: general concentration} to $f(\tau) = \phi_k(\tau)$, with $F = 1$.
\hfill\Halmos
\end{proof}

\medskip 

\begin{proof}{\it Proof of \Cref{prop: general concentration}.}
We have:
\begin{align}\label{eq:markov-concent}
&\PP\prns{\abs{\sum_{j \in \mathcal{M}_k}\sum_{t\in\tcal_h} f(\tau_{jt}) - \rho_k M T_h \sum_{\tau=0}^{N_h}\statprob{k}{\tau}{N_h}f(\tau)} \ge \rho_k M T_h \epsilon + \frac{{4}F}{1-\alpha}\rho_k M \mid N_h}\notag \\
&\leq \PP\Bigg(\abs{\sum_{j \in \mathcal{M}_k}\sum_{t\in\tcal_h}\Big(f(\tau_{jt}) - \expect\bracks{f(\tau_{jt})\mid N_h}\Big)}+\abs{\sum_{j\in\mathcal{M}_k}\sum_{t\in\mathcal{T}_h}\expect\bracks{f(\tau_{jt})\mid N_h}-\rho_k M T_h \sum_{\tau=0}^{N_h}\statprob{k}{\tau}{N_h}f(\tau)} \notag \\
&\hspace{10cm} \ge \rho_k M T_h \epsilon + \frac{{4}F}{1-\alpha}\rho_k M \mid N_h\Bigg).
\end{align}

For all $j \in \mathcal{M}_k$,
 \begin{align*}
     \abs{\expect\bracks{\sum_{t\in\tcal_h} f(\tau_{jt}) \mid N_h} - T_h \sum_{\tau=0}^{N_h}\statprob{k}{\tau}{N_h}f(\tau)}
     =& \abs{\expect\bracks{\sum_{t\in\tcal_h} \sum_{\tau=0}^{N_h} f(\tau)\mathds{1}\{\tau_{jt} = \tau\} \mid N_h} - T_h \sum_{\tau=0}^{N_h}\statprob{k}{\tau}{N_h}f(\tau)} \\
     \le & F\sum_{t\in \mathcal{T}_h}\max_{\tau_0 \in \{0,\ldots,N_h\}}\sum_{\tau=0}^{N_h} \abs{P_k^t(\tau_0, \tau; N_h) - \statprob{k}{\tau}{N_h}}\\
     \le & F\sum_{t\in \mathcal{T}_h} {4}\alpha^t\\
     \le & \frac{{4}F}{1-\alpha},
 \end{align*}
 where the first inequality uses linearity of expectation and the assumption that $f(\tau) \leq F$, and the second inequality follows from \Cref{cor:convergence-thm}. Plugging this into \eqref{eq:markov-concent}, we obtain:
 \begin{align}\label{eq:markov-concent-2}
&\PP\prns{\abs{\sum_{j \in \mathcal{M}_k}\sum_{t\in\tcal_h} f(\tau_{jt}) - \rho_k M T_h \sum_{\tau=0}^{N_h}\statprob{k}{\tau}{N_h}f(\tau)} \ge \rho_k M T_h \epsilon + \frac{{4}F}{1-\alpha}\rho_k M \mid N_h}\notag \\
&\leq\PP\Bigg(\abs{\sum_{j \in \mathcal{M}_k}\sum_{t\in\tcal_h}\Big(f(\tau_{jt}) - \expect\bracks{f(\tau_{jt})\mid N_h}\Big)} \ge \rho_k M T_h \epsilon \mid N_h\Bigg).
 \end{align}

Note that the Markov chain has a single irreducible class of states, and is thus uniformly ergodic. By \cref{thm: markov concentration}, then, for all $j \in \mathcal{M}_k$:
 \begin{align*}
     \pr\prns{\abs{\sum_{t\in\tcal_h} f(\tau_{jt}) - \expect\bracks{\sum_{t\in\tcal_h} f(\tau_{jt})\mid N_h}}\ge \epsilon \mid N_h} &\le 2\exp\prns{-\frac{2\epsilon^2}{9 T_h F^2  t_{mix,k}(N_h)}}\\
     &\leq 2\exp\prns{-\frac{2\epsilon^2}{9 T_h F^2  t_{mix}}}.
 \end{align*}
 Define $Z_j = \sum_{t\in\tcal_h} f(\tau_{jt}) - \expect\bracks{\sum_{t\in\tcal_h} f(\tau_{jt}) \mid N_h}$. Since customers are independent, by \cref{prop: tail bound to subgaussian}, $Z_j$'s are independent $3F \sqrt{5 T_h  t_{mix}}/2$-sub-Gaussian random variables. Then, applying Hoeffding's inequality (\cref{prop: hoeffding}) to \eqref{eq:markov-concent-2}, we have:
 \begin{align*}
     \pr\prns{\abs{\sum_{j \in \mathcal{M}_k} Z_j} \ge \rho_k M T_h \epsilon \mid N_h} &\leq 2\exp\prns{-\frac{(\rho_k M T_h \epsilon)^2}{2\cdot \rho_k M \cdot 45F^2T_ht_{mix}/4}}\\
     &= 2\exp\prns{-\frac{2\rho_k M T_h\epsilon^2}{45F^2  t_{mix}}},
 \end{align*}
 which concludes the proof of the claim.
 \hfill\Halmos
\end{proof}

\medskip 

\section{Maximum Likelihood Estimator for Generalized Linear Models}\label{apx:mle}
In this section, we briefly review some existing results on the likelihood theory of generalized linear models.

Consider a fixed, unknown $\theta^* \in \mathbb{R}^d$ and a fixed, strictly increasing, known link function $\mu: \mathbb{R}\rightarrow \mathbb{R}$.
For $i=1, 2, \dots$, assume the following model holds:
\begin{align*}
    Y_i = \mu\prns{Z_i^\intercal \theta^*} +\epsilon_i,
\end{align*}
where $Z_i$'s are features satisfying $\|{Z_i}\|\le 1$, and $\epsilon_i$'s are independent zero-mean noise.
Moreover, the conditional distribution of $Y$ given $Z$ is from the exponential family, and its density, parameterized by $\theta\in \Theta$, can be written as
\begin{align*}
    \PP(Y\mid Z) = \exp\braces{\frac{YZ^\intercal \theta^* - m(Z^\intercal \theta^*)}{g(\eta)} + h(Y, \eta)}.
\end{align*}
Here, $\eta\in \mathbb{R}^+$ is a known scale parameter; $m,g$ and $h$ are three normalization functions mapping from $\mathbb{R}$ to $\mathbb{R}$. The Gaussian, binomial, Poisson, gamma, and the inverse-Gaussian distributions are all examples of the exponential family.
It follows from standard properties of exponential families \citep{brown1986fundamentals} that $m$ is infinitely differentiable and satisfies $\dot{m}(Z^\intercal \theta^*) = \EE[Y\mid Z]$ and $\ddot{m}(Z^\intercal \theta^*) = \VV(Y\mid Z)$.

Suppose we have independent samples of $Y_1, Y_2, \dots, Y_n$, each respectively conditioned on $Z_1, Z_2, \dots, Z_n$. The maximum likelihood estimator $ \hat{\theta}_n $ can be written as the solution to the following equation (\citet{li2017provably}, Eq. (15)):
\begin{align*}
    \sum_{i=1}^n \prns{Y_i - \mu\prns{Z_i^\intercal \theta}}Z_i = 0.
\end{align*}

Consider the following assumptions on the data generating process.
\begin{assumption} \label{assumption: general mle}
    The data generating process satisfies the following conditions:
    \begin{enumerate}
    \item $\mu$ is twice differentiable. Its first- and second-order derivatives are upper-bounded by $L_\mu$ and $G_\mu$, respectively.
    \item $\kappa :=\inf_{\norm{z}\le 1, \norm{\theta - \theta^*}\le 1} \dot{\mu} \prns{z^\intercal \theta} >0$.
    \item The noise $\epsilon_i$ is sub-Gaussian with parameter $\sigma$, where $\sigma$ is some positive, universal constant.
    \end{enumerate}
\end{assumption}

The following theorem gives a non-asymptotic concentration bound for the MLE estimation.
\begin{theorem}[\citet{li2017provably}, Theorem 1]\label{thm: general mle}
Suppose \cref{assumption: general mle} holds.
    Define $V_n = \sum_{i=1}^n Z_i Z_i^\intercal$, and let $\delta>0$ be given. Furthermore, assume that
    \begin{align*}
        \lambda_{\min}(V_n)\ge \frac{512 G_\mu^2 \sigma^2}{\kappa^4}\prns{d^2 + \log \frac{1}{\delta}}.
    \end{align*}
    Then, with probability at least $1-3\delta$, the maximum likelihood estimator satisfies, for all $z\in \mathbb{R}^d$:
    \begin{align*}
        \abs{z^\intercal\prns{\hat{\theta}_n - \theta^*}}\le \frac{3\sigma}{\kappa}\sqrt{\log(1/\delta) }\sqrt{z^\intercal V_n^{-1} z}.
    \end{align*}
\end{theorem}

\revedit{
\section{Additional Computational Experiments}\label{apx:exper-k-types}

\subsection{Impact of Number of Types $K$}

In this section we investigate the impact of heterogeneity induced by an increasing number of types $K$. Leveraging the sequence of worst-case instances derived in the proof of \Cref{lem:pers-tight}, for every $K$ that we test, we let
\begin{align*}
\rho_k = \begin{cases}
\frac12\left(\frac1k-\frac{1}{k+1}\right) \quad &\forall \ k \leq K-2\\ 
\frac12 \cdot \frac{1}{K-1} \quad &k = K-1 \\
\frac12 \quad &k = K.
\end{cases}
\end{align*}

As in the worst-case instance, we assume that the type-$K$ customer is a ``frequent'' customer, with $\phi_K(\tau) = 1$ for all $\tau \in \mathbb{N}$ and $\bar{\phi}_K = 1$. For all types $k \in [K-1]$, we fix $\epsilon \in [0, 1/2]$ and $\alpha > 0$, and let $\bar{\phi}_k = \epsilon$, with 
\begin{align*}
\phi_k(\tau) = \epsilon + \frac{1-2\epsilon}{1+\exp\left(-\alpha(k+1/2-\tau)\right)}.
\end{align*}

\begin{figure}[t]
    \centering
    \begin{subfigure}[t]{0.48\textwidth}
        \centering
        \includegraphics[width=\textwidth]{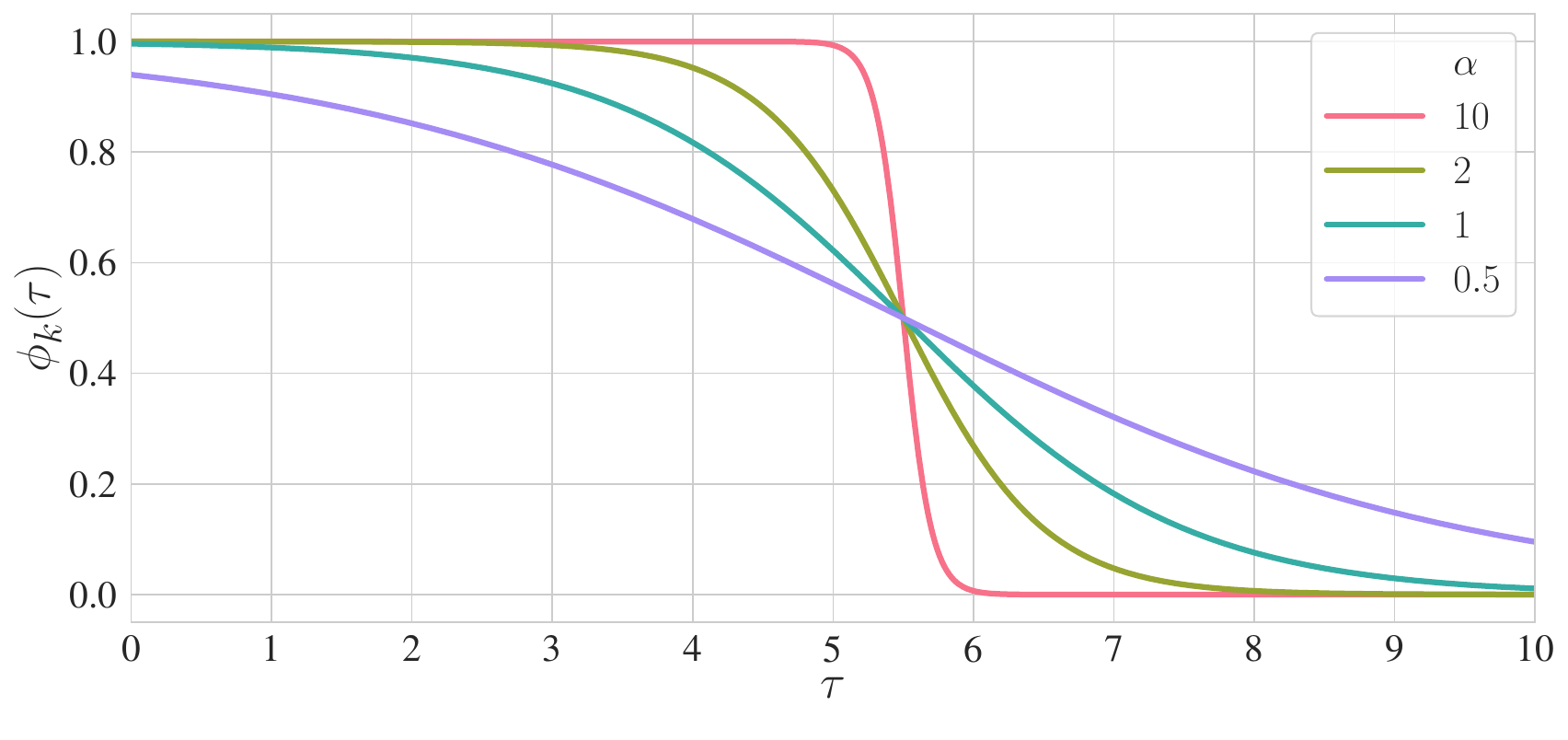}
        \caption{$\phi_k(\tau)$ vs. $\tau$, $\epsilon = 0$, $\alpha \in \{0.5,1,2,10\}$}
        \label{fig:smoothed_alpha}
    \end{subfigure}
    \hfill
    \begin{subfigure}[t]{0.48\textwidth}
        \centering
        \includegraphics[width=\textwidth]{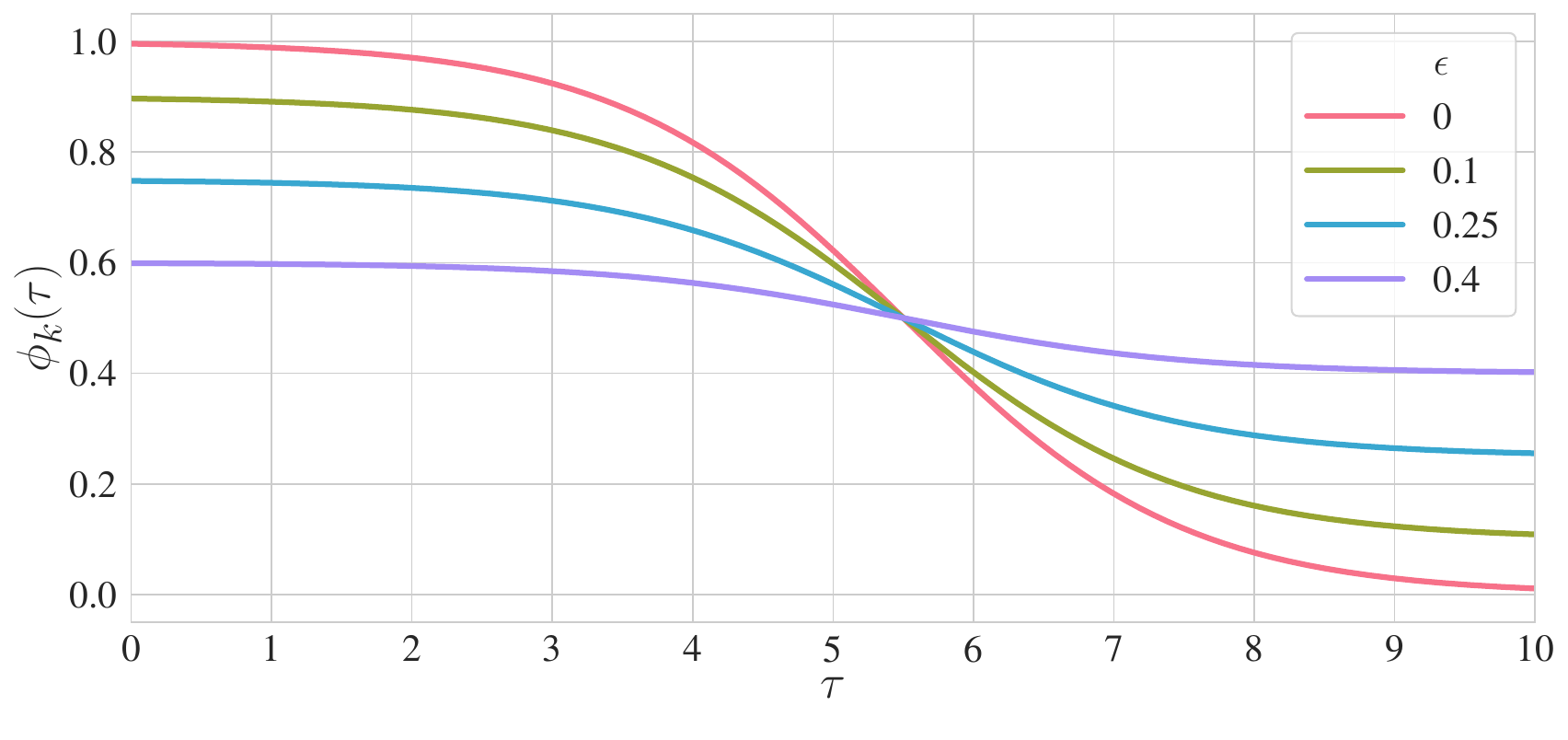}
        \caption{$\phi_k(\tau)$ vs. $\tau$, $\alpha = 1$, $\epsilon \in \{0,0.1,0.25,0.4\}$}
        \label{fig:smoothed_epsilon}
    \end{subfigure}
    \caption{\centering Dependence of $\phi_k(\tau)$ on $\alpha$ and $\epsilon$, for $k = 5$.}
    \label{fig:smoothed-worst-case}
\end{figure}

For finite values of $\epsilon$ and $\alpha$, this construction, depicted in \Cref{fig:smoothed-worst-case} for $k = 5$,  corresponds to a ``smoothed'' version of our worst-case instance. These two parameters respectively control (i) the maximum reward-induced lift in purchase probability, given by $\phi_k(0)-\bar{\phi}_k$, and (ii) the sharpness of the customer's transition around $k$. (We henceforth refer to $k$ as type-$k$'s ``sensitivity threshold,'' since this is the threshold past which they begin to purchase in the theoretical worst-case instance. We moreover use $\tau_k^*$ to denote this sensitivity threshold.) Small values of $\epsilon$ correspond to settings in which the rewards program substantially shifts purchase behavior, with $\phi_k(0)-\bar{\phi}_k$ approaching 1 as $\epsilon \to 0$. As $\epsilon \to 1/2$, the lift vanishes and the program has no effect, capturing the practical observation that customers with high baseline purchase probabilities are relatively insensitive to rewards programs. Large values of $\alpha$, on the other hand, produce sharp, near-step transitions around $\tau_k^*$, recovering the theoretical worst-case instance in the limit; small values of $\alpha$ produce gradual transitions, modeling customers whose purchase behavior responds smoothly to changes in points to redemption.

\paragraph{Results.} \Cref{fig:type_sensitivity_alpha} plots the dependence of the \revedit{VoP} on $K$ for varying values of $\alpha \in \{1,\ldots,10\}$ and fixed $\epsilon = 0$. Interestingly, we observe that the \revedit{VoP} has a nontrivial dependence on $K$, depending on the value of $\alpha$. In particular, for $\alpha \geq 6$, the \revedit{VoP} is concave and increasing in $K$, numerically validating the intuition that, as the number of types increases, so does the value of personalization. Surprisingly, however, for smaller values of $\alpha$, we observe an initial {\it decrease} in the \revedit{VoP} as $K$ increases, followed by a rapid plateau. The reason for this is that $K = 2$ corresponds to the case where the sensitivity threshold of one type is at $\tau_1^* = 1$, while the second type is a frequent customer for whom the seller should not implement a rewards program. For $K = 3$, the type-2 customer has sensitivity threshold $\tau_2^* = 2$; however, for small values of $\alpha$, the personalized thresholds of type-1 and type-2 customers are identical. Since the preferences of two of the three types are aligned, setting a uniform threshold for all three types is {\it less} costly than when $K = 2$. Additionally, for fixed values of $K$, we observe that the \revedit{VoP} converges to the theoretical worst-case upper bound as $\alpha$ increases. Indeed, for $(K, \alpha) = (10,10)$, the \revedit{VoP} is approximately $1.63$, compared to the upper bound $\frac{\pi^2}{6} \approx 1.645$. On the other hand, when $\alpha = 1$, the \revedit{VoP} is approximately $1.27$.

\begin{figure}[t]
    \centering
    \begin{subfigure}[t]{0.48\textwidth}
        \centering
        \includegraphics[width=\textwidth]{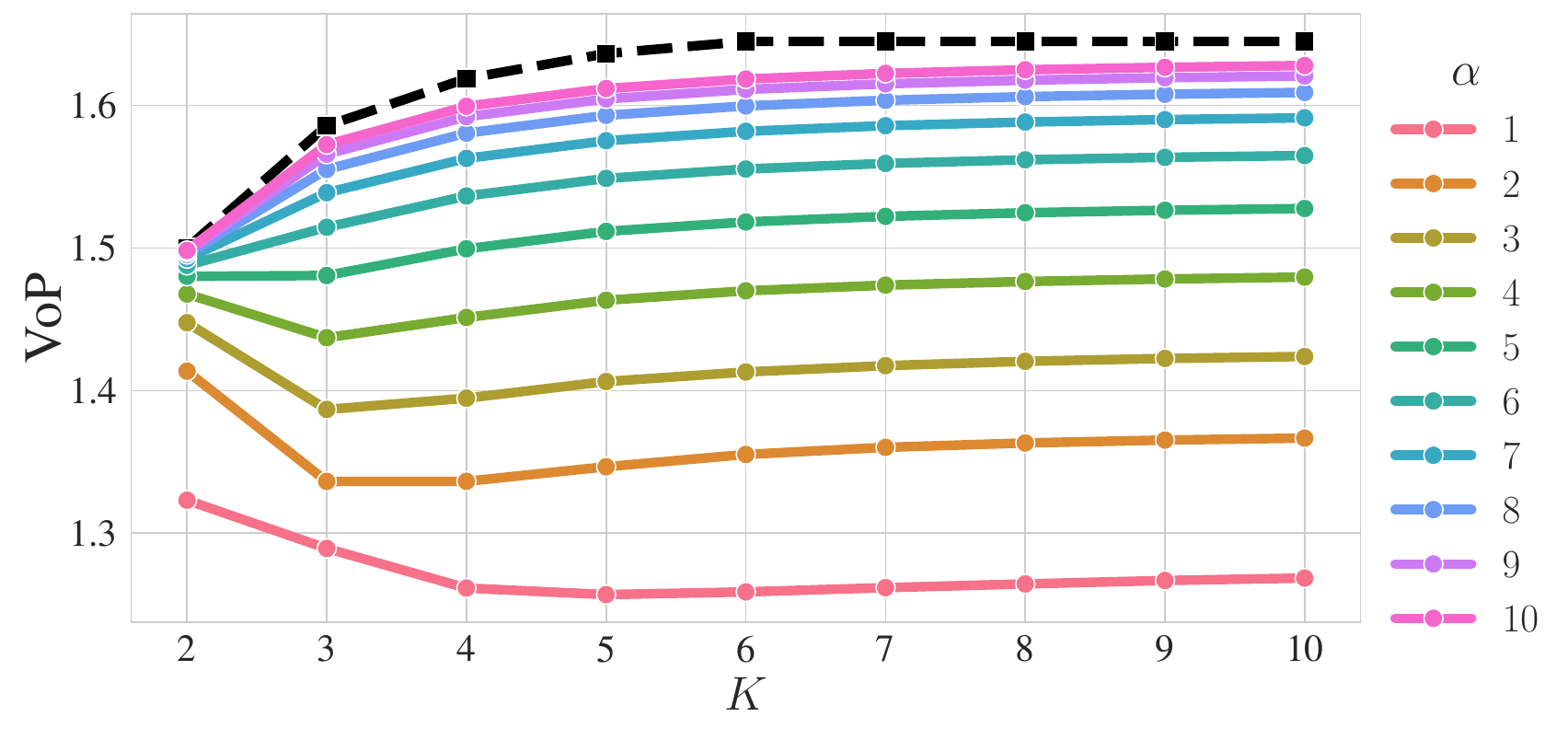}
        \caption{\revedit{VoP} vs. $K$, $\epsilon = 0$, $\alpha \in \{1,\ldots,10\}$}
        \label{fig:type_sensitivity_alpha}
    \end{subfigure}
    \hfill
    \begin{subfigure}[t]{0.48\textwidth}
        \centering
        \includegraphics[width=\textwidth]{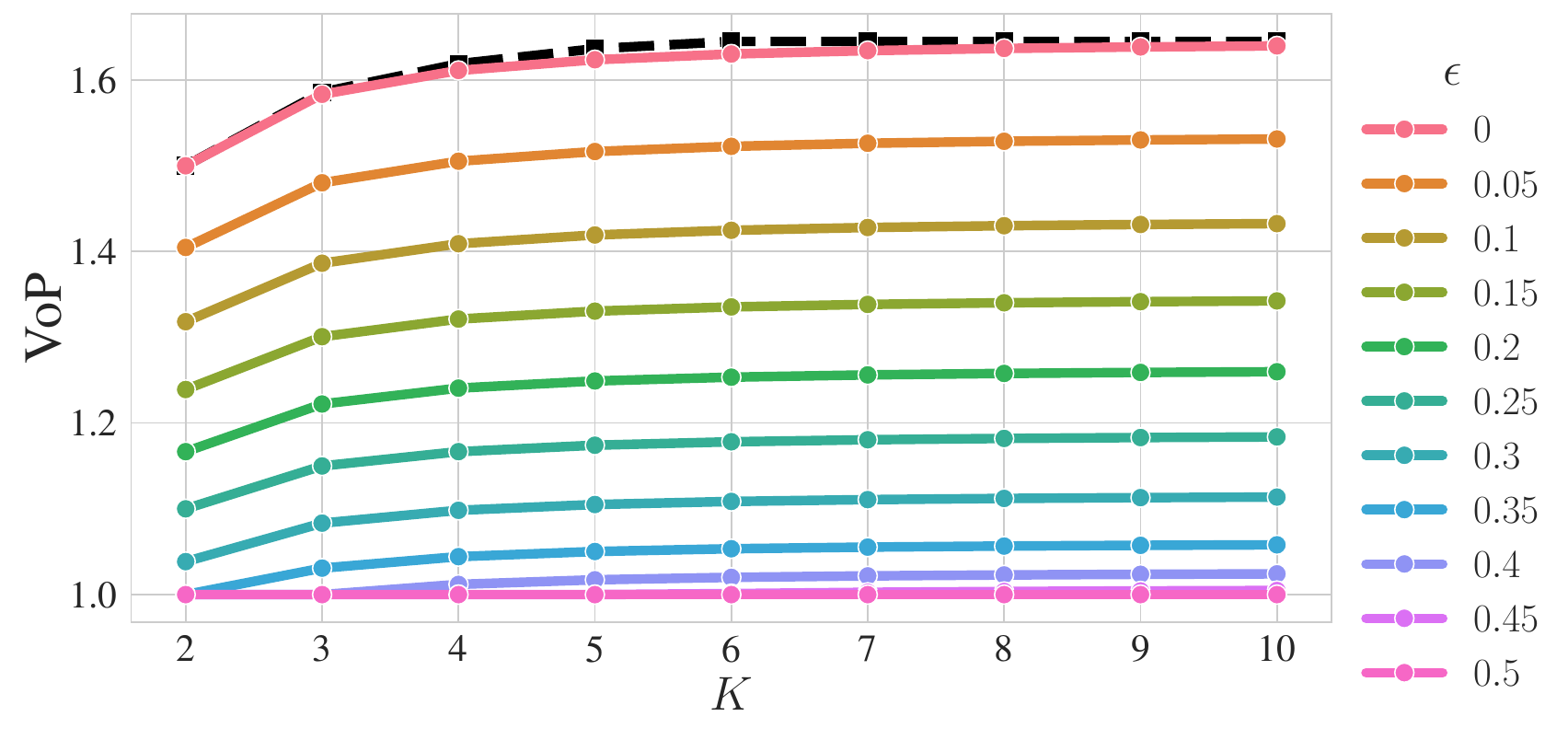}
        \caption{VoP vs. $K$, $\alpha = 10^3$, $\epsilon \in \{0,0.05,\ldots,0.5\}$}
        \label{fig:type_sensitivity_epsilon}
    \end{subfigure}
    \caption{\revedit{VoP} vs. $K$, for varying $\alpha$ and $\epsilon$. The black curve in both plots represents the worst-case upper bound of \mbox{$\min\left\{K-(K-1)2^{-1/(K-1)},\frac{\pi^2}{6}\right\}$} derived in \Cref{thm:price-of-fairness}.}
    \label{fig:smoothed-worst-case-pof}
\end{figure}

\Cref{fig:type_sensitivity_epsilon} next plots the dependence of the \revedit{VoP} on $K$, for varying values of $\epsilon$, fixing $\alpha = 10^3$. We similarly observe diminishing gains from personalization as a function of the number of types. Moreover, these gains strongly depend on the maximum reward-induced lift, as parameterized by $\epsilon$. In particular, for $\epsilon = 0$, the \revedit{VoP} tracks the worst-case upper bound extremely closely across all values of $K$: in the worst case, the relative difference between these two quantities is less than 0.9\%. As $\epsilon$ increases and $\bar{\phi}_k$ and $\phi_k(0)$ approach each other, the value of personalization decreases significantly, with the \revedit{VoP} approaching 1 as $\epsilon$ approaches $1/2$, for all values of $K$. The reason for this is that, in the best case, customers' purchase probability increases by $1-2\epsilon$ as a result of the rewards program. As $\epsilon$ increases, offering a rewards program is no longer optimal, and more types' preferences become aligned, resulting in a decrease in \revedit{VoP}. 

Taken together, our results show that, while the value of personalization generally increases in the number of types, personalization is most valuable when (i) customers have sharp, threshold-style responses to rewards programs, and (ii) the maximum reward-induced lift of the program is large. When the purchase probabilities change more smoothly around customers' sensitivity thresholds, or there is a limited overall lift in purchase probability, the gains from personalization are more modest, even under high levels of heterogeneity. 
}

\revedit{

\subsection{Sensitivity to $N_{\max}$} \label{sec: nmax exp} 

In this section we study the numerical dependence of our algorithms' cumulative regret on $N_{\max}$, using the same experimental setup as in \Cref{ssec:num-regret}. For the algorithm instantiations, we let $T_1 = 1$, $T_h = 2^{h-1}T_1$ as before.\footnote{\revedit{While our regret analysis requires $T_1$ to grow with $N_{\max}$, this dependence arises from a conservative sufficient condition used to guarantee concentration in finite samples. Our baseline numerical experiments, in which we also set $T_1 = 1$, illustrate that dependence on $N_{\max}$ is not required for strong empirical performance. Similar choices appear in prior work: for instance, \cite{bastani2020online} uses a parameter $q$ to determine the initial set of times each arm must be sampled; while $q$ is required to be a sufficiently large constant in their theoretical analysis, their numerical experiments set $q = 1$, with their algorithm demonstrating strong performance despite this relaxation.}} Additionally, we set the termination thresholds to be $\Delta_h = \frac{0.15 N_{\max}}{20\sqrt{M(\sum_{i=1}^{h-1}T_i)}}$ for all $h \geq 1$. This instantiation preserves the linear dependence of $\Delta_h$ on $N_{\max}$ used in our theoretical results (see \Cref{eq: termination thresholds}); moreover, we obtain the baseline value of $\Delta_h$ used for our initial set of experiments when $N_{\max} = 20$. 

\begin{figure}[t]
\centering
\includegraphics[width=0.5\textwidth]{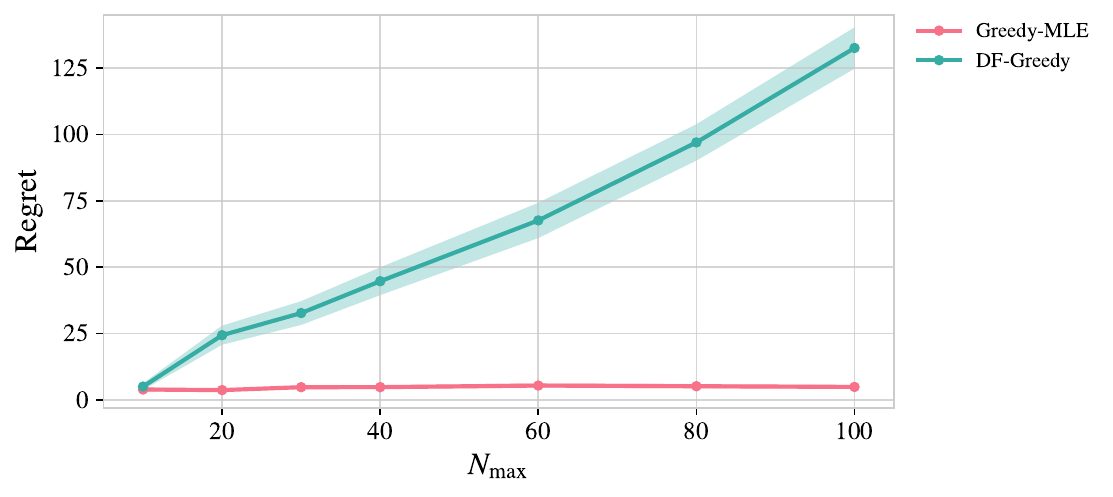}
\caption{\centering{Cumulative regret vs. $N_{\max}$, for $T = 5,000$. For all tested values of $N_{\max}$, the optimal uniform threshold $N^* = 8$.}
\label{fig:Nmax_sensitivity}}
\end{figure}

\paragraph{Results.} In \Cref{fig:Nmax_sensitivity} we plot the cumulative regret of our two algorithms for $T = 5,000$, averaged over 100 replications. We first observe that the performance of \revedit{Greedy-MLE} does not degrade with $N_{\max}$, in contrast to our theoretical guarantees. This can be explained by the fact that, as the algorithm proceeds along the horizon and epoch lengths become increasingly long, the MLE error decreases. Hence, the quality of the greedy solution $N_h$ only improves (i.e., $N_h$ approaches $N^* = 8$). As a result, larger values of $N_{\max}$ only hurt the algorithm early on; later on in the horizon, a larger feasible set does not hurt the algorithm's performance since the higher thresholds are clearly suboptimal. (We note that the linear dependence of $\Delta_h$ on $N_{\max}$ only makes the exit condition {\it less} likely to be satisfied, as $N_{\max}$ increases; this can only improve the performance of \revedit{Greedy-MLE} in this setting.)

The cumulative regret of \revedit{\DfGreedy}, on the other hand, grows linearly in $N_{\max}$, echoing our theoretical guarantees. The reason for this difference in performance is that, while \revedit{Greedy-MLE} selects the greedy threshold at the end of each epoch, \revedit{\DfGreedy} chooses the {\it largest} surviving threshold in its consideration set. Early in the horizon this threshold lies near $N_{\max}$, far above $N^* = 8$; the resulting loss in these early epochs is thus increasing in $N_{\max}$. Additionally, because $\Delta_h$ scales linearly with $N_{\max}$, high thresholds are eliminated from the consideration set more slowly as the horizon progresses, forcing \revedit{\DfGreedy} to select higher thresholds for longer.
}

\revedit{
\subsection{Robustness to Type Misspecification} \label{sec: noise within type}

In this section we numerically study the robustness of our algorithms to {\it within-type} heterogeneity.  

\paragraph{Setup.} We consider a setting in which there are $K = 2$ observable types, with five customers of each type. 
We generalize the purchase probability model considered in our basic setup by perturbing type-$k$ customers' responsiveness to the rewards program, parameterized by $\beta_k$, by some idiosyncratic, individual-level noise. Formally, for each customer $j$ of type $k$,
\[\phi_j(\tau) = \min\left\{\bar{\phi}_k\left(1 + \exp\left(\alpha_k - \beta_j\tau\right)\right), 1\right\},\]
where $\beta_j = \beta_k \exp(\xi_j), \xi_j \sim \mathcal{N}(0,\sigma_\beta^2)$. Observe that we recover our original setting when $\sigma_\beta = 0$; as $\sigma_\beta$ increases, however, so does customers' responsiveness to the rewards program. 

Importantly, we assume that our algorithms estimate a single value $\beta_k$ for all type-$k$ individuals, rather than estimating each individual $\beta_j$. We quantify the impact of this misspecification by computing our algorithms' regret for fixed $T = 5,000$ and $\sigma_\beta > 0$, where the benchmark against which we compare our algorithms has knowledge of the true value of $\beta_j$ for each individual $j$. Abusing notation and letting $\text{Regret}(\sigma_\beta)$ denote the cumulative regret for a fixed value of $\sigma_\beta$, we report each algorithm's {\it normalized} regret
\[\text{Normalized Regret}(\sigma_\beta) = \frac{\text{Regret}(\sigma_\beta)}{\text{Regret}(0)}.\]
Hence, the normalized regret represents the relative loss that results from partitioning individuals into segments based on observable features. 

Finally, all other parameters and algorithm instantiations are as in our baseline experiments (see \Cref{ssec:learning-numerics}).

\paragraph{Results.} Our results, shown in \Cref{fig:m10_beta_ratio_to_sigma0}, illustrate that both algorithms are robust to a small to moderate amount of within-type heterogeneity. For $\sigma_\beta \leq 0.3$, the normalized regret of our algorithms is not statistically different from 1, at a 5\% significance level. Over this range of $\sigma_\beta$, the 95\% CI for the ``heterogeneity multiplier'' $\exp({\xi_j})$ is $[0.82,1.22]$ when $\sigma_\beta = 0.1$; this interval widens to $[0.56,1.80]$ when $\sigma_\beta = 0.3$, representing a scenario in which the most rewards-sensitive individuals within a single type are over three times as sensitive as the least rewards-sensitive individuals within that type. This suggests that our algorithms tolerate meaningful within-type variation in reward responsiveness.

As $\sigma_\beta$ increases, we observe the intuitive phenomenon that the performance of both algorithms begins to deteriorate. For $\sigma_\beta \in \{0.4,0.5\}$, the normalized regret is still moderate for \revedit{Greedy-MLE}, averaging about $1.12$ and $1.24$, respectively, while that of \revedit{\DfGreedy} rises to
about $1.33$ and $1.47$ on average. The deterioration becomes more stark from $\sigma_\beta = 0.6$ onwards, with \revedit{Greedy-MLE} being significantly less robust to within-type heterogeneity than \revedit{\DfGreedy}. In the worst case, when $\sigma_\beta = 1$, the 95\% CI for $\exp(\xi_j)$ is $[0.14,7.10]$, corresponding to very large within-type variation in response slopes. Our algorithms see correspondingly large degradation, with the normalized regret of \revedit{Greedy-MLE} averaging $7.47$, compared to $3.10$ for \revedit{\DfGreedy}. These results reveal an interesting side benefit of \revedit{\DfGreedy}. By selecting the largest threshold in its active consideration set, \revedit{\DfGreedy} continues to observe purchase behavior for higher values of $\tau$, where customers with smaller $\beta_j$ --- whose rewards response decays more slowly in $\tau$ --- remain informative. This additional variation at higher points-to-redemption partially mitigates the loss from fitting a single parameter $\beta_k$ to heterogeneous customers.

\begin{figure}[t]
    \centering
    \includegraphics[width=0.7\textwidth]{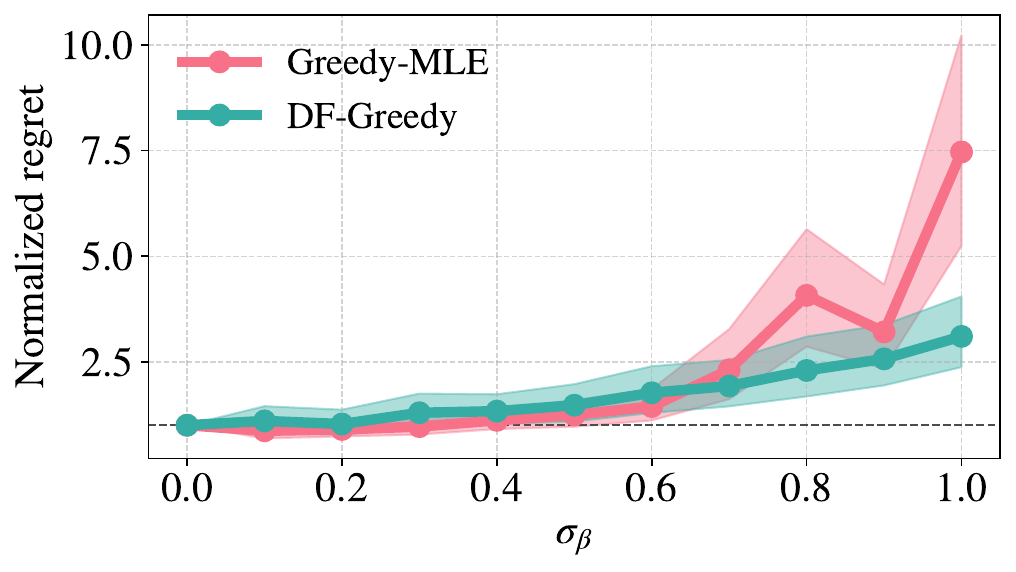}
    \caption{\centering Normalized regret of \revedit{Greedy-MLE} and \revedit{\DfGreedy} vs. $\sigma_\beta$.}
    \label{fig:m10_beta_ratio_to_sigma0}
\end{figure}

}

\end{APPENDICES}







\end{document}